\documentclass{article} 
\usepackage[nonatbib,final]{neurips_2025}

\usepackage{times}
\RequirePackage{algorithm}
\RequirePackage{algorithmic}
\usepackage{breakcites}
\usepackage{times}
\usepackage{graphicx} % Required for inserting images
\usepackage{mathtools}
\usepackage{amsfonts, amsmath, amssymb}
\usepackage{amsthm}
\usepackage{thm-restate}
\usepackage{microtype}
\usepackage{booktabs}  % for professional tables
\usepackage[utf8]{inputenc} % allow utf-8 input
\usepackage[T1]{fontenc}    % use 8-bit T1 fonts
\usepackage{hyperref}       % hyperlinks
\usepackage{url}            % simple URL typesetting  
\usepackage{xcolor}         % colors
\usepackage[utf8]{inputenc} % allow utf-8 input
\usepackage[T1]{fontenc}    % use 8-bit T1 fonts      % hyperlinks
\usepackage{url}            % simple URL typesetting
\usepackage{nicefrac}       % compact symbols for 1/2, etc.    % microtypography
\usepackage{xcolor}
\usepackage{caption} %new
\usepackage{subcaption}
\usepackage{wrapfig}
\usepackage{mathtools}
\usepackage{xspace}
\usepackage[utf8]{inputenc}
\usepackage{amsfonts, amsmath, amssymb, color}
\usepackage{hyperref}
\usepackage{ulem}
\usepackage{textcomp}
\usepackage{comment}
\usepackage[dvipsnames]{xcolor}
\usepackage{bbm}
\usepackage{notations}
\usepackage{tikz}
\usepackage[hang,flushmargin]{footmisc} %new package for removing indent in footnote

\usepackage{enumitem}

\title{Active Seriation: Efficient Ordering Recovery with Statistical Guarantees}
\author{%
  James Cheshire \qquad   \qquad Yann Issartel \\
  \\
  LTCI, Télécom Paris, Institut Polytechnique de Paris %\thanks{Use footnote for providing further information
}

\begin{document}

\maketitle

\begin{abstract}
Active seriation aims at recovering an unknown ordering of $n$ items by adaptively querying pairwise similarities. 
The observations are noisy measurements of entries of an underlying $n \times n$ permuted Robinson matrix, whose permutation encodes the latent ordering.
The framework allows the algorithm to start with partial information on the latent ordering, including seriation from scratch as a special case.
We propose an active seriation algorithm that provably recovers the latent ordering with high probability. Under a uniform separation condition on the similarity matrix, optimal performance guarantees are established, both in terms of the probability of error and the number of observations required for successful recovery.
\end{abstract}

\section{Introduction}

%\ya{To update refs (arXiv papers now published)} 

The seriation problem involves ordering $n$ items based on noisy measurements of pairwise similarities. 
This reordering problem originates in archaeology, where it was used for the chronological dating of graves~\cite{Robinson51}. 
More recently, it has found applications in data science across various domains, including envelope reduction for sparse matrices~\cite{barnard1995spectral}, read alignment in de novo sequencing~\cite{garriga2011banded,bioinfo17}, time synchronization in distributed networks~\cite{Clock-Synchro04,Clock-Synchro06}, and interval graph identification~\cite{fulkerson1965incidence}. 

In many of these settings, pairwise measurements can be made in an adaptive fashion, leveraging information from previously chosen pairs. 
Motivated by these applications, we study the problem of recovering an accurate item ordering from a sequence of actively selected pairwise measurements.

The classical seriation problem has attracted substantial attention in the theoretical literature~\cite{atkins1998spectral,fogel2013convex,janssen2020reconstruction,giraud2021localization,cai2022matrix,issartel2024minimax}. 
In this line of work, the learner receives a single batch of observations, whereas in the active setting considered in this paper, the algorithm can adaptively select which pairs to observe;
to the best of our knowledge, this active seriation setting has not been previously analyzed.

Related problems include adaptive ranking and sorting under noisy observations~\cite{jamieson2011active,braverman2009sorting,heckel2019active}. 
However, these problems typically rely on pairwise comparisons (e.g., is item~$i$ preferred to item~$j$?) to infer a total order. 
In contrast, seriation builds on pairwise similarity scores that encode proximity in the underlying ordering. 
This distinction leads to different statistical and algorithmic challenges.

%%%%%%%%%%%%

\paragraph{Problem setup.} In the seriation paradigm, we assume the existence of an unknown symmetric matrix $M$ representing pairwise similarities between a collection of $n$ items. 
The matrix $M$ is structured so that the similarities $M_{ij}$ are correlated with an unknown underlying ordering of the items, which is encoded by a permutation $\pi=(\pi_1,\ldots,\pi_n)$ of $[n]$.
The similarity $M_{ij}$ between items $i$ and $j$ tends to be large when their positions $\pi_i$ and $\pi_j$ are close, and small when they are far apart. 
To model this structure formally, the literature assumes that $M$ is a Robinson matrix whose rows and columns have been permuted by the latent permutation $\pi$~\cite{fogel2013convex,recanati2018reconstructing,janssen2020reconstruction,giraud2021localization}; see Section~\ref{section-setting} for a formal presentation. 

We consider a general framework in which the algorithm may be initialized with partial information on the latent ordering~$\pi$.
Specifically, it can be provided with a correct ordering of a subset of the items that is consistent with~$\pi$.
This setting includes seriation from scratch as the special case where no such information is given initially.

The algorithm adaptively selects pairs of items and observes noisy measurements of their similarities. 
A total of $T$ such observations are collected, and the noise is controlled by an unknown parameter $\sigma$.
The goal is to recover the latent ordering $\pi$ of the $n$ items from these $T$ observations.
We evaluate the performance of an estimator $\hat \pi$ by its probability of failing to identify~$\pi$.
See Section~\ref{section-global-setting} for details.

%%%%%%%%%%%%%%%

\paragraph{Contribution.}
We introduce \textit{Active Seriation by Iterative Insertion} (ASII), an active  procedure for estimating the ordering~$\pi$. 
Unlike most existing seriation methods, which are non-active, ASII is remarkably simple, runs in polynomial time, and yet achieves optimality guarantees both in terms of error probability and sample complexity.

In the general framework considered in this paper, we analyze the performance of ASII and establish exponential upper bounds on its probability of error.
In the special case of seriation from scratch, where no prior ordering information is available, we provide a sharp characterization of the statistical difficulty of ordering recovery over the class of similarity matrices with minimal gap $\Delta$ between adjacent coefficients.
This difficulty is governed by the signal-to-noise ratio (SNR) of order {{\small $\Delta^2 T / (\sigma^2 n)$}},
which can be interpreted as the number of observations per item, {{\small $T/n$}}, multiplied by the SNR per observation, {{\small $(\Delta / \sigma)^2$}}. 
Our results identify a phase transition at the critical level $\textup{SNR} \asymp \ln n$: 
below this threshold, ordering recovery is information-theoretically impossible, while above it, ASII achieves recovery with a probability of error that decays exponentially fast with the SNR.
Moreover,  we show that no algorithm can achieve a faster decay rate, establishing optimality in this regime.

Finally, we illustrate the performance of ASII through numerical experiments and a real-data application.

%%%%%%%%%%%%%%%%%%%%%%%%%%%%%%%%%%%%%%%%%%%%%%%

\subsection{Related work}

\textbf{Classical seriation.} 
The non-active seriation problem was first addressed by~\cite{atkins1998spectral} in the noiseless setting, using a spectral algorithm. 
Subsequent works analyzed this approach under noise~\cite{fogel2013convex,giraud2021localization,natik2021consistency}, typically relying on strong spectral assumptions to establish statistical guarantees. 
More recent contributions proposed alternative polynomial-time algorithms with guarantees under different and sometimes weaker assumptions~\cite{janssen2020reconstruction,giraud2021localization,cai2022matrix,issartel2024minimax}. 
Our analysis falls within the line of work on Lipschitz-type assumptions~\cite{giraud2021localization,issartel2024minimax}, as our $\Delta$-separation condition can be viewed as a lower Lipschitz requirement.

\textbf{Statistical-computational gaps.}
Prior studies have suggested statistical–computational gaps in the non-active seriation problem~\cite{giraud2021localization,cai2022matrix,berenfeld2024seriation}, 
where known polynomial-time algorithms fall short of achieving statistically optimal rates under certain noise regimes or structural assumptions. 
While some of these gaps have recently been closed~\cite{issartel2024minimax}, the resulting algorithms tend to be complex and may not scale well in practice. 
In contrast, in the active setting, we show that a simple and computationally efficient algorithm achieves statistically optimal performance.

\textbf{Bandit models.} 
In classical multi-armed bandits (MAB)~\cite{bubeck2012}, each arm yields independent rewards, and the goal is to maximize cumulative reward or identify the best arm. 
In contrast, in active seriation, each query $(i,j)$ measures the similarity between two interdependent items, and all measurements must be consistent with a single latent ordering. 
This interdependence prevents a direct application of standard MAB algorithms such as UCB or Thompson Sampling, which treat arms as independent and do not exploit structural relationships between them.

Algorithmically, our approach is  related to noisy binary search and thresholding bandits~\cite{feige1994computing,karp2007noisy,ben2008bayesian,Nowak09binary,emamjomeh2016deterministic,cheshire2020,cheshire2021problem}, 
which rely on adaptive querying under uncertainty. 
However, these methods operate on low-dimensional parametric models, 
whereas seriation involves a combinatorial  ordering that must remain globally consistent across  item pairs.

\textbf{Active ranking.} 
A related but distinct problem is active ranking~\cite{heckel2019active,shah2017simple}, where a learner infers a total order  based on noisy pairwise comparisons or latent score estimates.  
Extensions include Borda, expert, and bipartite ranking~\cite{pmlr-v202-saad23b,cheshire2023active}. 
These methods typically assume that each item is associated with an intrinsic scalar score, and that pairwise feedback expresses a directional preference between items. 
In contrast, seriation relies on pairwise similarity information, which encodes proximity rather than preference. 
Recovering the latent ordering therefore requires global consistency among all pairwise similarities, 
making the problem more constrained and structurally different from standard ranking tasks.

%%%%%%%
% Section Problem
%%%%%%%

\section{Active Seriation: Problem Setup}
\label{section-global-setting}
%In this section, we formally state the seriation problem considered in this paper.

\subsection{Similarity matrix and ordering}
\label{section-setting}

Given a collection of  items $[n] :=\{1,\ldots,n\}$, let $M=[M_{ij}]_{1\leq i,j \leq n}$ denote the (unknown)  \textit{similarity matrix},  
where the coefficient $M_{ij}\in \mathbb{R}$ measures the similarity between items~$i$ and $j$.
Our structural assumption   on $M$ is related to the class of Robinson matrices,  introduced below.

\begin{defi}
\label{def-robinson}
   A  matrix $R\in\R^{n\times n}$ is called a Robinson matrix (or R-matrix) if it is symmetric and %satisfies the inequalities
   \[ R_{i, j} > R_{i-1, j}   \quad \text{and} \quad   R_{i, j} > R_{i, j+1} \ , \]
   for all $i \le j$ on the upper triangle of $R$. 
\end{defi}
The entries of a Robinson matrix decrease as one moves away from the (main) diagonal (see Figure~1, left).
In other words, each row / column  is unimodal and attains its maximum on the diagonal.
\begin{wrapfigure}{r}{6cm}
    \centering
    \includegraphics[width=0.75\linewidth]{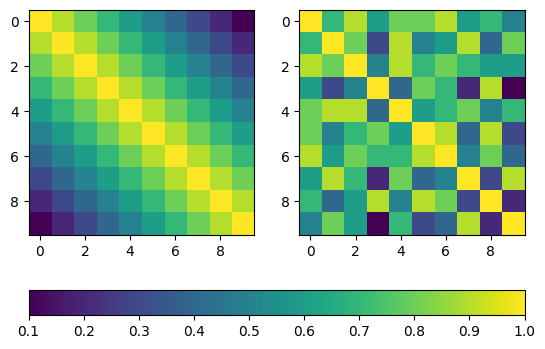}
    \caption{{{\footnotesize R-matrix \& a permuted version.}}}
    \label{fig:rob}
\end{wrapfigure}
Following \cite{atkins1998spectral}, a matrix is said to be \textit{pre-R} if it can be transformed into an R-matrix by simultaneously permuting its rows and columns (Figure~1, right). 
In this paper, we assume the similarity matrix $M$ is pre-R, i.e., 
\begin{equation}
\label{model}
M = R_\pi := [R_{\pi_i, \pi_j}]_{1 \leq i,j \leq n} \ ,
\end{equation}
for some R-matrix $R$ and some permutation $\pi = (\pi_1, \ldots, \pi_n)$ of $[n]$. 
The permutation $\pi$ represents the latent \textit{ordering} of the items.

In~\eqref{model}, the similarities $M_{ij}$ reflect the ordering~$\pi$ as follows: $M_{ij}$ tends to be larger when the positions $\pi_i$ and $\pi_j$ are close together, and smaller when they are far apart.

%We call $\pi$ an \textit{ordering} of the $n$ items with similarity matrix $M$ if it satisfies~\eqref{model} for some R-matrix. 

\begin{rem}
\label{two-orderings}
The items have exactly two orderings: 
if {{\small $\pi$}} is an ordering, then the reverse permutation~{{\small $\pi^{\mathrm{rev}}$}}, defined by {{\small $\pi^{\mathrm{rev}}_i = n+1 - \pi_i$}}, is also an ordering.
Indeed, if {{\small $M = R_{\pi}$}}, then {{\small $M = R^{\mathrm{rev}}_{\pi^{\mathrm{rev}}}$}}, where {{\small $R^{\mathrm{rev}}$}} is obtained by reversing the rows and columns of {{\small $R$}}.
In the sequel, we refer to either of these two orderings as \textit{the latent ordering}.
\end{rem}

%%%%%%%%%%%%

\subsection{Active observation model}

In the active seriation problem, the algorithm sequentially queries $T$ pairs of items and receives noisy observations of their pairwise similarities, encoded by the unknown similarity matrix $M$ in~\eqref{model}.
The goal is to recover the latent ordering~$\pi$ of the $n$ items from these noisy observations.

Initially, the algorithm is provided with partial information on the latent ordering $\pi$: it receives as input a correct ordering $\tilde{\pi}$ of the items $\{1,\ldots,n-\tilde n\}$, corresponding to the restriction of $\pi$ to this subset.
This framework interpolates between online seriation ($\tilde n = 1$), where a single item is inserted into an existing ordering, and seriation from scratch ($\tilde n = n$).

More precisely, for some $\tilde n \in [n]$, the algorithm is given a permutation $\tpi = (\tpi_1, \ldots, \tpi_{n-\tn})$ of $[n-\tn]$ satisfying\footnote{In full generality, condition~\eqref{new-relative-order-incremental} should be assumed to hold either for the latent ordering~$\pi$ or for its reverse~$\pi^{\mathrm{rev}}$; for simplicity of exposition, we state it in terms of~$\pi$, as this distinction plays no essential role in what follows.}
\begin{align}
\label{new-relative-order-incremental}
\forall\, i, j \in [n-\tn] \ , \qquad
\pi_i < \pi_j \quad \Longleftrightarrow \quad \tpi_i < \tpi_j \  .
\end{align}
In other words, the input permutation $\tpi$ preserves the relative order of the items $\{1,\ldots,n-\tn\}$ in the latent ordering $\pi$.

At each round $t=1,\ldots,T$, the algorithm selects a pair $(i_t, j_t)$ with $i_t \neq j_t$, possibly depending on the outcomes of previous queries.
It then receives a noisy observation of the similarity~$M_{i_t j_t}$, given by a $\sigma$-sub-Gaussian random variable\footnote{A random variable $X$ is said to be \textit{$\sigma$-sub-Gaussian} if $\mathbb{E}[\exp(u X)] \leq \exp\left(\frac{u^2\sigma^2}{2}\right)$ for all $u\in\mathbb{R}$.} with mean $M_{i_t j_t}$.
This sub-Gaussian setting covers standard observation models, including Gaussian noise and bounded random variables.

After $T$ queries, the algorithm outputs a permutation $\hat{\pi}$ of $[n]$ as its estimate of the latent ordering~$\pi$.

%%%%%%%%%%%%

\subsection{Error probability and minimal gap}

The algorithm is considered successful if $\hat{\pi}$ recovers either $\pi$ or its reverse $\pi^{\mathrm{rev}}$, as both permutations are valid orderings (Remark~\ref{two-orderings}). 
The probability of error is thus defined as:\footnote{
This lack of identifiability between $\pi$ and $\pi^{\mathrm{rev}}$ persists in the seriation from scratch case, since condition~\eqref{new-relative-order-incremental} is vacuous.
}
\begin{equation}
\label{objective}
p_{M, T} \ := \ \mathbb{P}_{M, T} \left\{ \hat{\pi} \neq \pi \text{ and } \hat{\pi} \neq \pi^{\mathrm{rev}} \right\} \ ,
\end{equation}
where the probability is over the randomness in the $T$ observations collected on the matrix $M$.

A key quantity in our analysis of~\eqref{objective} is the separation between consecutive entries in the underlying R-matrix. Specifically, for any pre-R matrix $M$ as in \eqref{model}, we define its \textit{minimal gap} as
\begin{equation}
\label{minimal-gap-def}
\Delta_M := \min_{\substack{1 < i < j \leq n-1}} \left\{ (R_{i,j} - R_{i-1,j}) \wedge (R_{i,j} - R_{i,j+1}) \right\} \ ,
\end{equation}
where $\wedge$ denotes the minimum.
The quantity~\eqref{minimal-gap-def} measures the smallest difference between neighboring entries in the R-matrix associated with $M$. Note that $\Delta_M > 0$ (by Definition~\ref{def-robinson}) and that $\Delta_M$ is well-defined even though $M$ can be associated with different R-matrices (Remark~\ref{two-orderings}).

Some of our guarantees are stated over classes of matrices with a prescribed minimal gap.
Namely, for any $\Delta>0$, we introduce
\begin{equation}
\label{model-set}
\mathcal{M}_\Delta := \left\{ M \in \mathbb{R}^{n \times n} : M \text{ is  pre-R, and } \Delta_M \geq \Delta \right\} \ ,
\end{equation}
as the set of pre-R  matrices with minimal gap at least $\Delta$.

To simplify the presentation of our findings, we focus on the challenging regime where
\begin{equation}
\label{challenging-regime}
\frac{\Delta}{\sigma} \leq 1 \ , \qquad \ \ (\sigma > 0) \ ,
\end{equation}
i.e., the signal-to-noise ratio  per observation is at most $1$. This  excludes mildly stochastic regimes where the problem has essentially the same difficulty as in the noiseless case.

%%%%%%%%%%%%%%%%%

%%%%%%%%%%%%%

%%%%%%
% ALgo
%%%%%%

\section{Seriation procedure}
\label{algo-section}

\textsc{ASII} (Active Seriation by Iterative Insertion) reconstructs the underlying ordering by iteratively inserting each new item into an already ordered list.
At iteration $k$, given a current estimated ordering {{\small $\hat{\pi}^{(k-1)}$}} of the items {{\small $\{1,\ldots,k-1\}$}}, the algorithm inserts item $k$ at its correct position to form an updated ordering {{\small $\hat{\pi}^{(k)}$}} of {{\small $\{1,\ldots,k\}$}}.
This process is repeated until the full ordering of all $n$ items is obtained.
The algorithm is fully data-driven: it takes as input only the sampling budget~$T$; it does not rely on any knowledge of the noise parameter~$\sigma$, the similarity matrix~$M$, nor its minimal gap~$\Delta_M$.

To perform this insertion efficiently, two key subroutines are used:

\begin{enumerate}[label=(\roman*), leftmargin=2em, itemsep=2pt, topsep=2pt]
    \item \textbf{Local comparison rule.}
    To decide where to insert $k$, the algorithm must compare its position relative to items already ordered in {{\small $\hat{\pi}^{(k-1)}$}}.
    This is achieved through the subroutine \textsc{Test}, which determines whether $k$ lies to the left, in the middle, or to the right of two reference items $(l,r)$.
    
    \item \textbf{Efficient insertion strategy.}
    To minimize the number of comparisons, the algorithm performs a binary search over the current ordering, where each comparison is made via \textsc{Test}.
    Because these tests are noisy, the procedure is further stabilized through a backtracking mechanism. %which allows the algorithm to detect and correct occasional local errors without increasing the overall sample complexity.
\end{enumerate}

%The next paragraphs describe these two subroutines in more detail. 

\paragraph{(i) Local comparison rule.}
Given two items $(l,r)$ such that $\pi_l < \pi_r$, the goal is to determine whether $k$ lies to the left, in the middle, or to the right of $(l,r)$ in the unknown ordering~$\pi$. 
Formally,  this means deciding whether $\pi_k < \pi_l$, or $\pi_k \in (\pi_l,\pi_r)$, or $\pi_k > \pi_r$, respectively.

The subroutine \textsc{Test} is based on the following property of Robinson matrices: 
when $k$ lies between $l$ and $r$, its similarity to both $l$ and $r$ tends to be higher than the similarity between $l$ and $r$ themselves.
Accordingly, \textsc{Test} compares the three empirical similarities {{\small $\widehat M_{kl}$, $\widehat M_{kr}$,}} and {{\small $\widehat M_{lr}$,}}
and identifies the smallest one as the pair of items that are farthest apart.
For example, if {{\small $\widehat M_{lr}$}} is the smallest, then $k$ lies in the middle. 
%if {{\small $\widehat M_{kr}$}} is the smallest, $k$ lies to the left of $(l,r)$;
%and if {{\small $\widehat M_{kl}$}} is the smallest, $k$ lies to the right.
%The subroutine \textsc{Test} returns $b \in \{-1,0,1\}$ accordingly. 
The pseudocode of \textsc{Test} is provided in Appendix~\ref{app:pseudocode}. 

Each call to the subroutine \textsc{Test} is performed with a limited sampling budget, typically of order {{\small $O\!\left(T / (\tilde n \log_2 k)\right)$}}, in order to be sampling-efficient.
With such a budget, individual tests are not designed to be reliable with high probability.
This design trades local test accuracy for global budget efficiency; occasional incorrect tests are later corrected by a backtracking mechanism.

Remark: Higher reliability is required only at a few critical steps of \textsc{ASII} (e.g., during the initialization of the binary search), where a larger budget of order {{\small $O(T / \tilde n)$}} is allocated to ensure correctness with high probability.

%%%%%%%%%%%%%%%

\paragraph{(ii) Efficient insertion strategy.} 
The idea of incorporating backtracking into a noisy search has appeared in several studies, e.g., \cite{feige1994computing, ben2008bayesian, emamjomeh2016deterministic,cheshire2021problem}.
Here, we adopt this general principle to design a robust insertion mechanism that remains reliable under noisy relational feedback.

At this stage of the algorithm, the items {{\small $\{1,\ldots,k-1\}$}} have already been placed into an estimated ordering {{\small $\hat{\pi}^{(k-1)} = (\hat{\pi}^{(k-1)}_1,\ldots,\hat{\pi}^{(k-1)}_{k-1})$}} constructed so far.
To insert the new item $k$ into this list, we use the subroutine \textsc{Binary \& Backtracking Search} (\bbs), which determines the relative position of~$k$ within the current ordering {{\small $\hat{\pi}^{(k-1)}$}}.
The search proceeds by repeatedly using the subroutine \textsc{Test} to decide whether $k$ lies in the left or right half of a candidate interval, thereby narrowing down the possible insertion range.

Because the outcomes of \textsc{Test} are noisy, even a single incorrect decision can misguide the search and lead to an erroneous final placement.
A naive fix would be to allocate many samples per \textsc{Test} to ensure highly reliable outcomes, but this would increase the sample complexity and undermine the benefit of active sampling.
Instead,  \bbs\ uses a small number of samples per \textsc{Test}, of order {{\small $O\!\left(T / (\tilde n \log_2 k)\right)$}}, just enough to ensure a constant success probability (e.g., around~$3/4$).
Backtracking then acts as a corrective mechanism that prevents local errors from propagating irreversibly.

The backtracking mechanism operates as follows: 
the algorithm keeps track of previously explored intervals and performs sanity checks at each step to detect inconsistencies in the search path.
When an inconsistency is detected, it backtracks to an earlier interval and resumes the search.
This prevents local mistakes from propagating irreversibly.
Theoretical analysis shows that, as long as the number of correct local decisions outweighs the number of incorrect ones --- an event that occurs with high probability under the assumption $M\in \cM_{\Delta}$--- the final insertion position is accurate.

Hence, the backtracking mechanism allows the algorithm to detect and correct occasional local errors.
As a result, we can use only a small number of samples per call to \textsc{Test}, while still ensuring correct insertions at the global level.
The \textsc{ASII} procedure thus provides both sampling efficiency and robustness, despite noisy observations.
A pseudocode of this procedure is given in Appendix~\ref{app:pseudocode}.

%%%%%%%%%%%%%%%%%%%%%%%

%%%%%%%%%%%%%%%%%

%Theoretical RESULTS
%%%%%%%%%%%%%%%

\section{Performance analysis}
\label{section-performance-analysis}
We study the fundamental limits of ordering recovery in active seriation, 
deriving information-theoretic lower bounds and algorithmic upper bounds on the error probability defined in~\eqref{objective}.

%%%%%%%%%%%%%%%%%
%Upper bounds with partial knowledge on the ordering
%%%%%%%%%%%%%%%%%

\subsection{Upper bounds for ASII}

We analyze the performance of \asii\ when the algorithm is provided with a permutation $\tpi$ of the items $\{1,\ldots,n-\tn\}$, consistent with the latent ordering~$\pi$ of the $n$ items.
The following theorem gives an upper bound on the error probability of ASII for recovering $\pi$.
Its proof is in Appendix~\ref{app:simple}.

\begin{thm}[Upper bound with partial information]
\label{thm:UB:insert-n2-into-n1}
There exists an absolute constant $c_0>0$ such that the following holds.
Let $n \ge 3$ and $\tn \in [n]$, and assume that the input permutation
$\tpi$ of $[n-\tn]$ satisfies~\eqref{new-relative-order-incremental}.
Let $(\Delta,\sigma,T)$ be such that condition~\eqref{challenging-regime} holds and 
\[
\frac{\Delta^2 T}{\sigma^2 \tn} \ge c_0 \ln n  \ .
\]
If $M \in \mathcal M_\Delta$, then the error probability of \asii\ satisfies %\ya{old $c_1$ became  $c_0$ here; (to adjust proofs)}
\begin{equation}
\label{eq:UB-partial}
p_{M,T} \le \exp\!\left(-\,\tfrac{1}{2400}\,\frac{\Delta^2 T}{\sigma^2 \tn}\right) \ .
\end{equation}
\end{thm}

Once the SNR, $\Delta^2 T / (\sigma^2 \tn)$, exceeds a logarithmic threshold in~$n$, the error probability of \asii\ decays exponentially fast with the SNR.
We emphasize that ASII achieves this performance without requiring any knowledge of the model parameters $(\Delta,\sigma)$.

Beyond this uniform guarantee over the class $\cM_\Delta$, the performance of ASII admits a finer, instance-dependent characterization.
The same bound holds with $\Delta$ replaced by the true minimal gap $\Delta_M$ of the underlying matrix $M$ (defined in~\eqref{minimal-gap-def}).
Precisely, for the instance-dependent $\mathrm{SNR}_M = \Delta_M^2 T / (\sigma^2 \tn)$, the bound becomes $p_{M,T} \le \exp(-\mathrm{SNR}_M/2400)$ whenever $\mathrm{SNR}_M \ge c_0 \ln n$.

%%%%%%%
% optimality from scratch
%%%%%%%%

\subsection{Minimax optimal rates for seriation from scratch}
\label{section:minimax-rates-Delta-separated-mat}

We establish matching information-theoretic lower bounds in the case of active seriation from scratch ($\tn = n$), where no prior ordering information is available.
This analysis identifies two regimes, depending on whether the SNR is below or above a critical threshold.

%%%%%%%%%%

\paragraph{4.2.1  \ \ Impossibility regime.} When the SNR satisfies $\frac{\Delta^2 T}{\sigma^2 n} \lesssim \ln n$, no algorithm can recover the ordering with vanishing error probability. 
The following theorem formalizes this impossibility, establishing a constant lower bound on the error probability~\eqref{objective} for any algorithm in this regime. 

\begin{thm}[Impossibility regime]
\label{thm:LB2}
There exists an absolute constant $c_1>0$ such that the following holds.
Let $n \ge 9$ and $(\Delta,\sigma,T)$ be such that 
\[\frac{\Delta^2 T}{\sigma^2 n} \le c_1 \ln n \ .\]
Then, for any algorithm $A$, there exists a matrix $M \in \mathcal M_\Delta$ such that the error probability of $A$ satisfies $p_{M,T} \ge 1/2$. %\ya{old $c_0$ became  $c_1$ here(to adjust proofs)}
\end{thm}

As expected, the impossibility regime is more pronounced when the minimal gap $\Delta$ is small or when the noise parameter $\sigma$ is large.
To escape this regime, the number of observations per item, $T/n$, must grow at least quadratically with $\sigma/\Delta$ (up to logarithmic factors).
The proof is in Appendix~\ref{app:LB}.

%%%%%%%%%%

\paragraph{4.2.2  \ \ Recovery regime.}

In the complementary regime where the SNR satisfies
$\frac{\Delta^2 T}{\sigma^2 n} \gtrsim \ln n$,
exact recovery becomes achievable.
The \asii algorithm attains an exponentially small error probability,
and this rate is minimax optimal over the class $\mathcal M_\Delta$, up to absolute constants in the exponent.

Specifically, the upper bound in this regime follows directly from
Theorem~\ref{thm:UB:insert-n2-into-n1} by taking $\tn=n$.
If
\begin{equation}
\label{SNR-condition-upper-bound-from-scratch}
    \frac{\Delta^2 T}{\sigma^2 n} \ge c_0 \ln n \ ,
\end{equation}
where $c_0$ is the same absolute constant as in Theorem~\ref{thm:UB:insert-n2-into-n1}, then for any $M\in\mathcal M_\Delta$, the probability of error of \asii satisfies,
\begin{equation*}
p_{M,T} \le \exp\!\left(-\,\tfrac{1}{2400}\,\frac{\Delta^2 T}{\sigma^2 n}\right) \ .
\end{equation*}
Conversely, the next theorem gives a matching lower bound, showing that no algorithm can achieve a faster error decay than exponential in the SNR.
Its proof is in Appendix~\ref{app:LB}.

\begin{thm}[Recovery regime]
\label{thm:UB}
Let $n \ge 4$ and $(\Delta,\sigma,T)$ be such that condition~\eqref{SNR-condition-upper-bound-from-scratch} holds.
Then, for any algorithm $A$, there exists $M \in \cM_{\Delta}$ such that the error probability of $A$ satisfies
\begin{equation*}
p_{M,T} \ge \exp\!\left(-\,8\,\frac{\Delta^2 T}{\sigma^2 n}\right) \ .
\end{equation*}
\end{thm}
Together, Theorems~\ref{thm:LB2} and~\ref{thm:UB} delineate the statistical landscape of active seriation from scratch, establishing a sharp phase transition between impossibility and recovery at the critical SNR level $\frac{\Delta^2 T}{\sigma^2 n} \asymp \ln n$.

%%%%%%%%%%%%%%%%%

\paragraph{4.2.3  \ \ Discussion: intrinsic hardness and invariance to model assumptions.}

Both lower bounds (Theorems~\ref{thm:LB2} and~\ref{thm:UB}) are established under a Gaussian noise model with centered, homoscedastic entries of variance~$\sigma^2$, 
whereas our upper bound is proved in the more general sub-Gaussian setting allowing heterogeneous noise levels. 
Since these bounds match, potential heterogeneity in the noise variances does not affect the minimax rates (at least in terms of exponential decay in SNR).

Moreover, the lower bounds are derived for the simple, affine, Toeplitz matrix
\begin{equation}
\label{simple-affine-toeplitz}
    R_{ij} = (n - |i-j|) \Delta  \ ,  
\end{equation} 
yet the attainable rates coincide with those obtained under the general assumption $M \in \cM_\Delta$. Hence, allowing heterogeneous, non-Toeplitz matrices comes at no statistical cost.
This may appear surprising, since the Toeplitz assumption is classical in batch seriation (e.g.,~\cite{cai2022matrix}).

Even when the latent matrix is fully known, as in the one-parameter family~\eqref{simple-affine-toeplitz} with known parameter~$\Delta$, the attainable rates remain unchanged.
This indicates that the hardness of active seriation arises from the combinatorial nature of the latent ordering, rather than from uncertainty about the latent matrix.

%%%%%%%%%%%%%%%%%%%%%%%%%

\subsection{Sample complexity for high probability recovery}

We summarize our results in terms of sample complexity, defined as the number of observations required to achieve exact recovery with probability at least $1-1/n^2$.
Combining the impossibility result of Theorem~\ref{thm:LB2} with the recovery guarantee of Theorem~\ref{thm:UB:insert-n2-into-n1}, we obtain a sharp characterization of the sample complexity in the from-scratch case  ($\tilde n = n$).

\begin{cor}[Sample complexity]
\label{cor:sample-complexity}
Let $n \ge 9$ and $\tilde n \in [n]$, and assume that the input permutation $\tpi$ of $[n-\tilde n]$ satisfies~\eqref{new-relative-order-incremental}.
Let $(\Delta,\sigma)$ be such that condition~\eqref{challenging-regime} holds.
Then \asii achieves exact recovery with probability at least $1 - 1/n^2$, if
\[
T \gtrsim \frac{\sigma^2 \tilde n \ln n}{\Delta^2} \ .
\]
In particular, in the from-scratch case $\tilde n = n$, the minimax-optimal number of observations required for exact recovery with probability at least $1 - 1/n^2$, satisfies
$T^{\star} \asymp \frac{\sigma^2 n \ln n}{\Delta^2}$.
\end{cor}

Thus, in active seriation from scratch, \asii\ attains the minimax-optimal sample complexity $T^{\star}$, which depends transparently on the problem parameters $(\Delta, \sigma, n)$.
Crucially, \asii\ can achieve high probability recovery with a number of queries $T \ll n^2$, highlighting a substantial advantage over the classical batch setting where all $n^2$ pairwise similarities are observed.

%%%%%%%%%%%%%%%%

\subsection{Extension beyond uniform separation}
\label{sec:extension}

Our analysis so far has focused  on exact recovery under the uniform separation assumption $M \in \cM_\Delta$.
While this assumption is natural for characterizing the fundamental limits of exact recovery, it is also idealized: in practice, some items may be nearly indistinguishable in terms of their pairwise similarities, making their relative ordering statistically impossible to recover.

We therefore consider arbitrary pre-R matrices $M$, as defined in model~\eqref{model}, without any separation assumption.
Even in this setting, a direct extension of the \asii\ procedure continues to provide meaningful and statistically optimal guarantees: it correctly recovers the relative ordering of items that are sufficiently well separated.
%This parameter should not be confused with the model parameter $\Delta$ appearing in the uniform separation assumption $M \in \mathcal M_\Delta$ used in Section~\ref{section-performance-analysis}.

%%%%%%%%%%%%%%%

\paragraph{Algorithmic extension of ASII.}
The procedure follows the same iterative insertion scheme as our \asii algorithm, but takes as input a tolerance parameter $\tilde \Delta > 0$, which sets the resolution at which the algorithm attempts to distinguish items.
Now, when \asii\ identifies a candidate insertion location for a new item in the current ordering, this location is checked through an additional high-probability validation test at precision $\tilde\Delta/2$.
The item is inserted if the test succeeds; otherwise it is discarded and does not appear in the final output. 
%Here $\tilde\Delta$ denotes a user-chosen tolerance parameter, which sets the resolution at which the algorithm attempts to distinguish items.

As a consequence, this extension no longer outputs a permutation $\hat{\pi}$ of all $n$ items.
Instead, it returns a subset $S \subset [n]$ together with a rank map $\hat \pi : S \to \{1,\ldots,|S|\}$, which serves as an estimator of the relative ordering of the items in $S$.
The corresponding pseudo-code is deferred to Appendix~\ref{appendix:extension-beyond-unif-speration}.

%%%%%%%%%%%%%%%%%%%%%%%%%%%%%%%%%%%%

\paragraph{Robustness beyond uniform separation.}
We now state robustness guarantees for the extension described above.
To this end, we introduce a notion of $\Delta$-maximality, inspired by the class $\mathcal M_\Delta$ defined in~\eqref{model-set}.
Specifically, for any subset $S \subset [n]$, we write $M_S$ for the submatrix of the similarity matrix $M$ restricted to the items in $S$.
By a slight abuse of notation, we write $M_S \in \mathcal M_{\Delta}$ to mean that $M_S$ is pre-R and satisfies $\Delta_{M_S} \ge \Delta$, where  $\Delta_{M_S}$ is the minimal gap defined in~\eqref{minimal-gap-def}.

\begin{defi}[$\Delta$-maximal subset]
\label{def:Delta-maximal}
For any $\Delta>0$, a subset $S \subset [n]$ is said to be $\Delta$-maximal  if
$M_S \in \mathcal M_{\Delta}$
and
$M_{S \cup \{k\}} \notin \mathcal M_{\Delta}$ for all $k \in [n] \setminus S$.
\end{defi}
Intuitively, a $\Delta$-maximal subset cannot be enlarged without adding items that are too similar (within~$\Delta$) to those already in $S$.

We now show that the guarantee of Theorem~\ref{thm:UB:insert-n2-into-n1} can be extended beyond the uniformly separated setting, to any pre-R matrix $M$, by operating at a user-chosen resolution level $\tilde\Delta$. 
The conditions on the parameters are the same as in Theorem~\ref{thm:UB:insert-n2-into-n1}, with $\Delta$ replaced by the input parameter $\tilde\Delta$; that is:
\begin{equation}
\label{condtion-parameter-extension}
    \frac{\tilde \Delta}{\sigma} \leq 1 , \qquad \quad  \frac{\tilde \Delta^2 T}{\sigma^2 \tn} \ge c \ln n \ .
\end{equation}

\begin{thm}[Guarantees beyond uniform separation]
\label{thm:UB-extension}
There exist absolute constants $c,c'>0$ such that the following holds. Let $n \ge 3$, $\tn \in [n]$ and the input permutation $\tpi$ of $[n-\tn]$ satisfy~\eqref{new-relative-order-incremental}.
Let  $(\tilde\Delta,\sigma,T)$ be such that~\eqref{condtion-parameter-extension} holds.
Then, for any pre-R matrix $M \in \R^{n \times n}$, the algorithmic extension of \asii\ outputs, with probability at least
$1 - \exp\!\left(
-\, c' \tfrac{\tilde\Delta^2 T}{\sigma^2 \tilde n}
\right)$,
 a $\tilde\Delta$-maximal subset $S \subset [n]$  in the sense of Definition~\ref{def:Delta-maximal}, and
 a correct ordering $\hat{\pi}_S$ of the items in $S$.
\end{thm}

The proof is given in Appendix~\ref{appendix:extension-beyond-unif-speration}.

%%%%%%%%%%%%
%%summary of empirical expe
%%%%%%%%

\section{Empirical results}
\label{sec:empirical}

We illustrate the behavior of \asii through numerical experiments and a real-data example.

\paragraph{Numerical simulations.}
We assess the empirical behavior of \textsc{ASII} on synthetic data and compare it to three benchmark methods:
(i) the batch seriation algorithm \textsc{Adaptive Sampling}~\cite{cai2022matrix},
(ii) the batch \textsc{Spectral Seriation}~\cite{atkins1998spectral},
and (iii)  an active, naive insertion variant of \textsc{ASII} without backtracking.
Since, to the best of our knowledge, the literature does not contain seriation methods designed for the active setting, these three procedures serve as reference points.
All methods are evaluated under identical sampling budgets on four representative scenarios, covering both homogeneous (Toeplitz) and non-homogeneous (non-Toeplitz), Robinson matrices.
Full experimental details are deferred to Appendix~\ref{sec:exp}.

Figure~\ref{fig:scenario-main} provides a visual illustration of the four scenarios considered.
Scenario (1) corresponds to a Toeplitz Robinson structure, while the remaining scenarios (2-3-4) depart from the Toeplitz assumption and exhibit more heterogeneous Robinson geometries.

\begin{figure}[htbp]
    \centering
    \includegraphics[width=0.57\linewidth]{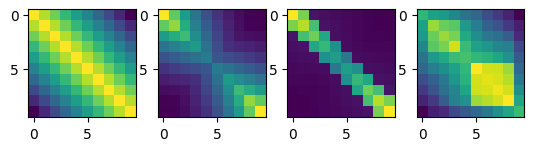}
    \caption{{{\small Robinson matrices for scenarios (1)-(4), from left to right.}}}
    \label{fig:scenario-main}
\end{figure}

Figure~\ref{fig:simer-main} reports the empirical probability of error of all methods as a function of the minimal gap~$\Delta$.
Across all four scenarios, \textsc{ASII} consistently outperforms the naive iterative insertion procedure, highlighting the benefits of its backtracking corrections; 
these gains are consistent with the logarithmic improvement predicted by our theoretical analysis.

As expected, in scenario~(1), where the underlying matrix is Toeplitz, \textsc{ASII} performs below the two batch methods, which are known to perform well in this setting.
This behavior can be attributed to the fact that \textsc{ASII} is designed for general Robinson matrices and does not exploit Toeplitz regularity (unlike its competitors);
moreover, its sampling budget is distributed across multiple binary-search iterations, which can be less efficient than batch sampling.
This does not contradict our theoretical guarantees, which establish rates up to absolute constants that are not characterized by the analysis and may be large for active, iterative
procedures such as \textsc{ASII}.

In contrast, in the heterogeneous (non-Toeplitz) scenarios~(2-3-4), \textsc{ASII} remains consistently accurate, whereas both batch methods exhibit unstable behavior and may fail entirely in some cases.

Overall, these experiments illustrate that \textsc{ASII} maintains stable empirical performance across various scenarios, and can be particularly effective on matrices with localized variations (scenarios~2-3-4), where batch methods tend to struggle.

\begin{figure}[htbp]
\begin{minipage}[b]{0.50\linewidth}
    \centering
    \includegraphics[width=0.57\linewidth]{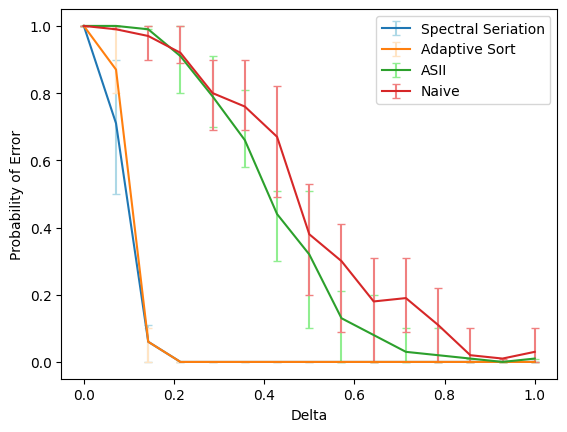}  
\end{minipage}\hfill
\begin{minipage}[b]{0.50\linewidth}
    \centering
    \includegraphics[width=0.57\linewidth]{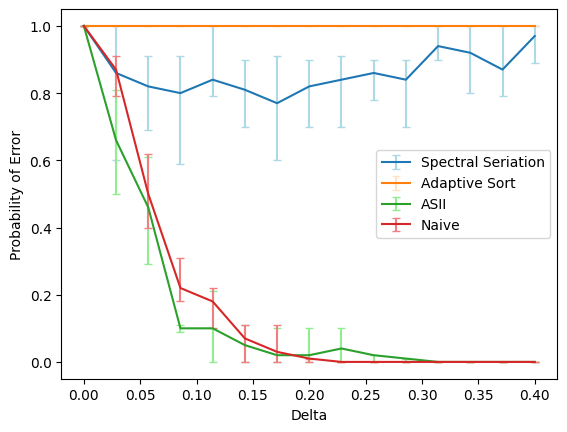}
\end{minipage}\hfill
\begin{minipage}[b]{0.50\linewidth}
    \centering
    \includegraphics[width=0.60\linewidth]{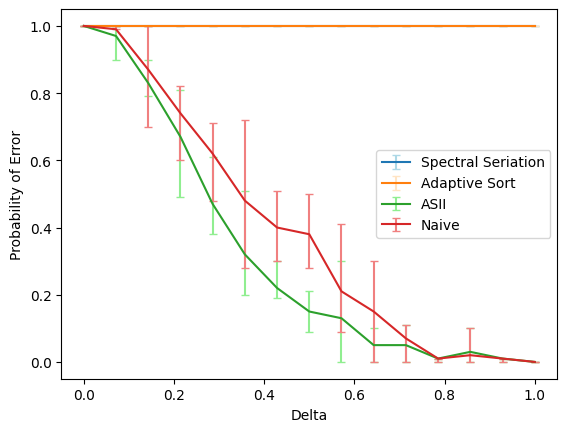}
\end{minipage}\hfill
\begin{minipage}[b]{0.50\linewidth}
    \centering
    \includegraphics[width=0.57\linewidth]{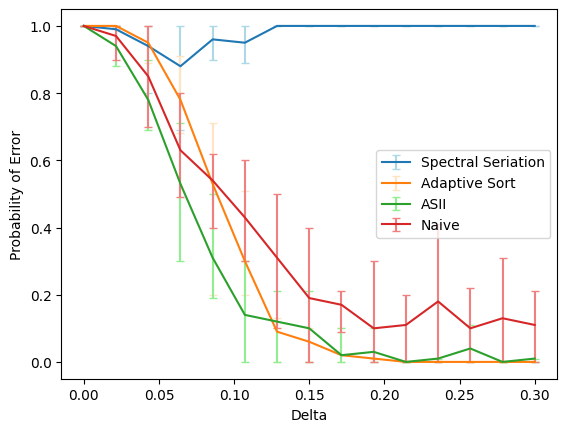}
\end{minipage}
\caption{ 
Empirical error probabilities  for \textsc{ASII} and  three benchmark methods as the  parameter $\Delta$ varies. Scenarios (1-2-3-4) are displayed from left to right and top to bottom. Each experiment uses $n=10$ items and $T=10{,}000$ observations. For each value of~$\Delta$, 100 Monte Carlo runs are split into $10$ equal groups; 
error bars show the $0.1$ and $0.9$ quantiles of the empirical error across groups.}
\label{fig:simer-main}
\end{figure}  

%%%%%%%%%%%

\paragraph{Application to real data.}
We further assess the robustness of \textsc{ASII} on real single-cell RNA sequencing data (human primordial germ cells, from~\cite{guo2015transcriptome}, previously analyzed by~\cite{cai2022matrix}). 
Although such biological data depart substantially from the idealized Robinson models assumed in our theory, \textsc{ASII} still produces a meaningful reordering of the empirical similarity matrix, revealing coherent developmental trajectories among cells. 
This example highlights the potential practical relevance of the proposed approach beyond the stylized assumptions of our theoretical framework.
Full experimental details are provided in Appendix~\ref{sec:exp}.

\begin{figure}[htbp]
\begin{minipage}[b]{0.50\linewidth}
    \centering
    \includegraphics[width=0.60\linewidth]{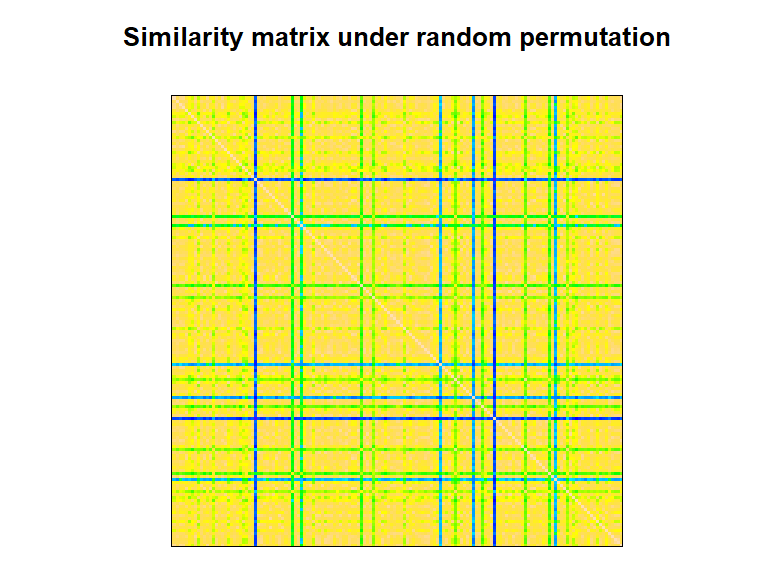}  
\end{minipage}\hfill
\begin{minipage}[b]{0.50\linewidth}
    \centering
    \includegraphics[width=0.60\linewidth]{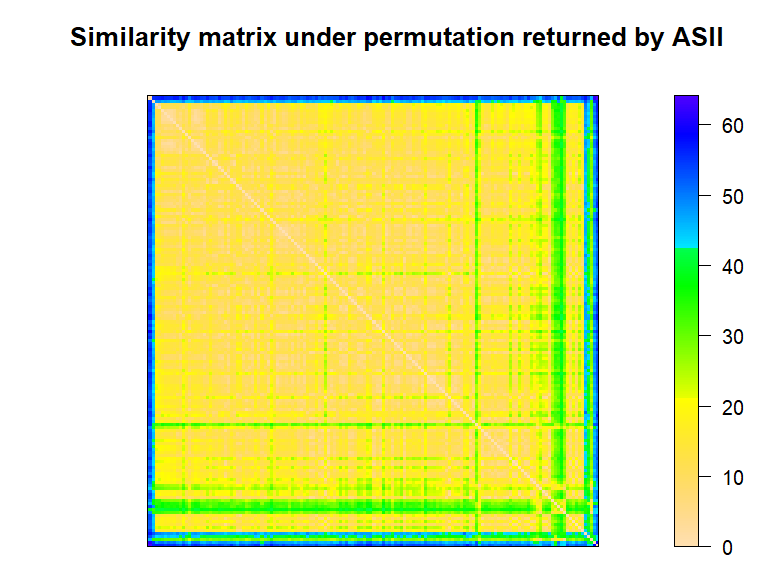}
\end{minipage}
\caption{
Similarity matrix $M$ of a single-cell RNA-seq dataset before and after reordering by ASII.  
The recovered ordering reveals a clear block-diagonal structure consistent with developmental progression: dissimilar regions (blue) are pushed to the boundaries, while groups of highly similar cells (yellow and green) align along the diagonal.
}
\label{fig:rna-main}
\end{figure}

%%%%%%%%%%%%%

%%%%%%%%%%%%%%%%
% Disucssion
%%%%%%%%%%%%%%%%%%%
\section{Discussion}

This work introduces an active-learning formulation of the seriation problem, together with sharp theoretical guarantees and a simple polynomial-time algorithm.
We characterize a phase transition in sample complexity governed by the {{\small $\mathrm{SNR} = \Delta^2 T / (\sigma^2 n)$}}: recovery is impossible when {{\small $\mathrm{SNR} \lesssim \ln n$}}, while \asii\ achieves near optimal performances once this threshold is exceeded.
Our analysis highlights how adaptive sampling combined with corrective backtracking can substantially improve statistical efficiency.
We now turn to a discussion of several complementary aspects of the problem.

\paragraph{Noise regimes.}
Our analysis focused on the stochastic regime where the per-observation signal-to-noise ratio $\Delta/\sigma$ is at most~$1$, which captures the most challenging setting for active seriation. 
The results, however, extend naturally to less noisy regimes ($\Delta/\sigma>1$): in that case, accurate recovery  requires only $T \gtrsim n\ln n$ queries, reflecting the intrinsic $O(n\ln n)$ cost of performing $n$ adaptive binary insertions. 
Further details are provided in Appendix~\ref{discussion:non-stochastic-regime}.

\paragraph{Gain from active learning.}
Our active framework enables recovery of the underlying ordering without observing the entire similarity matrix.
Whereas batch approaches require {{\small $O(n^2)$}} observations, our active algorithm \asii\ succeeds with only {{\small $T \gtrsim (\sigma/\Delta)^2 n \ln n$}} samples.
This corresponds to a fraction $\tfrac{\ln n }{(\Delta/\sigma)^2 n}$ of the full matrix and yields a substantial reduction in sample complexity, 
as long as {{\small $\Delta/\sigma$}} is not too small.
This gain arises from the ability of adaptive sampling to draw information from a well-chosen, small subset of pairwise similarities, from which the entire matrix can be reordered, achieving strong statistical guarantees under limited sampling budgets.

In certain scenarios, when differences between candidate orderings are highly localized, our active algorithm can succeed under weaker signal conditions than some batch methods, see Appendix~\ref{sec:as} for a detailed comparative example.

\paragraph{Fixed-budget formulations.} Throughout this work, we focused on the fixed-budget setting, where the total number of samples~$T$ is fixed in advance and the objective is to minimize the error probability within this budget. 
This way, the algorithm does not require prior knowledge of the noise parameter~$\sigma$, nor of the minimal signal gap~{{\small $\Delta_M$}} of the latent matrix {{\small $M$}}; it simply allocates the available budget~{{\small $T$}} across tests. 
Yet, its performance depends on the unknown~$\sigma$ and~$\Delta_M$ through the signal-to-noise ratio {{\small $\mathrm{SNR}_M = \Delta_M^2 T / (\sigma^2 n)$}}, which determines the achievable accuracy.

\paragraph{Potential applications.}
Seriation techniques are broadly relevant in domains where pairwise similarity information reflects a latent one-dimensional structure.
Examples include genomic sequence alignment, where seriation helps reorder genetic fragments by similarity, and recommendation systems, where item-item similarity matrices can reveal latent preference orderings.
We also illustrated, on real single-cell RNA sequencing data, that \textsc{ASII} can recover biologically meaningful trajectories despite the data departing strongly from our theoretical model.
These settings often involve noisy or costly pairwise measurements, for which active seriation provides an appealing alternative to batch reordering methods.

\paragraph{Future directions.}
We also proposed an extension to settings where the uniform separation assumption does not hold, in which \asii outputs an ordering of a subset of well-separated items. 
How to exploit more global ordering structure beyond this setting is an important open question.

Another natural extension is to study the fixed-confidence setting, where the algorithm must adaptively decide when to stop sampling in order to achieve a prescribed confidence level. 
Such a formulation would typically require variance-aware sampling policies and data-driven stopping rules, possibly involving online estimation of~$\sigma$. 
Developing such an adaptive, fixed-confidence version of~\asii\ is an interesting avenue for future work.

%%%%%%%%%%

%%%%%%%%%%%%%

\subsubsection*{Acknowledgements}
The work of J.~Cheshire is supported by the FMJH, ANR-22-EXES-0013. 
\newpage

%%%%%%%%%%%%%%%%
%\input{unfinished}

\bibliography{biblio_seriation}

\bibliographystyle{apalike}

\newpage

%%%%%%%%%%%%%%%%%%%%%%%%%%

\appendix

% \crefalias{section}{appendix} % uncomment if you are using cleveref

%%%%%%%%%%
% Additional discussion
%%%%%%%%%%

\section{Additional discussion}

\subsection{Beyond the stochastic regime $\Delta/\sigma \leq 1$}
\label{discussion:non-stochastic-regime}
To simplify the presentation of our results, we have primarily focused on the challenging regime where the signal-to-noise ratio per observation, $\Delta/\sigma$, is at most $1$. In this subsection, we discuss how our findings extend to less noisy regimes where $\Delta/\sigma$ may exceed $1$.

In fact, our upper bound can be stated in a slightly more general form than that of Theorem~\ref{thm:UB:insert-n2-into-n1}, covering all values of $\Delta/\sigma$. Specifically, we show that if
\begin{equation}
\label{general-condition-ccl}
\ln n \lesssim \left(\frac{\Delta^2}{\sigma^2} \wedge 1\right) \frac{T}{\tilde n}\,,
\end{equation}
then for any $M \in \cM_{\Delta}$, the error probability of \asii satisfies
\[
p_{M,T} \leq  \exp\left(-c \frac{\Delta^2 T}{\sigma^2 \tilde n}\right)
\]
for some absolute constant $c > 0$. When $\Delta/\sigma \leq 1$, this recovers the bound stated in Theorem~\ref{thm:UB:insert-n2-into-n1}. 
In contrast, when $\Delta/\sigma \geq 1$, the condition~\eqref{general-condition-ccl} reduces to $T \gtrsim n \ln n$, which reflects the fact that \asii performs $O(n)$ binary insertions, each requiring up to $O(\ln n)$ queries.

Whether the lower bound of Theorem~\ref{thm:LB2} extends to the entire complementary regime, where the inequality in~\eqref{general-condition-ccl} is reversed, is an open and more delicate question. Nevertheless, computational lower bounds are well-established for related settings. For instance, in the case from scratch ($\tilde n = n$), any comparison-based sorting algorithm (with noiseless comparisons) requires at least $O(n \ln n)$ comparisons in the worst case~\cite[Section 8.1]{book-intro-algo}. More generally, such lower bounds apply to broader algorithmic classes, such as bounded-degree decision trees or algebraic computation trees~\cite{algebraic-computaiton-trees}. Taken together, these computational barriers and the information-theoretic lower bound from Theorem~\ref{thm:LB2} provide a more comprehensive understanding of the fundamental limitations of the problem. In particular, they support condition~\eqref{general-condition-ccl} as a natural threshold for the success of sorting algorithms modeled as bounded-degree decision trees.

%%%%%%%%%%%%

\subsection{Comparison with batch methods: illustrative example}
\label{sec:as}

The seriation problem has mainly been studied in non-adaptive settings, notably by~\cite{cai2022matrix}, who analyze exact reordering of Robinson Toeplitz matrices from a single noisy observation of the full matrix. 
There exist classes of instances for which their guarantees require a signal-to-noise ratio growing polynomially with~$n$, whereas our active approach succeeds under dramatically weaker conditions, with an effective threshold of order $\sqrt{(\ln n)/n}$. 
An illustrative example of this gap is presented below (Appendix~\ref{appendix-tony-cai-comparison}). 
This contrast stems from a fundamental difference in design: our active algorithm concentrates sampling effort on locally ambiguous regions, while the batch analysis of~\cite{cai2022matrix} relies on global matrix discrepancies,  naturally leading to stronger separation requirements in scenarios involving highly localized differences between orderings.

\subsection{Theoretical comparison with \textsc{Adaptive-Sorting} from~\cite{cai2022matrix}}
\label{appendix-tony-cai-comparison}

The exact reordering of Robinson Toeplitz matrices has recently been analyzed in the batch (non-adaptive) setting by~\cite{cai2022matrix}.  
In their framework, the learner observes a single noisy realization of the entire similarity matrix,
\[
Y_{ij}=R_{\pi^{\star}_i\,\pi^{\star}_j}+\sigma Z_{ij}, \qquad 1\le i,j\le n,
\]
where $R$ is a Robinson Toeplitz matrix and $Z$ is a symmetric noise matrix.  
The goal is to recover the correct reordering of~$R$, equivalently the permutation~$\pi^{\star}$.

The authors introduce a signal-to-noise ratio (SNR) over a parameter space $\mathcal{R} \times \mathcal{S}$,  
where $\mathcal{R}$ is a class of Robinson Toeplitz matrices and $\mathcal{S}$ a set of permutations.  
It is defined as
\[
\rho(\mathcal{R}\times\mathcal{S}) = 
\min_{R\in\mathcal{R},\,\pi,\pi'\in\mathcal{S}}
\|R_\pi - R_{\pi'}\|_F,
\]
where $\|\cdot\|_F$ denotes the Frobenius norm.  
Their main result shows that the \textsc{Adaptive-Sorting} algorithm exactly reorders the matrix with high probability whenever
\begin{equation}
\label{cai-guarantees}
    \rho(\mathcal{R}\times\mathcal{S}) \gtrsim \sigma n^2.
\end{equation}

Although this result is not directly comparable to ours ---since it pertains to the batch observation model--- it is instructive to focus on a simple case where the two settings overlap.  
Let $\mathcal{R}_\Delta$ denote the set of Robinson Toeplitz matrices with minimal gap~$\Delta$, and let $\mathcal{S}_n$ be the set of all permutations of~$[n]$.  
Assume moreover that the number of queries in our active framework equals $T=n^2$, corresponding to the batch regime.

\paragraph{A concrete example.}
Consider the Toeplitz matrix $R_{ij}=\Delta(n-|i-j|)$, and compare the identity permutation $\pi_{\mathrm{id}}$ with the transposition~$(1,2)$.  
The two permuted matrices differ only in their first two rows and columns, and a direct calculation gives
\[
\|R_{\pi_{\mathrm{id}}}-R_{(1,2)}\|_F \asymp \Delta \sqrt{n}.
\]
Substituting into~\eqref{cai-guarantees} yields the sufficient condition
\[
\Delta \gtrsim \sigma n^{3/2}.
\]
Hence, in this instance, \textsc{Adaptive-Sorting} requires the signal gap~$\Delta$ to grow polynomially with~$n$.

By contrast, our Theorem~\ref{thm:UB:insert-n2-into-n1} (in the batch case $\tilde n = n$) guarantees exact recovery by \asii\ under
\[
T \gtrsim \frac{\sigma^2 n \ln n}{\Delta^2},
\]
which, for $T=n^2$, becomes
\[
\Delta \gtrsim \sigma \sqrt{\frac{\ln n}{n}}.
\]
For this restricted class of matrices and permutations, the required signal level is therefore exponentially smaller in~$n$ than in the batch setting, revealing the potential statistical benefit of adaptive sampling.

\paragraph{Remark.}
The above comparison relies on a specific, highly localized perturbation of the permutation: the identity versus the transposition~$(1,2)$.  
In this situation, the two hypotheses differ only through a few rows and columns, creating a local variation in the similarity structure.  
Our active procedure is particularly suited to detect such localized discrepancies, as it concentrates its sampling effort on uncertain or informative regions.  
In contrast, the batch analysis of~\cite{cai2022matrix} is formulated in terms of global matrix discrepancies (such as Frobenius norms), 
which naturally leads to stronger separation requirements in these localized scenarios.

\subsection{Further related literature on ranking}
\label{section-relat-lit-ranking}

The seriation problem is related, though fundamentally different, from the literature on ranking from pairwise comparisons. In that setting, each pair of items $i, j$ is associated with a probability $p_{i,j}$ that item $i$ beats item $j$, and the learner aims to recover the full ranking from noisy comparisons. A common assumption is stochastic transitivity (ST), which posits the existence of a true underlying order $r_1,\ldots,r_n$ such that $r_i > r_j$ implies $p_{i,j} > 1/2$. This problem has been studied under the fixed-confidence setting (e.g., \cite{pmlr-v70-falahatgar17a, ren2019sample}), as well as under simplified noisy sorting models where $p_{i,j}$ depends only on the relative rank difference (e.g., \cite{gu2023optimal}).

Many ranking algorithms proceed incrementally, inserting items one-by-one into a growing list using a binary search strategy --similar to our \textsc{ASII} procedure. However, in seriation, unlike ranking, pairwise scores do not directly reveal the ordering; rather, the ordering must be inferred from global structural constraints.

Score-based ranking algorithms (e.g., via Borda scores) offer another perspective. These approaches assign each item a score, typically based on its average performance against others, and sort items accordingly. Several active ranking methods (e.g., \cite{heckel2019active, cheshire2023active, cheshire2025active}) use elimination-based sampling strategies based on these scores. However, such techniques are not applicable to seriation, where no intrinsic score is associated with individual items.

Another related problem is the thresholding bandit problem (TBP), where the learner must place a threshold element into a totally ordered set of arms with known monotonic means. The binary search approach in our algorithm is partly inspired by that of \cite{cheshire2021problem}, developed for TBP. Still, seriation poses unique challenges due to the absence of scores and the reliance on relative similarity between pairs of items.

%%%%%%%%%%%%%

%%%%%%%%%%
%pseudo codes of our algos
%%%%%%%%%%%

%%%%%%%%%%%%%%%%%%%%%%%
%%%%%%%%%%%%%%%%%%%%%%%%%%%%%%%%%%%%%%%%%%%%%%%%%%%%%%%%%%%%%%
\section{Algorithmic details}
\label{app:pseudocode}
%%%%%%%%%%%%%%%%%%%%%%%%%%%%%%%%%%%%%%%%%%%%%%%%%%%%%%%%%%%%%%

To complement Section~\ref{algo-section}, we provide here the full pseudocode of the main procedure \textsc{ASII} and of the two subroutines, \textsc{Test} and \textsc{Binary \& Backtracking Search}.

\textsc{Notation.} Throughout the appendices, we denote by $\pi^*$ the true latent ordering. The permutation constructed by the algorithm is denoted by $\pi$.

\paragraph{Subroutine \textsc{Test}.}
This subroutine compares three items $(k,l,r)$ and determines whether
$k$ lies to the left, in the middle, or to the right of $(l,r)$,
based on noisy empirical similarities.
It forms the basic local comparison rule used throughout the procedure.

\smallskip 

{{\small\underline{\textbf{Subroutine}} \textsc{Test}}}
 {{\footnotesize
     \label{algo:test}
    \begin{algorithmic}[1]

        \REQUIRE  $(k, l, r, T_0)$
        \ENSURE \ $b\in \{-1, 0, 1\}$
        \STATE \label{sample} Sample $\lfloor T_0/3 \rfloor$ times each  of the pairs $\{l, r\}$, $\{k, l\}$, and $\{k, r\}$. Denote the respective sample means by $\hat M_{lr}$, $\hat M_{kl}$, and $\hat M_{kr}$.
        \STATE \textbf{if} $\hat M_{lr} < \hat M_{kl} \wedge \hat M_{kr}$ \
        \textbf{then} \ $b = 0$
          \STATE \textbf{else if}  $\hat M_{kl } > \hat M_{kr}$ \ 
                    \textbf{then} \  $b=-1$
             \STATE \textbf{else} \ $b=1$
    \end{algorithmic}
    }}

\medskip

\paragraph{Main procedure \textsc{Active Seriation by Iterative Insertion} (ASII).}
The algorithm inserts items one by one into an ordered list, using \textsc{Test} to compare
relative positions and \textsc{BBS} to determine insertion points.

If $n-\tn \ge 3$, the algorithm is initialized with a given permutation $\tpi = (\tpi_1,\ldots,\tpi_{n-\tn})$ of the items $\{1,\ldots,n-\tn\}$.
If $n-\tn \in \{0,1,2\}$, any such initial permutation $\tpi$ provides no information on the latent ordering 
(even for $n-\tn = 2$, as the ordering is only identifiable up to reversal).
Accordingly, in this case, the ASII procedure takes no input and is initialized by constructing an arbitrary permutation of the items $\{1,2\}$.

Recall that here, $\pi^*$ denotes the true latent ordering, while $\pi$ denotes the permutation constructed by the algorithm.

\smallskip 

{{\small\underline{\textbf{Procedure}} \textsc{Active Seriation by Iterative Insertion} (ASII)}} \label{algo:asii}
{{\footnotesize
\begin{algorithmic}[1]
    \REQUIRE $\tpi = (\tpi_1,\ldots,\tpi_{n-\tn})$ a permutation of $[n-\tn]$ if $n-\tn \ge 3$;  no input otherwise
    \ENSURE $\pi =  (\pi_1,\ldots,\pi_n)$ an estimator of $\pi^*$
    \STATE \textbf{if} $n-\tn \leq 2$ \ 
                    \textbf{then}
            \STATE \quad Initialize the permutation $\pi^{(2)} =  (1,2)$ where $\pi_1^{(2)}= 1$ and $\pi_2^{(2)} = 2$ \label{line-initialization-ASII-permut}
            \STATE \textbf{else}
            \STATE \quad Initialize the permutation $\pi^{(n-\tn)}$ with $\pi^{(n-\tn)}_i = \tpi_i$ for all $i \in [n-\tn]$    \label{new-n1-line-initialization-ASII-permut}
            \STATE  $k_0 = \max (2, n-\tn)$ \label{new-n1-k0-line-initialization-ASII-permut}
    \FOR{$k= k_0 + 1,\ldots,n$ }
        \STATE \label{extremities} Choose $(l^{(k-1)}, r^{(k-1)}) \in [k-1]^2$ such that $(\pi^{(k-1)}_{l^{(k-1)}}, \pi^{(k-1)}_{r^{(k-1)}}) = (1, k-1)$
        \STATE \label{test-in-pi} $b = \textsc{Test}(k, l^{(k-1)}, r^{(k-1)}, \lfloor T / (3 \tn) \rfloor)$ 
        \STATE \textbf{if} $b = -1$ \ 
                    \textbf{then}
            \STATE \quad  $\pi_k^{(k)} = 1$, and set $\pi_i^{(k)} = \pi_i^{(k-1)} + 1$ for all $i \in [k-1]$ \label{insert-outside-pi} 
        \STATE \textbf{else if} $b = 1$ \ 
                    \textbf{then}
            \STATE \quad $\pi_k^{(k)} = k$, and set $\pi_i^{(k)} = \pi_i^{(k-1)}$ for all $i \in [k-1]$ \label{insert-outside-pi-bis}
        \STATE \textbf{else}
            \STATE \quad \label{insert-in-pi} $\pi_k^{(k)} = \textsc{Binary \& Backtracking Search}(\pi^{(k-1)})$
            \STATE \quad \label{insert-in-pi-update} Set $\pi_i^{(k)} = \pi_i^{(k-1)}$ for all $i$ such that $\pi_i^{(k-1)} < \pi_k^{(k)}$
            \STATE \quad \label{insert-in-pi-update-bis} Set $\pi_i^{(k)} = \pi_i^{(k-1)} + 1$ for all $i$ such that $\pi_i^{(k-1)} \geq \pi_k^{(k)}$
    \ENDFOR
    \STATE $\pi = \pi^{(n)}$        
    \end{algorithmic}
    }}

At each iteration $k = k_0 + 1,\ldots,n$, the algorithm maintains the current permutation 
$\pi^{(k-1)}$ of the first $k-1$ items.
The pair $(l^{(k-1)},r^{(k-1)})$ in lines~\ref{extremities}–\ref{test-in-pi}
corresponds to the leftmost and rightmost elements of $\pi^{(k-1)}$.
If the initial test indicates that $k$ lies outside the current range,
$k$ is inserted as the first or last element
(lines~\ref{insert-outside-pi}–\ref{insert-outside-pi-bis}).
Otherwise, the algorithm calls the \textsc{Binary \& Backtracking Search} subroutine
(line~\ref{insert-in-pi}) to locate the correct insertion point.
Lines~\ref{insert-in-pi-update}–\ref{insert-in-pi-update-bis}
then update the permutation indices accordingly.

%%%%%%%%%%
\medskip 

\paragraph{Subroutine \textsc{Binary \& Backtracking Search}.}
This subroutine performs a noisy binary search with backtracking.
It maintains a stack of explored intervals and revisits previous ones when inconsistencies
are detected, ensuring robustness without increasing the total sampling budget.

\smallskip 

{{\footnotesize  \label{algo:BBS}
\begin{algorithm}[H]
\begin{algorithmic}[1]
\REQUIRE {{\small $(\pi_1,\ldots,\pi_{k-1})$ a permutation of $[k-1]$}}
\ENSURE {{\small $\pi_{k} \in[k]$}}
\STATE {{\small Set  $(l_0, r_0)\in [k-1]^2$ s.t.  $(\pi_{l_0},\pi_{r_0}) = (1, k-1)$; initialize  $L_0 = [(l_0,r_0)]$;  set  $T_k = 3 \lceil   \log_2 k \rceil$}} \label{init-l0-BBS}
\FOR{{{\small $t= 1, \ldots, T_k$}}}
     \STATE {{\small \textbf{if}  $|L_{t-1}| \geq 1$, and  \textsc{Test}$\left(k, l_{t-1}, r_{t-1}, \lfloor T/(3 \tn T_k) \rfloor\right) \neq 0$  \ 
                    \textbf{then} }} \label{sanity-check-passed}
        \STATE  \quad  {{\small Remove last element: $L_t = L_{t-1} \setminus{\{ L_{t-1}[-1]\}}$ }} \label{del-list} % \hfill \COMMENT{remove last element of the list}
        \STATE  \quad {{\small Set  $(l_{t}, r_{t}) =  L_{t}[-1]$ }} \label{backtrack}   % \hfill \COMMENT{backtracking}
    \STATE {{\small \textbf{else}  }} 
    \STATE \qquad {{\small \textbf{if}  $\pi_{r_{t-1}} - \pi_{l_{t-1}} \leq 1$ \ \textbf{then} set $(l_{t}, r_{t}) = (l_{t-1}, r_{t-1})$}}  \label{binary-s}%\hfill \COMMENT{leaf} 
    \STATE\qquad {{\small \textbf{else}  choose $m_t \in[k-1]$ \ s.t. \ $\pi_{m_t} = \lfloor (\pi_{l_{t-1}} + \pi_{r_{t-1}})/2 \rfloor$}}  %\hfill \COMMENT{binary search}
        \STATE \qquad \qquad {{\small   \textbf{if}   \textsc{Test}$\left(k, l_{t-1}, m_t, \lfloor T/(3 \tn T_k) \rfloor \right)$ = 0 \ \textbf{then} set $(l_{t}, r_{t}) = (l_{t-1}, m_{t})$}} \label{test-left-interval}
         \STATE \qquad \qquad  {{\small \textbf{else} set $(l_{t}, r_{t}) = (m_{t}, r_{t-1})$}} \label{test-right-interval}
    \STATE \qquad {{\small Update list: $L_t = L_{t-1} \oplus [(l_t, r_t)]$}} \label{line-adding}
\ENDFOR 
\STATE {{\small Return $\pi_k = \pi_{l_{T_k}} + 1$}} \label{return-BBS}
\end{algorithmic}
\caption*{{{\small \textbf{Subroutine} \textsc{Binary  \& Backtracking Search} (BBS)}}} % <-- 
\end{algorithm}
}}

The subroutine maintains a list $L_t$ of explored intervals, where 
$L_t[-1]$ denotes the current active interval.
The operator $\oplus$ represents list concatenation, and $|L_t|$  to denote the last index of the list (in particular, \( |L_t| = 0 \) when \( L_t \) contains a single element).
At each iteration, a sanity check (line~\ref{sanity-check-passed})
verifies consistency of the current search path.
If the check fails, the last interval is removed and the algorithm backtracks
(lines~\ref{del-list}–\ref{backtrack}); otherwise, the current interval is split in two,
and the appropriate subinterval is appended to the list
(lines~\ref{binary-s}–\ref{line-adding}).
This corrective mechanism ensures that local inconsistencies do not propagate and that the search remains robust under noisy test outcomes.

%%%%%%%%%%%%%%%%%

%%%%%%%%%%%%%%%%%%%%%
%%%%%%%%%%%%%%%%%%%%%%%

\section{Proof of Theorem~\ref{thm:UB:insert-n2-into-n1} (Upper Bounds)}
\label{app:simple}

Recall that, if $\pi$ is an ordering of the $n$ items (as defined in Section~\ref{section-setting}), then its reverse $\pi^{\mathrm{rev}}$ is also an ordering. In the following, we denote by $\pi^* \in \{\pi,\pi^{\mathrm{rev}}\}$ the true ordering that satisfies $\pi^*_1 < \pi^*_2$. For simplicity, and without loss of generality, we also assume that $\tilde \pi_1 < \tilde \pi_2$. 

% We denote by $\pi^*$ the ordering (of the $n$ items) that satisfies $\pi^*_1 < \pi^*_2$.

For any $2 \leq k \leq n$, we say that a permutation $\pi = (\pi_1, \ldots, \pi_k)$ of $[k]$ is \textit{coherent} with $\pi^*$ if
\begin{align}
\label{relative-order-incremental}
\forall\, 1 \leq i < j \leq k\ : \qquad \spi_i < \spi_j \quad \Longleftrightarrow \quad \pi_i < \pi_j.
\end{align}
In other words, $\pi$ agrees with $\pi^*$ on the relative ordering of items $1,\ldots,k$.

\vspace{0.2cm}
We begin by analyzing the success of the \textsc{Test} subroutine, which is called by the \textsc{ASII} procedure introduced in Section~\ref{algo-section}; its full pseudocode is provided in Appendix~\ref{app:pseudocode}.
For each 
$k \geq 3$, let $\mathcal{A}_k$ denote the event on which, at iteration~$k$, the \textsc{Test} subroutine called by (line~\ref{test-in-pi} of) ASII outputs a correct recommendation :
\begin{equation}\label{good-event-procedure-call-to-test}
\mathcal{A}_k : \qquad b := \textsc{Test}\left(k, l^{(k-1)}, r^{(k-1)}, \lfloor T/(3 \tn) \rfloor\right)  = \left\{
\begin{array}{ll}
-1 & \text{if } \spi_k < \spi_{l^{(k-1)}}, \\
\ \ 0 & \text{if } \spi_k \in (\spi_{l^{(k-1)}}, \spi_{r^{(k-1)}}), \\
\ \ 1 & \text{if } \spi_k > \spi_{r^{(k-1)}}.
\end{array}
\right.
\end{equation}
This event assumes that $\spi_{l^{(k-1)}} < \spi_{r^{(k-1)}}$, which holds when $\pi^{(k-1)}$ is coherent with $\pi^*$ as defined in~\eqref{relative-order-incremental}, since $\pi^{(k-1)}_{l^{(k-1)}} = 1 \ < \  k-1 =  \pi^{(k-1)}_{r^{(k-1)}}$ by construction of $(l^{(k-1)}, r^{(k-1)})$  in line~\ref{extremities} of ASII.

We now state the performance guarantee for the \textsc{Test} subroutine. The proof is deferred to Appendix~\ref{appendix-proba-bounds}.

\begin{prop}
\label{lem:test-call-by-procedure} 
The following holds for any $\sigma > 0$, $\Delta > 0$ and $T \geq 90 \tn$. If $M \in \mathcal{M}_{\Delta}$, and   $\pi^{(k-1)}$ is coherent with $\pi^*$ as in~\eqref{relative-order-incremental}, then  the event $\mathcal{A}_k$ in~\eqref{good-event-procedure-call-to-test} satisfies
\[
\mathbb{P}(\mathcal{A}_k)\ \geq \ 1 - 6\exp\left(-\frac{T\Delta^2}{80 \tn \sigma^2}\right).
\]
\end{prop}

Next, we turn to the performance of the \textsc{BBS} subroutine (introduced in Section~\ref{algo-section}, whose full pseudocode is provided in Appendix~\ref{app:pseudocode}). At iteration $k$, when invoked by (line~\ref{insert-in-pi} of) ASII, it receives as input the permutation $\pi^{(k-1)}$ of $[k-1]$, and performs a binary search to locate the correct position of item $k$. The binary search is successful if:\\
- the final interval \((l_{T_k}, r_{T_k})\) contains the true position of $\pi_k^*$, and\\
- this final interval has length one.\\
Formally, we define the event:
\begin{equation}
\label{ebent-B-k}
\mathcal{B}_k(\pi) := \left\{ \pi_k^* \in (\pi^*_{l_{T_k}}, \pi^*_{r_{T_k}}) \ \text{and} \ \pi_{r_{T_k}} - \pi_{l_{T_k}} = 1 \right\},
\end{equation}
which captures the success of \bbs when applied to an input permutation $\pi = (\pi_1,\ldots,\pi_{k-1})$.

The following proposition guarantees success under appropriate conditions; its proof appears in Appendix~\ref{proof-main-prop-garantees-Algo-1}.

\begin{prop}
\label{lem:auto}
The following holds for any  $\sigma > 0$, $\Delta > 0$ and $T$ such that 
\[
 \frac{T \Delta^2}{ \tn \sigma^2 }  \geq 14400 \log_2 n  \qquad \text{and} \qquad    T\geq 54 \tn \lceil \log_2 n \rceil \, .
 \]
 Assume that $M \in \mathcal{M}_{\Delta}$, and that the input  $\pi = (\pi_1,\ldots,\pi_{k-1})$  to the \bbs subroutine is   coherent with  $\pi^*$ as in~\eqref{relative-order-incremental}, and that $\spi_k  \in ( \spi_{l_0}, \spi_{r_0})$ with $(\pi_{l_0}, \pi_{r_0}) = (1, k-1)$. Then,  the event $\mathcal{B}_k(\pi)$ in~\eqref{ebent-B-k} satisfies
\[\P\ac{\cB_k(\pi)} \ \geq \ 1 -  \exp\left(-\frac{T\Delta^2}{1200 \tn \sigma^2}\right).
\]
\end{prop}

We are now ready to prove Theorem~\ref{thm:UB}, showing that the final output $\pi^{(n)}$ of the \asii procedure equals the true ordering $\pi^*$ with high probability. Since ASII constructs $\pi^{(k)}$ incrementally from $\pi^{(k-1)}$, we proceed by induction on $k$.

%%%%%%%%%%%%%%%%%%%%%%

%%%%%%%%%%%%%%%%%%%%%

\subsection{Proof of Theorem~\ref{thm:UB:insert-n2-into-n1}}

We initialize the induction at 
\[k_0 = \max(2, n-\tn) .\]
If $k_0 = 2$, the initial permutation $\pi^{(2)} = (1,2)$
(see line~\ref{line-initialization-ASII-permut} of the ASII procedure)
is trivially coherent with $\pi^*$ in the sense of~\eqref{relative-order-incremental}.

If $k_0 \neq 2$, then $k_0 = n-\tn$ and so $3 \le n -\tilde n$. In this case, the initial permutation $\pi^{(n-\tn)}$ is given by the input ordering $\tpi$
(line~\ref{new-n1-line-initialization-ASII-permut}) and we set $\pi^* \in \{\pi,\pi^{\mathrm{rev}}\}$ such that $\tilde \pi$ is coherent with $\pi^*$. Thus, $\pi^{(n-\tn)}$ is coherent with $\pi^*$ in the sense of~\eqref{relative-order-incremental}. 

% If $k_0 \neq 2$, then $k_0 = n-\tn$ and the initial permutation $\pi^{(n-\tn)}$ is given by the input ordering $\tpi$
% (line~\ref{new-n1-line-initialization-ASII-permut}).
% Since $\tpi$ is assumed to be coherent with the latent ordering, and since $\pi^*$ was fixed arbitrarily among the two valid orderings (up to reversal), we may assume without loss of generality that $\pi^{(n-\tn)}$ is coherent with $\pi^*$ in the sense of~\eqref{relative-order-incremental}. \ya{True w.l.o.g? the algo dont use a particular  orientation?} \ja{no loss of generality issue, we simpy have $\pi^{(n-\tn)}$ is coherent by definition}

\medskip 
Now, taking $k$ such that $k_0 + 1  \leq  k   \leq  n $ and assuming that $\pi^{(k-1)}$ is coherent with $\pi^*$, we show that, conditionally on event $\cA_k$ in~\eqref{good-event-procedure-call-to-test}, the permutation $\pi^{(k)}$ is coherent with $\pi^*$ with high probability. There are three cases defined by the value of the output $b$ of \test.

\smallskip 

- \textit{Case 1: $b=-1$.} If the call to \test at iteration $k$ returns $b=-1$, then on  $\cA_k$, we have $\spi_{k} < \spi_{l^{(k-1)}}$. Since $\pi^{(k-1)}_{l^{(k-1)}} = 1$, and $\pi^{(k-1)}$ is coherent with $\pi^*$,  we deduce that $\spi_k < \spi_s $ for all $s\in[k-1]$. Therefore,  the permutation $\pi^{(k)}$ defined in line~\ref{insert-outside-pi} of ASII is coherent with $\pi^*$. 

\smallskip 

- \textit{Case 2: $b = 1$.} Similarly, if the call to \test returns $b = 1$, then on $\cA_k$ we have $\spi_k > \spi_{r^{(k-1)}}$. Since $\pi^{(k-1)}_{r^{(k-1)}} = k-1$ and $\pi^{(k-1)}$ is coherent with $\pi^*$, we deduce that $\spi_k > \spi_s$ for all $s \in [k-1]$. Therefore, the permutation $\pi^{(k)}$ defined in line~\ref{insert-outside-pi-bis} of ASII is coherent with $\pi^*$.

For these two cases, we have shown that if $\pi^{(k-1)}$ is coherent, then the new permutation $\pi^{(k)}$ is coherent with probability at least
\begin{equation}
\label{Ak-proba-bound-thm}
    \P\ac{\cA_k} \geq 1 - 6 \exp\left(-\frac{T \Delta^2}{80 \tn \sigma^2}\right),
\end{equation}
where we used Proposition~\ref{lem:test-call-by-procedure} for  $T \geq 90 \tn$.

\medskip 

- \textit{Case 3: $b = 0$.} The ASII procedure (in line~\ref{insert-in-pi}) calls the \bbs subroutine with input $\pi^{(k-1)}$ to determine the position of $k$ in the new permutation $\pi^{(k)}$, defined as
\[
    \pi^{(k)}_k = \pi^{(k-1)}_{l_{T_k}} + 1 \; ,
\]
by the (line~\ref{return-BBS} of)  BBS subroutine.
On event $\cB_k(\pi^{(k-1)})$ defined in~\eqref{ebent-B-k} (with input $\pi = \pi^{(k-1)}$), the permutation $\pi^{(k)}$ (lines~\ref{insert-in-pi} to~\ref{insert-in-pi-update-bis}  of ASII) is coherent with $\pi^*$.

Conditionally on $\cA_k$ in~\eqref{good-event-procedure-call-to-test}, with $b=0$, we have
\[
    \spi_k \in (\spi_{l^{(k-1)}}, \spi_{r^{(k-1)}}).
\]
We readily check that 
\[(l^{(k-1)}, r^{(k-1)}) \ = \ (l_0, r_0)  ,\]
since these two pairs are respectively defined by  $(\pi^{(k-1)}_{l^{(k-1)}}, \pi^{(k-1)}_{r^{(k-1)}}) = (1, k-1)$ in line~\ref{extremities} of ASII, and  by $(\pi_{l_0}, \pi_{r_0}) = (1, k-1)$ in line~\ref{init-l0-BBS} of BBS with input $\pi = \pi^{(k-1)}$. 

Hence
\[
\spi_{k} \in ( \spi_{l_0}, \spi_{r_0}) .
\]

The assumption of Proposition~\ref{lem:auto} is thus satisfied for the input $\pi = \pi^{(k-1)}$, allowing us to apply the proposition directly. Hence, conditionally on $\cA_k$,
\begin{equation}
\label{Bk-proba-bound-thm}
    \P_{|\cA_k}\ac{\cB_k(\pi^{(k-1)})} \geq 1 - \exp\left(-\frac{T \Delta^2}{1200 \tn \sigma^2}\right),
\end{equation}
provided that $\frac{T \Delta^2}{\tn \sigma^2} \geq 14400 \log_2 n$ and $T \geq 54 \tn \lceil \log_2 n \rceil$.

Denoting by $\cB_k^c(\pi^{(k-1)})$ the complement of $\cB_k(\pi^{(k-1)})$, we have
\begin{align*}
    \P\ac{\cB_k^c(\pi^{(k-1)})} &= \P\ac{\cB_k^c(\pi^{(k-1)}) \cap \cA_k} + \P\ac{\cB_k^c(\pi^{(k-1)}) \cap \cA_k^c} \\
    &= \P_{|\cA_k}\ac{\cB_k^c(\pi^{(k-1)})} \cdot \P\ac{\cA_k} \ + \  \P_{|\cA_k^c}\ac{\cB_k^c(\pi^{(k-1)})} \cdot \P\ac{\cA_k^c} \\
    &\leq \P_{|\cA_k}\ac{\cB_k^c(\pi^{(k-1)})}  \ + \ \P\ac{\cA_k^c} \\
    &\leq \exp\left(-\frac{T \Delta^2}{1200 \tn \sigma^2}\right) \ + \ 6 \exp\left(-\frac{T \Delta^2}{80 \tn \sigma^2}\right) \\
    &\leq 7 \exp\left(-\frac{T \Delta^2}{1200 \tn \sigma^2}\right),
\end{align*}
where we used~\eqref{Ak-proba-bound-thm} and~\eqref{Bk-proba-bound-thm}.

We have shown that if $\pi^{(k-1)}$ is coherent, then the new permutation $\pi^{(k)}$ is coherent with probability at least
\[
    1 - 7 \exp\left(-\frac{T \Delta^2}{1200 \tn \sigma^2}\right).
\]

\medskip 

- \textit{Conclusion:} In all three cases, the permutation $\pi^{(k)}$ is coherent with probability at least $1 - 7 \exp\left(-\frac{T \Delta^2}{1200 \tn \sigma^2}\right)$. By induction and a union bound over all iterations $k_0 + 1  \leq  k   \leq  n$ of the \asii procedure, we obtain that $\pi^{(n)}$ is coherent with $\pi^*$
(that is, $\pi^{(n)} = \pi^*$), with probability at least $1 - 7 \tn \exp\left(-\frac{T \Delta^2}{1200 \tn \sigma^2}\right)$. 
Here, we used that the number of iterations is upper bounded by $n - k_0 \le n - (n-\tn) = \tn$, where $k_0 = \max(2,n-\tn)$.

Thus, for
\begin{equation}\label{cosntraint-proof-upper-bound-union-bound}
    \frac{T \Delta^2}{1200 \tn \sigma^2} \geq 2 \ln(7 \tn),
\end{equation}
we recover $\pi^*$ with probability at least $1 - \exp\left(-\frac{T \Delta^2}{2400 \tn \sigma^2}\right)$.

To prove this result, we have used several conditions on the parameters $(n, T, \Delta, \sigma)$ which are all satisfied when
\[
  n \ \geq \ k_0 +1 , \qquad  \frac{\Delta}{\sigma} \leq 1, \qquad  \frac{T \Delta^2}{\tn \sigma^2} \geq 14400  \log_2 n, \qquad \text{and} \quad \eqref{cosntraint-proof-upper-bound-union-bound} ,
\]
for $k_0 = \max(2, n-\tn)$.
Since $\log_2(x) = \frac{\ln(x)}{\ln(2)}$,
the last two conditions are implied by
\[
    \frac{T \Delta^2}{\tn \sigma^2} \geq 16800 \ln n.
\]
Theorem~\ref{thm:UB:insert-n2-into-n1} follows for $c_0 = 16800$.

%%%%%%%%%%%%%
%%%%%%%%%%

\subsection{Proof of Proposition~\ref{lem:auto}}
\label{proof-main-prop-garantees-Algo-1}

\textit{Notations.}\\
Given any permutation $\pi$ of $[n]$, and any integers $k,l,r \in[n]$, we use the notation $k \inpi (l,r)$ to indicate that the position $\pi_k$ belongs to the interval $(\pi_l , \pi_r)$:
\begin{equation}\label{notation-interval-image-by-pi}
    k \inpi (l,r) \quad \text{if} \quad \pi_k \in (\pi_l, \pi_r)\enspace.
\end{equation}
In this case, we say that $(l,r)$ is, with respect to $\pi$, an interval containing $k$. Recall that $(\pi_l,\pi_r)$ denotes the set of integers between $\pi_l$ and $\pi_r$, that is $\{m\in \mathbb{N} : \ \pi_l + 1 \leq m \leq \pi_r - 1\}$.

Any $(l,r)$ is called a ``good'' interval if, $k \inspi (l,r)$, as defined in~\eqref{notation-interval-image-by-pi} with $\pi = \pi^*$. Intuitively, the \bbs subroutine will be  successful if it inserts the new item $k$ into good intervals.

At iteration $t$, the \bbs subroutine builds a list $L_t$ of intervals (see Appendix~\ref{app:pseudocode} for its pseudocode). We define $w_t$ as the index of the last good interval in $L_t$ :
\begin{equation}
\label{wt-last-good-index}
w_t \ = \ \max_{0 \leq w  \leq |L_{t}| } \{w : \ \ k \inspi L_{t}[w]\} \ , \qquad \text{for all } t \in [T_k],
\end{equation}
where $|L_t|$ denotes the largest index in the list $L_t$ (that is, $L_t = \big{[}L_t[0], L_t[1], \ldots,L_t[|L_t|]\big{]}$. Note that $w_t$ is well-defined, since there is at least one good interval in $L_t$ (indeed, the initial interval $(l_0, r_0)$ is a good interval by assumption).

\medskip 

\textit{Proof of Proposition~\ref{lem:auto}.}\\
Let $3  \leq  k  \leq  n$, and let $\pi$ be the permutation given as input to the \bbs subroutine. The length of an interval $(l, r)$ is defined as: $\text{length}(l,r) = \pi_r - \pi_l$. In the next lemma, we show that all intervals $L_t[w]$ with $w \geq \lceil \log_2 k \rceil$ have length $1$, and are equal to each other. The proof is in Appendix~\ref{section-proof-lem-constant-list-beyond-large-index}. We recall that $L_t[-1]$ denotes the last element of the list $L_t$ (that is, $L_t[-1] = L_t[|L_t|]$).

\begin{lem}
\label{constant-list-beyond-large-index}
Let $t \geq 1$. If $|L_t| \geq \lceil \log_2 k \rceil$, then for all $w$ such that $\lceil \log_2 k \rceil \leq w \leq |L_t|$, we have $\textup{length}(L_t[w]) = 1$ and $L_t[w] = L_t[-1]$. 
\end{lem}

Recall that the \bbs subroutine makes $T_k := 3 \lceil \log_2 k \rceil$ iterations (see lines~\ref{init-l0-BBS}-2 of BBS).

If $w_{T_k} \geq \lceil \log_2 k \rceil$ at the final iteration $T_k$, then Lemma~\ref{constant-list-beyond-large-index} implies
\begin{equation}
\label{good-last-interval-of-length-1}
    L_{T_k}[-1] = L_{T_k}[w_{T_k}] \qquad \text{and} \qquad \text{length}(L_{T_k}[-1]) = 1\enspace,
\end{equation}
which means that the last interval of $L_{T_k}$ is a good interval of length $1$. Since the last element of $L_{T_k}$ is $(l_{T_k}, r_{T_k})$ (by construction of $L_{T_k}$ in \bbs), this is equivalent to the occurrence of the event
\begin{equation}
\label{good-last-interval-ccl-propC2}
\left\{ k \inspi (l_{T_k}, r_{T_k}) \ \ \text{and} \ \ \pi_{r_{T_k}} - \pi_{l_{T_k}} = 1 \right\} ,
\end{equation}
which is exactly the event $\mathcal{B}_k(\pi)$ of  Proposition~\ref{lem:auto}.
To conclude the proof, it remains to show that $w_{T_k} \geq \lceil \log_2 k \rceil$ with high probability.

\medskip 

\textit{Proof of $w_{T_k} \geq \lceil \log_2 k \rceil$ with high probability.}\\
We introduce the quantity
\begin{equation}
\label{number-steps}
  N_t \ := \ |L_{t}| + \lceil \log_2 k \rceil - 2w_t \ , \qquad \text{for all } t \in [T_k].
\end{equation}
We will show that $N_{T_k} \leq 0$ with high probability, which implies the desired inequality $w_{T_k} \geq \lceil \log_2 k \rceil$, since $|L_t| - w_t \geq 0$ for all $t$ by definition of $w_t$ in~\eqref{wt-last-good-index}. 

Before proceeding, let us explain the intuition behind $N_t$. It can be interpreted as an upper bound on the number of steps required to complete the binary search, assuming no mistakes are made. Indeed, $|L_t| - w_t$ is the number of backtracking steps needed to return to the last good interval, $L_t[w_t]$, and from there, the number of steps to reach an interval of length $1$ is at most $\lceil \log_2 k \rceil - w_t$. Adding these gives $N_t$.

Intuitively, when a mistake is made at step $t$ of the binary search, $N_t$ increases by $1$ compared to $N_{t-1}$. Conversely, when the correct subinterval is chosen in a forward step, or a mistake is corrected via a backwards step, $N_t$ decreases by $1$.

We first formalize this in the next lemma, showing that $N_t - N_{t-1} \leq 1$. (Proof in Appendix~\ref{appendix:proof-lem-N+1}.)
\begin{lem}
\label{lem:dclose}
If $(l_0, r_0)$ is a good interval, then for all $1 \leq t \leq T_k$, we have $N_t \leq N_{t-1} + 1$. 
\end{lem}

To formalize the effect of correct steps, we define the following ``good'' event: $\cE_t$ is the event on which, at step $t$ of \bbs, the \test subroutine returns correct recommendations (in both possible calls):
\begin{equation}
\label{good-event}
\begin{split}
\cE_t := \Big\{\test(k,l_{t-1}, r_{t-1}, \lfloor  \tfrac{T}{3\tn T_k} \rfloor) = 0 \ \ \text{iff} \ \ \spi_{k} \in (\spi_{l_{t-1}}, \spi_{r_{t-1}}) \Big\} \\
\cap \Big\{\test(k,l_{t-1}, m_t, \lfloor  \tfrac{T}{3\tn T_k} \rfloor) =  0 \ \ \text{iff} \ \ \spi_{k} \in (\spi_{l_{t-1}},\spi_{m_t}) \Big\}\;,
\end{split}
\end{equation}
where ``iff'' stands for ``if and only if''.

We now show that conditionally on $\cE_t$, the number of remaining steps decreases:
\begin{lem}
\label{lem:dsimp}
If $(l_0, r_0)$ is a good interval, then for all $1 \leq t \leq T_k$, on the event $\cE_t$ we have $N_t \leq N_{t-1} - 1$. 
\end{lem}
(Proof in Appendix~\ref{proof-lem-N-1}.)

We now apply Lemmas~\ref{lem:dclose} and~\ref{lem:dsimp}. Using the decomposition $N_t = N_t \mathbf{1}_{\cE_t} + N_t \mathbf{1}_{\cE_t^c}$, we get
\[
N_t \leq (N_{t-1} - 1)\mathbf{1}_{\cE_t} + (N_{t-1} + 1)\mathbf{1}_{\cE_t^c} = N_{t-1} - 1 + 2\mathbf{1}_{\cE_t^c},
\]
for all $t \in [T_k]$. By induction on $t= 1,\ldots, T_k$, we obtain
\begin{equation}
\label{upper-bound-N-T1}
N_{T_k} \leq N_0 - T_k + 2 \sum_{t=1}^{T_k} \mathbf{1}_{\cE_t^c} = \lceil \log_2 k \rceil - T_k + 2 \sum_{t=1}^{T_k} \mathbf{1}_{\cE_t^c},
\end{equation}
since $N_0 = \lceil \log_2 k \rceil$ (as $|L_0| = w_0 = 0$) in~\eqref{number-steps}. Now apply:

\begin{lem}
\label{lem:sum-indicator-functions}
If 
\[\frac{T \Delta^2}{\tn \sigma^2} \geq 14400 \log_2 n \, , \qquad \text{and} \quad T \geq 54\tn \lceil \log_2 n \rceil \, ,\]
then for  $T_k = 3 \lceil \log_2 k \rceil$, we have
\[
\P\left(\sum_{t=1}^{T_k} \mathbf{1}_{\cE_t^c} \geq T_k / 4 \right) \leq \exp\left(-\frac{T \Delta^2}{1200 \tn \sigma^2} \right).
\]
\end{lem}
(Proof in Appendix~\ref{appendix-proba-bounds}.)

With this, we conclude that with probability at least $1 - \exp(-\frac{T \Delta^2}{1200 \tn \sigma^2})$, we have
\begin{equation}
\label{negative-N_t}
N_{T_k} \leq \lceil \log_2 k \rceil - T_k / 2 \leq 0,
\end{equation}
for $T_k = 3 \lceil \log_2 k \rceil$. 
Combining this with~\eqref{number-steps} and the fact that $w_{T_k} \leq |L_{T_k}|$, yields $w_{T_k} \geq \lceil \log_2 k \rceil$ with the same high probability.

\medskip 

\textit{Conclusion.}\\
We have shown that the event~\eqref{good-last-interval-ccl-propC2} occurs with high probability, which completes the proof of Proposition~\ref{lem:auto}.

%%%%%%%%%%%%%%%%%%%%
%%%%%%%%%%%%%%%%%%%
%%%%%%%%%%%%%%%%%%%

%%%%%%%%%%%%%

\subsection{Proof of Lemma~\ref{constant-list-beyond-large-index}}
\label{section-proof-lem-constant-list-beyond-large-index}
%Let $t\geq 1$. %Observe that  the last element of the list $L_t$ is equal to $L_{t}[-1] = (l_{t}, r_{t})$. 

At any step $t \geq 1$, the \bbs subroutine either moves forward in the binary search or backtracks:
\begin{itemize}
    \item \textit{Forward step.}  If $\pi_{r_{t-1}} - \pi_{l_{t-1}} \ge 2$, it performs a binary search step: it selects the midpoint $m_t = \lfloor (l_{t-1} + r_{t-1})/2 \rfloor$, and defines the new interval $(l_t, r_t)$ as either $(l_{t-1}, m_t)$ or $(m_t, r_{t-1})$, depending on the outcome of the \textsc{Test} subroutine. This new interval is then appended to the list $L_{t-1}$, forming $L_t$. If $\pi_{r_{t-1}} - \pi_{l_{t-1}} \le 1$, the interval is of length $1$, and the same interval $(l_{t-1}, r_{t-1})$ is simply duplicated and appended to form $L_t$.
    \item \textit{Backtracking step.} The subroutine removes the last interval from $L_{t-1}$, and sets $L_t$ to the truncated list. The interval $(l_t, r_t)$ is defined as the new last element of $L_t$.
\end{itemize}
Therefore, after $t$ steps, the list $L_t$ can be viewed as the sequence of intervals obtained by performing a standard binary search (of $|L_t|$ steps) without backtracking. In particular, whenever a forward step occurs, each interval strictly refines the previous one:
\begin{equation}
\label{inclusion-property}
\forall\, 1 \leq w \leq |L_t|: \quad L_t[w] \subset L_t[w-1] \enspace.
\end{equation}
We now bound the length of each interval in $L_t$. Since the initial interval $(l_0, r_0)$ satisfies $(\pi_{l_0}, \pi_{r_0}) = (1, k-1)$  by definition, its length equals $\pi_{r_0} - \pi_{l_0} = k - 2$. Then, the interval length is at most
\[
\textup{length}(L_t[w]) \leq \frac{2^{\lceil \log_2 (k-2) \rceil}}{2^w} \vee 1 \enspace.
\]
In particular, for any index $w \geq \lceil \log_2 k \rceil$, we have
\[
\textup{length}(L_t[w]) = 1 \enspace.
\]
By \eqref{inclusion-property}, these intervals are nested and of the same length, so they must be equal:
\[
\forall\, w \in \big(\lceil \log_2 k \rceil,\, |L_t|\big): \quad L_t[w] = L_t[-1] \enspace.
\]

This proves the claim of Lemma~\ref{constant-list-beyond-large-index}.

%%%%%%%%%%%%%%%%%%

\subsection{Proof of Lemma~\ref{lem:dclose}}
\label{appendix:proof-lem-N+1}

Fix $t \in [T_k]$. At step $t$, the \bbs subroutine either backtracks (Case~1) or continues the binary search (Case~2), and updates the list $L_t$ accordingly from the previous list $L_{t-1}$. We analyze these two cases separately.

Before proceeding, we make a key observation about the structure of $L_t$:

\medskip
\noindent \textbf{Observation:} If $L_t[w]$ is a bad interval (i.e., it does not contain the target $k$), then every subsequent interval $L_t[w']$ with $w' \geq w$ is also bad. This follows from the nested inclusion property of the intervals in the list,
\[
L_t[w] \subseteq L_t[w-1],
\]
which holds by definition of the binary search. Since $L_t[0] = (l_0, r_0)$ is a good interval and $w_t$ is defined in~\eqref{wt-last-good-index} as the index of the last good interval in $L_t$, the list $L_t$ consists of consecutive good intervals for $w=0,\ldots,w_t$ followed by bad intervals for $w = w_t + 1,\ldots,|L_t|$:
\begin{equation}
\label{observation}
\forall 0 \leq w \leq w_t: \quad L_t[w] \text{ is good}; \quad \forall w_t + 1 \leq w \leq |L_t|: \quad L_t[w] \text{ is bad.}
\end{equation}

\smallskip

\paragraph{Case 1: Backtracking step.}  
At step $t$, the algorithm backtracks by removing the last element of $L_{t-1}$, so
\[
L_t = L_{t-1} \setminus \{L_{t-1}[-1]\} \quad \Rightarrow \quad |L_t| = |L_{t-1}| - 1 \enspace.
\]
Both $L_{t-1}$ and $L_t$ satisfy the observation~\eqref{observation}, so their good and bad intervals are arranged consecutively. We claim that
\[
w_{t-1} - 1 \leq w_t \enspace.
\]
To see this, note:
- If the last interval $L_{t-1}[-1]$ is good, then  $L_{t}[-1]$ is also good, and
\[
w_{t-1} = |L_{t-1}| = |L_t| + 1 = w_t + 1,
\]
so $w_t = w_{t-1} - 1$.

- If $L_{t-1}[-1]$ is bad, then removing it does not change the last good interval, so
\[
w_t = w_{t-1}.
\]
In both cases, $w_t \geq w_{t-1} - 1$ holds.

Using the definition of $N_t$ from~\eqref{number-steps},
\[
N_t = |L_t| + \log_2(k) - 2 w_t,
\]
we have
\[
N_t \leq (|L_{t-1}| - 1) + \log_2(k) - 2(w_{t-1} - 1) = N_{t-1} + 1.
\]

\medskip

\paragraph{Case 2: Continuing the binary search.}  
At step $t$, the algorithm continues by adding a new interval:
\[
L_t = L_{t-1} \oplus [(l_t, r_t)] \quad \Rightarrow \quad |L_t| = |L_{t-1}| + 1.
\]

We check that
\[
w_t \geq w_{t-1}.
\]
Indeed, if $L_t[-1]$ is good, then
\[
w_t = |L_t| = |L_{t-1}| + 1 \geq w_{t-1} + 1,
\]
since $w_{t-1} \leq |L_{t-1}|$. Otherwise, if $L_t[-1]$ is bad, then the last good interval is unchanged, so
\[
w_t = w_{t-1}.
\]

Substituting these inequalities into the definition of $N_t$,
\[
N_t = |L_t| + \log_2(k) - 2 w_t \leq (|L_{t-1}| + 1) + \log_2(k) - 2 w_{t-1} = N_{t-1} + 1.
\]

\medskip
This completes the proof of Lemma~\ref{lem:dclose}.

%%%%%%%%%%%%%%%%%%
%%%%%%%%%%%%%%%%%

\subsection{Proof of Lemma~\ref{lem:dsimp}}
\label{proof-lem-N-1}
Fix $t \in [T_k]$. At step $t$, either $(l_{t-1}, r_{t-1})$ is a bad interval (Case~1) or a good interval (Case~2). We analyze these two cases separately. 

Before proceeding, we note an important fact about the list $L_t$:
\begin{equation}
\label{last-list-element}
    L_{t}[-1] = (l_{t}, r_{t}) \enspace.
\end{equation}

 \smallskip 

\paragraph{Case 1: $(l_{t-1}, r_{t-1})$ is a bad interval.}

If $|L_{t-1}| = 0$, then by~\eqref{last-list-element} we have $(l_{t-1}, r_{t-1}) = (l_0, r_0)$, which contradicts the assumption that $(l_0, r_0)$ is good. Hence, $|L_{t-1}| \geq 1$. 

On event $\cE_t$ -- see~\eqref{good-event} -- the call $\test(k,l_{t-1},r_{t-1},\ldots)$ returns the correct recommendation, which is nonzero since the interval is bad. Consequently, the algorithm removes the last element from the list:
\[
L_t = L_{t-1} \setminus \{L_{t-1}[-1]\}.
\]
Therefore,
\begin{equation}\label{L_t-relation}
    |L_t| = |L_{t-1}| - 1, \quad \text{and} \quad L_t[w] = L_{t-1}[w], \quad \forall 0 \leq w \leq |L_{t-1}| - 1.
\end{equation}

From the definition~\eqref{wt-last-good-index} of $w_t$, it follows that
\[
w_t = \max_{w \leq |L_{t-1}| - 1} \{ w : k \inspi L_{t-1}[w] \}.
\]
Since $L_{t-1}[-1]$ is bad, 
\[
w_{t-1} \leq |L_{t-1}| - 1,
\]
and by definition of $w_{t-1}$, we obtain
\[
w_{t-1} = \max_{w \leq |L_{t-1}| - 1} \{ w : k \inspi L_{t-1}[w] \}.
\]
Hence,
\[
w_t = w_{t-1}.
\]
Plugging into the definition~\eqref{number-steps} of $N_t$,
\[
N_t = (|L_{t-1}| - 1) + \lceil \log_2 k \rceil - 2 w_{t-1} = N_{t-1} - 1.
\]

\smallskip 

\paragraph{Case 2: $(l_{t-1}, r_{t-1})$ is a good interval.}

\textit{Step 1: The algorithm does not backtrack.} If $|L_{t-1}| = 0$, this is immediate from the \bbs subroutine. If $|L_{t-1}| \geq 1$, then on event $\cE_t$, the call $\test(k,l_{t-1},r_{t-1},\ldots)$ returns 0 (correct recommendation for a good interval), so the algorithm continues the binary search and does not backtrack.

\textit{Step 2: $(l_t, r_t)$ is a good interval.} Since the algorithm continues the binary search, it sets
\[
(l_t, r_t) = \text{either } (l_{t-1}, r_{t-1}) \text{ or one half-interval of } (l_{t-1}, r_{t-1}).
\]
If $(l_t, r_t) = (l_{t-1}, r_{t-1})$, it remains good by assumption. Otherwise, the algorithm uses $\test(k,l_{t-1},m_t,\ldots)$ on event $\cE_t$ to select the good half-interval, so $(l_t, r_t)$ is good.

\textit{Step 3: Conclusion.} By~\eqref{last-list-element}, the last element of $L_t$ is the good interval $(l_t, r_t)$, so by definition
\[
w_t = |L_t|.
\]
Because the algorithm continues (no backtracking), we have
\[
L_t = L_{t-1} \oplus [(l_t, r_t)] \implies |L_t| = |L_{t-1}| + 1.
\]
Therefore,
\[
w_t = |L_t| = |L_{t-1}| + 1 \geq w_{t-1} + 1,
\]
since $w_{t-1} \leq |L_{t-1}|$ by definition. Using the definition~\eqref{number-steps} of $N_t$,
\[
N_t = |L_t| + \lceil \log_2 k \rceil - 2 w_t \leq (|L_{t-1}| + 1) + \lceil \log_2 k \rceil - 2 (w_{t-1} + 1) = N_{t-1} - 1.
\]

This completes the proof of Lemma~\ref{lem:dsimp}.

%%%%%%%%%%%

%%%%%%%%%%%%%%

\subsection{Proofs of Proposition~\ref{lem:test-call-by-procedure} and Lemma~\ref{lem:sum-indicator-functions}}
\label{appendix-proba-bounds}

\begin{proof}[$\circ$ Proof of Proposition~\ref{lem:test-call-by-procedure}]
Recall that $l^{(k-1)}$ and $r^{(k-1)}$ are defined in the \asii procedure so that $(\pi^{(k-1)}_{l^{(k-1)}}, \pi^{(k-1)}_{r^{(k-1)}}) = (1, k-1)$. For clarity, we write $l = l^{(k-1)}$ and $r = r^{(k-1)}$. Since $\pi^{(k-1)}$ is coherent with $\pi^*$ (by assumption), this implies $\pi^*_l < \pi^*_r$. 

We define the event
\begin{equation*}
    \mathcal{C}_k := \left\{ 
        \max \left( 
            |M_{lr} - \hat M_{lr}|, \ 
            |M_{kl} - \hat M_{kl}|, \ 
            |M_{kr} - \hat M_{kr}| 
        \right) < \frac{\Delta}{2}
    \right\}.
\end{equation*}

Using that $\pi^*_l <\pi^*_r$ and $M \in \cM_{\Delta}$, and denoting $a \wedge b$ for $\min(a, b)$, we have 
\begin{align*}
    (M_{kl} - M_{lr}) \wedge (M_{kr} - M_{lr}) &\geq \Delta  \quad   \textrm{if} \    \pi^*_l < \pi^*_k < \pi^*_r \ ,\\
    (M_{lr} - M_{kr}) \wedge (M_{kl} - M_{kr}) &\geq \Delta  \quad \textrm{if} \   \pi^*_k < \pi^*_l \ ,\\
    (M_{lr} - M_{kl}) \wedge (M_{kr} - M_{kl}) &\geq \Delta  \quad \textrm{if} \  \pi^*_r < \pi^*_k  \ .
\end{align*}

On the event $\mathcal{C}_k$, this implies:
\begin{align*}
    & \hat M_{lr} < \hat M_{kl} \wedge \hat M_{kr} \quad \text{and } b = 0 
    \quad \text{if } \pi^*_l < \pi^*_k < \pi^*_r, \\
    & \hat M_{kr} < \hat M_{lr} \wedge \hat M_{kl} \quad \text{and } b = -1 
    \quad \text{if } \pi^*_k < \pi^*_l, \\
    & \hat M_{kl} < \hat M_{lr} \wedge \hat M_{kr} \quad \text{and } b = 1 
    \quad \text{if } \pi^*_r < \pi^*_k.
\end{align*}

Hence, whenever $\mathcal{C}_k$ occurs, the output $b$ of \test satisfies the condition defining $\mathcal{A}_k$ in~\eqref{good-event-procedure-call-to-test}, so that $\mathcal{C}_k \subset \mathcal{A}_k$.  Taking complements, we obtain
\[
    \mathbb{P}(\mathcal{A}_k^c) \leq \mathbb{P}(\mathcal{C}_k^c).
\]

 Conditioning on the choice of $(l, r)$, each of the three sample means 
$\hat M_{lr}$, $\hat M_{kl}$, and $\hat M_{kr}$ is an average of  
\[\left\lfloor \frac{T_0}{3} \right\rfloor
=
\left\lfloor \frac{1}{3}\left\lfloor \frac{T}{3\tilde n}\right\rfloor\right\rfloor
= \left\lfloor \frac{T}{9\tn} \right\rfloor  \ \ \text{independent samples.} \]
By applying a standard Hoeffding-type inequality (see Lemma~\ref{lem:chernoff} with $\epsilon = \Delta/2$), 
and taking a union bound over the three comparisons, we obtain
\[
    \mathbb{P}(\mathcal{C}_k^c \mid l, r) 
    \leq 6 \exp\left( -\frac{\lfloor T / (9\tn) \rfloor \Delta^2}{8\sigma^2} \right) 
    \leq 6 \exp\left( -\frac{T \Delta^2}{80\tn \sigma^2} \right),
\]
where the last inequality uses the assumption $T \geq 90\tn$, so that 
$\lfloor T / (9\tn) \rfloor \geq T / (10\tn)$.

Since this bound holds for any $(l, r)$, it also holds unconditionally. Thus,
\[
    \mathbb{P}(\mathcal{A}_k^c) \leq \mathbb{P}(\mathcal{C}_k^c) 
    \leq 6 \exp\left( -\frac{T \Delta^2}{80 \tn \sigma^2} \right).
\]
This completes the proof of Proposition~\ref{lem:test-call-by-procedure}.
\end{proof}

\medskip 

%%%%%%%%%%%%%%%%%%%

\begin{proof}[$\circ$ Proof  Lemma~\ref{lem:sum-indicator-functions}]
As in \cite{cheshire2021problem}, the proof relies on a Chernoff-type argument to bound the probability of the event $\left\{\sum_{t=1}^{T_k} \mathbf{1}_{\cE_t^c} \geq \frac{T_k}{4}\right\}$,
where \(T_k = 3 \lceil \log_2 k \rceil\).

Fix any \(t \in [T_k]\). By the same reasoning as in the proof of Proposition~\ref{lem:test-call-by-procedure}, we upper bound the probability of the complement event \(\cE_t^c\), defined in~\eqref{good-event}. 
Note that at each step \(t\), there are two calls to \test, and in each call, we observe three sample means. Conditionally on \((l_{t-1}, r_{t-1})\), each sample mean is an average of $\lfloor \frac{T}{9\tn T_k} \rfloor$ independent samples.
Applying a Hoeffding-type inequality  (Lemma~\ref{lem:chernoff} with \(\epsilon = \Delta/2\)) to each sample mean, and using a union bound over the three sample means and the two calls to \test, yields:
\begin{equation}
\label{eq:cE}
\forall t \in [T_k], \quad p_t := \P(\cE_t^c | \cF_{t-1}) = \P(\cE_t^c | (l_{t-1}, r_{t-1})) \leq 12 \exp\left(- \frac{\left\lfloor \frac{T}{9 \tn T_k} \right\rfloor \Delta^2}{8 \sigma^2}\right),
\end{equation}
where \(\cF_{t-1}\) is the sigma-algebra containing the information
available up to step \(t-1\) of the \bbs subroutine.

Using the bound \(\left\lfloor \frac{T}{9 \tn T_k} \right\rfloor \geq \frac{T}{18 \tn T_k}\), valid for $T \geq 54 \tn \lceil \log_2 n \rceil$, we get
\begin{equation}
\label{eq:cE-bis}
\forall t \in [T_k], \quad p_t \leq 12 \exp\left(-\frac{T \Delta^2}{144 \tn T_k \sigma^2}\right) := \bar{p}.
\end{equation}

Since \(T_k = 3 \lceil \log_2 k \rceil \leq 6 \log_2 n\), we have
\[
\frac{T \Delta^2}{144 \tn T_k \sigma^2} \geq \frac{T \Delta^2}{864\tn \sigma^2 \log_2 n}.
\]
Thus, for $\frac{T \Delta^2}{\tn \sigma^2} \geq 864 \ln(96) \log_2 n$, we obtain
\begin{equation}
\label{bar-p-1/8}
\bar{p} \leq \frac{1}{8}.
\end{equation}

Applying Markov's inequality for any \(\lambda \geq 0\), we have
\begin{align}
\P\left(\sum_{t=1}^{T_k} \mathbf{1}_{\cE_t^c} \geq \frac{T_k}{4}\right)
&\leq \E\left[\exp\left(\lambda \sum_{t=1}^{T_k} \mathbf{1}_{\cE_t^c}\right)\right] e^{-\lambda \frac{T_k}{4}}.
\label{eq:markov_mono}
\end{align}

Let \(\phi_p(\lambda) = \ln(1 - p + p e^\lambda)\) be the log-moment generating function of a Bernoulli\((p)\) variable. Since \(p \mapsto \phi_p(\lambda)\) is non-decreasing for \(\lambda \geq 0\), it follows from~\eqref{eq:cE-bis} that \(\phi_{p_t}(\lambda) \leq \phi_{\bar{p}}(\lambda)\) for all \(t\). By iterated conditioning and induction,
\begin{align}
\E\left[\exp\left(\lambda \sum_{t=1}^{T_k} \mathbf{1}_{\cE_t^c}\right)\right]
&= \E\left[\E\left[\exp\left(\lambda \mathbf{1}_{\cE_{T_k}^c}\right) | \cF_{T_k-1}\right] \exp\left(\lambda \sum_{t=1}^{T_k-1} \mathbf{1}_{\cE_t^c}\right)\right] \nonumber \\
&= \E\left[ e^{\phi_{p_{T_k}}(\lambda)} \exp\left(\lambda \sum_{t=1}^{T_k-1} \mathbf{1}_{\cE_t^c}\right) \right] \nonumber \\
&\leq e^{\phi_{\bar{p}}(\lambda)} \E\left[\exp\left(\lambda \sum_{t=1}^{T_k-1} \mathbf{1}_{\cE_t^c}\right)\right] \leq e^{T_k \phi_{\bar{p}}(\lambda)}.
\label{bound-exp-lambda}
\end{align}

Combining \eqref{eq:markov_mono} and \eqref{bound-exp-lambda} and optimizing over \(\lambda \geq 0\) yields
\[
\P\left(\sum_{t=1}^{T_k} \mathbf{1}_{\cE_t^c} \geq \frac{T_k}{4}\right) \leq \exp\left(-T_k \sup_{\lambda \geq 0} \left\{\frac{\lambda}{4} - \phi_{\bar{p}}(\lambda)\right\}\right).
\]

By standard properties of the KL divergence (proved in Appendix~\ref{appendix:proof-kl-ln}), for any \(0 < p < q \leq 1\),
\begin{equation}
\label{eq:klsup}
\sup_{\lambda \geq 0} \left\{ \lambda q - \phi_p(\lambda) \right\} = \kl(q, p).
\end{equation}
Taking \(p = \bar{p}\) and \(q = 1/4\), which is valid since \(\bar{p} \leq 1/8 < 1/4\) by~\eqref{bar-p-1/8}, we obtain
\begin{equation}
\label{bound-T_k-exp}
\P\left(\sum_{t=1}^{T_k} \mathbf{1}_{\cE_t^c} \geq \frac{T_k}{4}\right) \leq e^{-T_k \kl(1/4, \bar{p})}.
\end{equation}
Using the following inequality (proved in Appendix~\ref{section:small-lem-proof})
\begin{equation}
\label{eq:fano}
\forall q \in [0,1], p \in (0,1): \quad \kl(q,p) \geq q \ln\left(\frac{1}{p}\right) - \ln(2),
\end{equation}
and then taking $q=1/4$ and $p=\bar{p}$, we have
\begin{align*}
T_k \kl(1/4, \bar{p}) 
&\geq \frac{T_k}{4} \ln\left(\frac{1}{\bar{p}}\right) - T_k \ln(2) \\
&\geq \frac{T \Delta^2}{576 \tn \sigma^2} - T_k \left(\frac{\ln(12)}{4} + \ln(2)\right),
\end{align*}
where the last step follows from the definition of \(\bar{p}\) in \eqref{eq:cE-bis}.

Noting \(\ln(12)/4 \leq 1\) and \(T_k = 3 \lceil \log_2 n \rceil \leq 6 \log_2 n\), we get
\[
T_k \kl(1/4, \bar{p}) \geq \frac{T \Delta^2}{600 \tn \sigma^2} - 12 \log_2 n \geq \frac{T \Delta^2}{1200 \tn \sigma^2}
\]
for \(\frac{T \Delta^2}{1200 \tn \sigma^2} \geq 12 \log_2 n\).

Plugging this into \eqref{bound-T_k-exp} completes the proof of Lemma~\ref{lem:sum-indicator-functions}.
\end{proof}

%%%%%%%%%%%%
% End proof of upper bound for Delta-separated matrix
%%%%%%%%%%%%%%

%%%%%%%%%%%%%%%
%extension Delta tes
%%%%%%%%%%%%%%%%%

%%%%%%%%%%%%%%%%%%%%%%%%%%%%%%
\section{Extension Beyond Uniform Separation}
\label{appendix:extension-beyond-unif-speration}

This appendix complements Section~\ref{sec:extension}. 
It provides additional algorithmic details and the proof of Theorem~\ref{thm:UB-extension}. 

Throughout the appendix, we write $\pi^*$ for the true latent ordering and $\pi$ for the rank map returned by the algorithm.

%%%%%%%%

\subsection{Algorithmic extension}
\label{algo-description-appendix-extension}

We introduce a slight modification of \asii.
Compared to the main procedure, the test \textsc{Test} is replaced by 
\textsc{Test$_{\tilde\Delta}$}, which depends on a tolerance parameter $\tilde \Delta > 0$.
In addition, when a candidate insertion location is identified by the routine \textsc{Binary \& Backtracking Search}, 
a second validation \textsc{Test$_{\tilde\Delta}$} at precision $\tilde\Delta /2$ is performed; 
if this test fails, the item is discarded.
All other steps of the iterative insertion scheme remain unchanged.

%%%%%%%%%%%%%%%%%%%%%%%%%%%%%%%%

\paragraph{Modified subroutine \textsc{Test}\(_{\tDelta}\).}
Given a triplet of items $(\ell,r,k)$ and a tolerance parameter 
$\tDelta>0$, the subroutine \textsc{Test}\(_{\tDelta}\) 
uses empirical similarities to decide whether $k$ lies to the left,
in the middle, or to the right of $(\ell,r)$.
The decision is made by checking inequalities with margin $\tDelta / 2$
(see \eqref{eq:test-delta-gen}).
If none of the three cases is detected, the subroutine returns
\textsc{null}, in which case items $k, l, r$ are not
$\tDelta$-separated, and $k$ is discarded
from further insertion in \textsc{ASII}-Extension.

\textsc{Subroutine \textsc{Test}\(_{\tDelta}\).} Given an integer $T_0\ge 4$, and a triplet of items $(\ell,r,k)$, the subroutine \textsc{Test}\(_{\tDelta}(k, l, r, T_0)\) computes the corresponding sample means
 $\widehat M_{\ell r}$, $\widehat M_{\ell k}$, and $\widehat M_{r k}$, each computed from $\lfloor T_0 / 3  \rfloor$ observations per pair.
For any $a,b,x \in[n]$, define the event
\[
 I[(a,b),x,\tDelta]
:=
\Big\{
\widehat M_{a b} + \tfrac{\tDelta}{2}
<
\widehat M_{x a}
\wedge
\widehat M_{x b}
\Big\}.
\]

The output $b \in \{-1,0,1,\textsc{null}\}$ of 
\textsc{Test}\(_{\tDelta}(k, l, r, T_0)\) is defined as
\begin{equation}
\label{eq:test-delta-gen}
b \ := \ b_{k, (l,r)} \ := \
\begin{cases}
0 & \text{if }  I[(\ell,r),k,\tDelta],\\[2mm]
-1 & \text{if }  I[(k,r),\ell,\tDelta],\\[2mm]
+1 & \text{if }  I[(k,\ell),r,\tDelta],\\[2mm]
\textsc{null} & \text{otherwise.}
\end{cases}
\end{equation}

\paragraph{Extension of ASII.}

Throughout this procedure, $S^{(k)} \subset [k]$ denotes the set of items retained after processing items $1,\dots,k$.
At iteration $k$, the ordering of $S^{(k)}$ is represented by a rank map
$\pi^{(k)} : S^{(k)} \to \{1,\ldots,|S^{(k)}|\}$,
where $\pi^{(k)}_i$ denotes the position of item $i$ in the ordering.

We do not redefine the subroutine \textsc{Binary \& Backtracking Search} for this extension; this subroutine procedure remains essentially the same as before, with only minor modifications in the indexing.

\smallskip 

{{\small\underline{\textbf{Procedure}} \textsc{Active Seriation by Iterative Insertion} (ASII) - Extension }} \label{algo:asii-ext}
{{\footnotesize
\begin{algorithmic}[1]
    \REQUIRE Initial ranking $\tpi = (\tpi_1,\ldots,\tpi_{n-\tn})$ of $[n-\tn]$ 
(if $n-\tn \ge 3$); 
tolerance parameter $\tilde \Delta > 0$
    \ENSURE a set $S$  such that  $[n - \tn] \subset S \subset[n]$ and a rank map $\pi_S : S \to \{1,\ldots,|S|\}$
    \STATE \textbf{if} $n-\tn \leq 2$ \ 
                    \textbf{then}
            \STATE \quad Initialize  $S^{(2)} = \{1,2\}$ and $\pi^{(2)} =  (1,2)$ where $\pi_1^{(2)}= 1$ and $\pi_2^{(2)} = 2$ 
            \STATE \textbf{else}
            \STATE \quad Initialize  $S^{(n-\tn)} = [n-\tn]$ and $\pi^{(n-\tn)}$ with $\pi^{(n-\tn)}_i = \tpi_i$ for all $i \in [n-\tn]$    
            \STATE  $k_0 = \max (2, n-\tn)$ 
    \FOR{$k= k_0 + 1,\ldots,n$ }
        \STATE   Choose $l^{(k-1)}, r^{(k-1)} \in S^{(k-1)}$ such that $\pi^{(k-1)}_{l^{(k-1)}}=1$ and $\pi^{(k-1)}_{r^{(k-1)}}=|S^{(k-1)}|$.
        \STATE  \label{ASII-Ext:1st-test} $b = \textsc{Test\(_{\tilde \Delta}\)}(k, l^{(k-1)}, r^{(k-1)}, \lfloor T / (4 \tn) \rfloor)$ 
        \STATE \textbf{if} $b = \textsc{null}$ \ 
                    \textbf{then}
            \STATE \quad Set $S^{(k)} = S^{(k-1)}$ and  $\pi^{(k)} = \pi^{(k-1)}$   \label{reject} 
        \STATE \textbf{else if} $b = -1$ \ 
                    \textbf{then}
            \STATE \quad Set $S^{(k)} = S^{(k-1)} \cup \{k\}$, and $\pi_k^{(k)} = 1$, and  $\pi_i^{(k)} = \pi_i^{(k-1)} + 1$ for all $i \in S^{(k-1)}$ 
        \STATE \textbf{else if} $b = 1$ \ 
                    \textbf{then}
            \STATE \quad Set $S^{(k)} = S^{(k-1)} \cup \{k\}$, and $\pi_k^{(k)} = |S^{(k-1)}|+1$, and  $\pi_i^{(k)} = \pi_i^{(k-1)}$ for all $i \in S^{(k-1)}$ 
        \STATE \textbf{else}
            \STATE \quad \label{ASII-ext:call-to-bbs} let $m_k = \textsc{Binary \& Backtracking Search}(k, \pi^{(k-1)})$
            \STATE \quad Choose $\tilde l^{(k-1)}, \tilde r^{(k-1)} \in S^{(k-1)}$ such that $\pi_{\tilde l^{(k-1)}}^{(k-1)} = m_k -1$ and $\pi_{\tilde r^{(k-1)}}^{(k-1)}  = m_k$
            \STATE \quad \label{ASII-ext:2nd-test}   $b' = \textsc{Test$_{\tilde \Delta}$}(k, \tilde l^{(k-1)}, \tilde r^{(k-1)}, \lfloor T / (4 \tn) \rfloor)$ 
            \STATE \textbf{if} $b' = 0$ \ 
                    \textbf{then}
            \STATE \quad Set $S^{(k)} = S^{(k-1)}\cup \{k\}$ and $\pi_k^{(k)} = m_k$, and  $\pi_i^{(k)} = \pi_i^{(k-1)}$ for all $i \in S^{(k-1)}$ such that $\pi_i^{(k-1)} < m_k$, and  $\pi_i^{(k)} = \pi_i^{(k-1)} + 1$ for all $i \in S^{(k-1)}$ such that $\pi_i^{(k-1)} \geq m_k$
            \STATE \textbf{else} 
                    \STATE \quad  Set $S^{(k)} = S^{(k-1)}$ and $\pi^{(k)} = \pi^{(k-1)}$ 
    \ENDFOR
    \STATE $S = S^{(n)}$ and $\pi_S = \pi^{(n)}$       
    \end{algorithmic}
    }}

%%%%%%%%%%%%%%%%%%%%%

\subsection{Proof sketch of Theorem~\ref{thm:UB-extension}}
\label{app:proof-UB-extension}

The argument follows the same structure as the proof of Theorem~\ref{thm:UB:insert-n2-into-n1} and relies on the following two properties of the ASII-Extension procedure at each iteration $k$:

\begin{enumerate}
    \item[\textbf{(P1)}] 
    If $M_{S^{(k-1)} \cup \{k\}} \in \cM_{\tDelta}$, then with high probability, item $k$ is inserted at its correct position in the current ordering $\pi^{(k-1)}$ of $S^{(k-1)}$.
    
    \item[\textbf{(P2)}] 
    If $M_{S^{(k-1)} \cup \{k\}} \notin \cM_{\tDelta}$, then with high probability, the procedure either discards item $k$, or inserts it at its correct position in $\pi^{(k-1)}$.
\end{enumerate}

In particular, assuming that $\pi^{(k-1)}$ correctly orders $S^{(k-1)}$, properties (P1)–(P2) imply that $\pi^{(k)}$ correctly orders $S^{(k)}$ with high probability. Iterating this argument over $k$ shows that the final output $S=S^{(n)}$ is a $\tDelta$-maximal subset whose elements are correctly ordered.

Throughout the proof, $\pi^*$ denotes the true latent ordering of the $n$ items.

%%%%%

\subsubsection{Preliminaries}

\paragraph{General guarantees for \textsc{Test}\(_{\tDelta}\).}
Recall that \textsc{Test}\(_{\tDelta}\), defined in~\eqref{eq:test-delta-gen}, returns 
$b \in \{-1,0,1,\textsc{null}\}$. 
Its output is correct whenever it coincides with the population decision $b^*_{k,(l,r)}$, defined as follows. 
For any pair $(l,r)$ and any item $k$,
\begin{equation}
\label{true-b*}
b^*_{k,(l,r)} :=
\begin{cases}
-1 & \text{if } \spi_k < \spi_l \wedge \spi_r, \\
\ \ 0 & \text{if } \spi_k \in (\spi_l,\spi_r), \\
\ \ 1 & \text{if } \spi_k > \spi_l \vee \spi_r .
\end{cases}
\end{equation}
For a triplet $(l,r,k)$, define
\begin{equation}
\label{def-event-tilde-C}
\tilde \cC_{k,(l,r)} :=
\left\{
\max_{\{u,v\}\subset\{l,r,k\}}
|\widehat M_{uv}-M_{uv}|
\le \tDelta/4
\right\}.
\end{equation}

Under this event, all empirical similarities involved in 
\textsc{Test}\(_{\tDelta}(k,l,r,\lfloor T/(4\tn)\rfloor)\)
are within $\tDelta/4$ of their expectations.

\begin{prop}[Correctness of the local test]
\label{prop:local-test-correctness-new}
Fix $(l,r,k)$ and $\tDelta>0$, and assume that 
$\tilde \cC_{k,(l,r)}$ holds. Then:

\smallskip
\noindent
\textbf{(i)} If \textsc{Test}\(_{\tDelta}\), defined in~\eqref{eq:test-delta-gen}, returns $b\in\{-1,0,1\}$, then $b=b^*_{k,(l,r)}$.

\smallskip
\noindent
\textbf{(ii)} If $M_{\{l,r,k\}}\in \cM_{\tDelta}$, then 
\textsc{Test}\(_{\tDelta}\) returns some $b\in\{-1,0,1\}$.
\end{prop}

Hence, whenever the test returns a non-null value, it agrees with the population ordering. Moreover, if the three items $(l,r,k)$ are $\tDelta$-separated, the test will return the correct relative position. 

We omit the proof of Proposition~\ref{prop:local-test-correctness-new}.
It reduces to a deterministic verification on the event~\eqref{def-event-tilde-C}.

%%%%%%%%%%%%%%

\paragraph{The two calls to \textsc{Test}\(_{\tDelta}\) in ASII-Extension.}

Let $(l^{(k-1)},r^{(k-1)})$ and $(\tilde l^{(k-1)}, \tilde r^{(k-1)})$ denote the pairs of items selected at lines~7 and~17 of the ASII-Extension procedure (see Appendix~\ref{algo-description-appendix-extension}). 

For each of these pairs, the probability that the corresponding concentration event  $\tilde \cC_{k, (l,r)}$, defined in~\eqref{def-event-tilde-C}, fails, decays exponentially fast in $T \tDelta^2 / (\tilde n \sigma^2)$. More precisely,
\begin{equation}
\label{Extension-proof:high-proba-event-test}
\P\!\left(
\tilde \cC_{k, (l^{(k-1)},r^{(k-1)})}^c 
\cup 
\tilde \cC_{k, (\tilde l^{(k-1)}, \tilde r^{(k-1)})}^c
\right)
\le
C \exp\!\left(
- c \frac{T \tDelta^2}{\tilde n \sigma^2}
\right),
\end{equation}
for some absolute constants $c, C >0$.

The proof follows the same argument as Proposition~\ref{lem:test-call-by-procedure} in Appendix~\ref{appendix-proba-bounds}.

%%%%%%%%%%%%

\subsubsection{Proof sketch of (P1) and (P2)}  

We first summarize the mechanism at a high level. 
The modified test \textsc{Test}\(_{\tDelta}\) has the following property (with high probability): whenever the queried items are $\tDelta$-separated, it returns a non-null decision that matches the population ordering; otherwise, it may return \textsc{null}, but any non-null decision remains correct. 

Consequently, if $M_{S^{(k-1)}\cup\{k\}}\in \cM_{\tDelta}$, the ASII-Extension procedure behaves as the original \textsc{ASII} algorithm and inserts $k$ at the correct position. If $M_{S^{(k-1)}\cup\{k\}}\notin \cM_{\tDelta}$, the only additional possibility is that a call to \textsc{Test}\(_{\tDelta}\) returns \textsc{null}, in which case $k$ is discarded; otherwise, the insertion remains correct.

We now formalize this argument.

\textit{Proof sketch of (P1) and (P2).} Fix an iteration $k$ and consider the current ordered set $(S^{(k-1)},\pi^{(k-1)})$.

\smallskip

\noindent
$\circ$ \textbf{Case (P1): $M_{S^{(k-1)} \cup \{k\}} \in \cM_{\tDelta}$.}

\noindent
\emph{First call to \textsc{Test}\(_{\tDelta}\).}
Let $b=b_{k,(l^{(k-1)},r^{(k-1)})}$ be the output of the call at line~\ref{ASII-Ext:1st-test}.
By Proposition~\ref{prop:local-test-correctness-new} and~\eqref{Extension-proof:high-proba-event-test}, we have
\begin{equation}
\label{profo-extension-b=b*}
b_{k,(l^{(k-1)},r^{(k-1)})} = b^*_{k,(l^{(k-1)},r^{(k-1)})}
\end{equation}
with probability at least $1- C \exp\!\left(-c \tfrac{T \tDelta^2}{\tilde n \sigma^2}\right)$.

If $b^*_{k,(l^{(k-1)},r^{(k-1)})}\in\{\pm1\}$, ASII-Extension inserts $k$ at the corresponding extremity, hence at its correct position.

\smallskip
\noindent
\emph{Middle case and BBS.}
Assume now that $b^*_{k,(l^{(k-1)},r^{(k-1)})}=0$. The procedure calls \bbs and returns an insertion index $m_k$ (line~\ref{ASII-ext:call-to-bbs}). Since \bbs is unchanged, the analysis of the original \textsc{ASII} procedure applies under $M_{S^{(k-1)} \cup \{k\}} \in \cM_{\tDelta}$. In particular, Proposition~\ref{lem:auto} implies that $m_k$ is the correct insertion location in $\pi^{(k-1)}$, with high probability.

\smallskip
\noindent
\emph{Second call to \textsc{Test}\(_{\tDelta}\).}
Let $b'=b'_{k,(\tilde l^{(k-1)},\tilde r^{(k-1)})}$ be the output of the call at line~\ref{ASII-ext:2nd-test}. Arguing as above,
\begin{equation}
\label{profo-extension-b=b*:bis}
b'_{k,(\tilde l^{(k-1)},\tilde r^{(k-1)})} = b^*_{k,(\tilde l^{(k-1)},\tilde r^{(k-1)})}
\end{equation}
with probability at least $1- C \exp\!\left(-c \tfrac{T \tDelta^2}{\tilde n \sigma^2}\right)$.

On the intersection of the above high-probability events (and the corresponding event controlling \bbs), item $k$ is inserted at position $m_k$, hence at its correct location in the ordering. Consequently, if $\pi^{(k-1)}$ correctly orders $S^{(k-1)}$, then $\pi^{(k)}$ correctly orders $S^{(k)}$ with high probability.

\medskip

\noindent
$\circ$ \textbf{Case (P2): $M_{S^{(k-1)} \cup \{k\}} \notin \cM_{\tDelta}$.}

In this case, at least one of the two calls to \textsc{Test}\(_{\tDelta}\) may return \textsc{null}. If this happens, ASII-Extension discards item $k$ and keeps $(S^{(k)},\pi^{(k)})=(S^{(k-1)},\pi^{(k-1)})$. Otherwise, if both calls return non-null outputs, Proposition~\ref{prop:local-test-correctness-new} ensures that these outputs are correct on the corresponding concentration events
(namely, $\tilde \cC_{k, (l^{(k-1)},r^{(k-1)})}$ and $\tilde \cC_{k, (\tilde l^{(k-1)}, \tilde r^{(k-1)})}$), 
and the same reasoning as in~\eqref{profo-extension-b=b*}--\eqref{profo-extension-b=b*:bis} shows that the insertion (when performed) is at the correct location.

%%%%%%%%%%%

%%%%%%%%%%%%%%%%%%%%%%%%%%%%%

%%%%%%%%%%%%%%%%%%%%%%%
%Lower bounds
%%%%%%%%%%%%%%%%%%%%%%%%%%%%%%%%%%%%%%%ye

\section{Proofs of Theorems~\ref{thm:LB2} and~\ref{thm:UB} (Lower Bounds)}
\label{app:LB}

\paragraph{Model and parametric instance.}
We prove the lower bounds stated in Theorems~\ref{thm:LB2} and~\ref{thm:UB} by analyzing a simple parametric family of active seriation problems.  
Let $\mathcal{S}_n$ denote the set of all permutations of $[n]:=\{1,\ldots,n\}$.  
A learner interacts with an environment defined by a similarity matrix $M \in \cM_\Delta$ over $n$ items.  
At each round $t \in [T]$, the learner selects an unordered pair $\{a_t,b_t\}$ with $a_t \neq b_t$ and observes a random sample 
\[
Y_t \sim \nu_{\{a_t,b_t\}} = \mathcal{N}(M_{a_t,b_t},\, \sigma^2),
\]
where $\sigma>0$ is the (known) noise standard deviation.  
The goal is to recover the underlying ordering of the items based on all $T$ observations.  
This setting can be viewed as a bandit problem with $\tfrac{n(n-1)}{2}$ arms corresponding to the pairs $\{i,j\}$.

To derive explicit lower bounds, we focus on a parametric instance parameterized by a permutation $\pi \in \mathcal{S}_n$, a signal gap $\Delta>0$, and a noise variance $\sigma^2>0$.  
Specifically, we take
\begin{equation}
\label{robinson-mat-LB-parametric}
    M = R_\pi := \left[R_{\pi_i \pi_j}\right]_{i,j \in[n]}, \qquad R_{ij} = \Delta \bigl( n - |i-j| \bigr),
\end{equation}
so that $R$ is a Toeplitz Robinson matrix and $M \in \cM_\Delta$.  
Each observed pair $\{a,b\}$ then follows
\[
\nu^\pi_{\{a,b\}} := \mathcal{N}(R_{\pi_a \pi_b}, \sigma^2).
\]
We write $\nu^{\pi,\Delta,\sigma^2,T}$ for the resulting instance of the active seriation problem, i.e., the joint law of all pairwise observations under parameters $(\pi,\Delta,\sigma^2,T)$. Throughout the remainder of the proof we fix $\Delta$, $\sigma^2$, and $T$; to lighten notation, we set
\[
\nu^{\pi} \coloneqq \nu^{\pi,\Delta,\sigma^2,T},
\]
and simply write $\nu^{\pi}$ for the corresponding instance. After $T$ rounds, the learner outputs an estimated permutation $\hat{\pi}\in\mathcal{S}_n$ (up to reversal), 
and the goal is to lower bound the probability of error $p_{M,T}$ for this instance. %Recall that we denote by $\nu_{\{a,b\}}$ the distribution associated with the unordered pair $\{a,b\}$. 

% Consider the \(n-1\) instances \(\nu^{\pi^1}, \ldots, \nu^{\pi^{n-1}}\) and define
% \begin{align}
% \label{x-simple-def}
% x &:= \frac{1}{n-1} \sum_{k=1}^{n-1} \mathbb{P}_{\nu^{\pi^k}}\!\big( \hat{\pi} \in \{ \pi^k, (\pi^k)_{\mathrm{rev}} \} \big), \\
% \label{y-simple-def}
% y &:= \frac{1}{n-1} \sum_{k=1}^{n-1} \mathbb{P}_{\nu^{\pi}}\!\big( \hat{\pi} \in \{ \pi^k, (\pi^k)_{\mathrm{rev}} \} \big),
% \end{align}
% where \((\pi^k)_{\mathrm{rev}}\) is the reversal of \(\pi^k\).
% Our goal is to upper- and lower-bound \(\kl(x,y)\).
%%%%%%%%%%%%%%%%%%%%%%

\paragraph{Proofs of Theorems~\ref{thm:LB2} and~\ref{thm:UB}.}

Fix $\Delta,\sigma^2,$ and $T$, and consider the class of instances 
$\{\nu^{\pi} : \pi \in \mathcal{S}_n\}$.

\begin{prop}[Impossibility regime]
\label{thm:impossibility-regime}
For any \(n \ge 9\), \(\Delta > 0\), \(\sigma > 0\), and \(T \ge 1\) satisfying 
\[
\frac{T \Delta^2}{n \sigma^2} \le \frac{\ln n}{64},
\]
we have
% \[
% \inf_{\hat{\pi}} \max_{\pi \in \mathcal{S}_n} 
% \bP_{\nu^{\pi}}\!\left( 
% \hat{\pi} \notin \{\pi, \pi^{\mathrm{rev}}\}
% \right)
% \ge \frac{1}{2},
% \]
% where the infimum is taken over all estimators \(\hat{\pi}\).
\[
\inf_{\hat{\pi}} \max_{\pi \in \mathcal{S}_n} 
\bP_{\nu^{\pi}}\!\left( 
\hat{\pi} \notin \{\pi, \pi^{\mathrm{rev}}\}
\right)
\ge \frac{1}{2},
\]
where the infimum is taken over all estimators \(\hat{\pi}\).
\end{prop}

\smallskip 

\begin{prop}[Exponential-rate lower bound]
\label{thm:lower-bound:all-T}
For any \(n \ge 4\), \(\Delta, \sigma > 0\), and \(T \ge 1\) such that 
\[
\frac{T \Delta^2}{n \sigma^2} \ge 3,
\]
we have
\[
\inf_{\hat{\pi}} \max_{\pi \in \mathcal{S}_n} 
\bP_{\nu^{\pi}}\!\left( 
\hat{\pi} \notin \{\pi, \pi^{\mathrm{rev}}\}
\right)
\ge \exp\!\left(- 8 \frac{T \Delta^2}{\sigma^2 n} \right).
\]
\end{prop}

Proposition~\ref{thm:impossibility-regime} implies Theorem~\ref{thm:LB2} for any \(c_1 \le 1/64\), 
while Proposition~\ref{thm:lower-bound:all-T} yields Theorem~\ref{thm:UB} for any $c_0 \ge 3$. 

%%%%%%%%%%%%%

\begin{rem}
In Theorem~\ref{thm:UB}, the stronger condition 
\(\tfrac{T \Delta^2}{\sigma^2 n} \ge c_0 \ln n\) 
is imposed only for consistency with Theorem~\ref{thm:UB:insert-n2-into-n1} (the upper bound).  
Theorem~\ref{thm:UB}, corresponding to the lower bound proved here in 
Proposition~\ref{thm:lower-bound:all-T}, 
remains valid under the weaker assumption 
\(\tfrac{T \Delta^2}{\sigma^2 n} \ge 3\).
\end{rem}

The proofs of Propositions~\ref{thm:impossibility-regime} and~\ref{thm:lower-bound:all-T} are given below.  
Before presenting them, we briefly recall standard concepts and notations for active learning problems, which will be used to provide formal statements of Propositions~\ref{thm:impossibility-regime} and~\ref{thm:lower-bound:all-T}. We will also recall some useful information-theoretic lemmas.

%%%%%%%%%%%%%%%%%

%%%%%%%%%%%%%%%

\subsection{Bandit formulation}

We adopt the standard bandit formalism of~\cite[Section~4.6]{lattimore2020bandit}. 
A learner interacts with the environment according to a strategy consisting of an adaptive sampling rule and a recommendation rule.

%\subsection{The canonical bandit model and formal statement of Propositions~\ref{thm:impossibility-regime} and~\ref{thm:lower-bound:all-T}}
%In this section we follow the canonical bandit model - a standard formalism for multi armed bandits, see~\cite[Section~4.6]{lattimore2020bandit}, and provide a formal statement of Propositions~\ref{thm:impossibility-regime} and~\ref{thm:lower-bound:all-T}.

%\paragraph{Strategy of the learner}We demonstrate that the lower bounds of Theorems~\ref{thm:LB2} and~\ref{thm:UB} hold for any possible algorithm, or what we will now term strategy, of the learner. Where by strategy, we refer to the way in which the learner interacts with the environment, i.e. how they choose which pairs of items to draw samples from and eventually specify their recommendation $\hpi$. Formally, a strategy of the learner consists of a pair $\psi = ((\vartheta)_{t<T},\varphi)$, where $(\vartheta)_{t<T}$ is a sequence of sampling rules and $\varphi$ is a recommendation rule, where we specify: 

\paragraph{Strategy.}
A strategy $\psi = ((\vartheta_t)_{t<T}, \varphi)$ is defined as follows.  
For each round $t<T$, the sampling rule
\[
\vartheta_t : (\mathcal{P}_n \times \mathbb{R})^{t-1} \to \mathcal{P}_n ,
\qquad
\mathcal{P}_n = \{ \{a,b\} : 1 \le a < b \le n \},
\]
maps the past history $(\{a_s,b_s\}, Y_s)_{s<t}$ to the next queried pair $\{a_t,b_t\}$ at round $t$.

The recommendation rule
\[
\varphi : (\mathcal{P}_n \times \mathbb{R})^{T} \to \mathcal{S}_n
\]
maps the full history $(\{a_s,b_s\}, Y_s)_{s\le T}$ to an estimated permutation
$\hat{\pi}^\psi = \varphi((\{a_s,b_s\}, Y_s)_{s\le T})$.

\paragraph{Probability space and notation.} 
Let $\Omega_T = (\mathcal{P}_n \times \mathbb{R})^T$ and $\mathcal{F}_T$ be the associated Borel $\sigma$-algebra.  
For an active seriation instance $\nu$ and a strategy $\psi$, we denote by $\bP_\nu^\psi$ the probability measure on $(\Omega_T, \mathcal{F}_T)$  induced by the strategy $\psi$ interacting with $\nu$, and by $\E_\nu^\psi$ the corresponding expectation.  
We omit the superscript $\psi$ when the dependence on the strategy is unambiguous.

We now restate Propositions~\ref{thm:impossibility-regime} and~\ref{thm:lower-bound:all-T} in this formalism.

\newtheorem*{prop1}{Proposition E.1}
\begin{prop1}\label{prop1}
For any \(n \ge 9\), \(\Delta > 0\), \(\sigma > 0\), and \(T \ge 1\) satisfying 
\[
\frac{T \Delta^2}{n \sigma^2} \le \frac{\ln n}{64},
\]
we have
\[
\inf_{\psi} \max_{\pi \in \mathcal{S}_n} 
\bP_{\nu^{\pi}}^\psi\!\left( 
\hat{\pi}^\psi \notin \{\pi, \pi^{\mathrm{rev}}\}
\right)
\ge \frac{1}{2},
\]
where the infimum is taken over all possible strategies \(\psi\).
\end{prop1}

\newtheorem*{prop2}{Proposition E.2}
\begin{prop2}\label{prop2}
For any \(n \ge 4\), \(\Delta, \sigma > 0\), and \(T \ge 1\) such that 
\[
\frac{T \Delta^2}{n \sigma^2} \ge 3,
\]
we have
\[
\inf_{\psi} \max_{\pi \in \mathcal{S}_n} 
\bP_{\nu^{\pi}}^\psi\!\left( 
\hat{\pi}^\psi \notin \{\pi, \pi^{\mathrm{rev}}\}
\right)
\ge \exp\!\left(- 8 \frac{T \Delta^2}{\sigma^2 n} \right).
\]
\end{prop2}

%%%%%%%%%%%%

\subsection{Preliminaries and notation}
We denote by $N_{\{a,b\}}^\psi(T)$ the number of times pair $\{a,b\}$ has been sampled by strategy~$\psi$ up to time~$T$.  
When the dependence on $\psi$ is clear, we write $N_{\{a,b\}}(T)$ or simply $N_{\{a,b\}}$ for brevity.
\paragraph{Useful lemmas.}
To prove Propositions~\ref{thm:impossibility-regime} and~\ref{thm:lower-bound:all-T}, 
we rely on two standard lemmas.

We denote by $\mathrm{kl}$ the Kullback–Leibler divergence between Bernoulli distributions:
\[
\forall\, p,q \in [0,1], 
\quad 
\mathrm{kl}(p,q) = p \ln \frac{p}{q} + (1-p) \ln \frac{1-p}{1-q}.
\]

Lemma~\ref{lem:fonda} is an adaptation of~\cite{garivier2019explore}, derived from the data-processing inequality and joint convexity of KL divergence.  
A proof is provided in Section~\ref{proof-fundamental-lemma}.

\begin{lem}[Fundamental inequality for bandits]
\label{lem:fonda}
Let $\nu^1, \ldots, \nu^N$ and $\tilde{\nu}^1, \ldots, \tilde{\nu}^N$ be two sequences of $N$ seriation problems.  
For any events $(\mathcal{E}_k)_{k \le N}$ with $\mathcal{E}_k \in \mathcal{F}_T$, we have
\[
\mathrm{kl}\!\left(
\frac{1}{N} \sum_{k=1}^N \bP_{\nu^k}(\mathcal{E}_k),
\frac{1}{N} \sum_{k=1}^N \bP_{\tilde{\nu}^k}(\mathcal{E}_k)
\right)
\le
\frac{1}{N} \sum_{k=1}^N \sum_{1 \le a < b \le n}
\E_{\nu^k}[N_{\{a,b\}}(T)] \,
\mathrm{KL}(\nu_{\{a,b\}}^k, \tilde{\nu}_{\{a,b\}}^k).
\]
\end{lem}

Lemma~\ref{lem:gausskl} provides the closed-form expression of the KL divergence between Gaussian distributions.

\begin{lem}[KL divergence between Gaussians]
\label{lem:gausskl}
Let $\rho_1 = \mathcal{N}(\mu_1, \sigma^2)$ and $\rho_2 = \mathcal{N}(\mu_2, \sigma^2)$. Then,
\[
\mathrm{KL}(\rho_1, \rho_2) = \frac{(\mu_1 - \mu_2)^2}{2 \sigma^2}.
\]
\end{lem}
A proof of Lemma~\ref{lem:gausskl} can be found in~\cite{statisticalproofs}, 
see entry “Kullback–Leibler divergence for the normal distribution”.

%%%%%%%%%%%%%%

%%%%%%%%%%%%%%%%%%%%%%%%%%%%%%%%%%%%%%%%

\subsection{Proof of Proposition~\ref{thm:impossibility-regime}}
\label{proof-1st-lower-bound}

For clarity, we first establish the lower bound for strategies that are invariant under relabeling of the items, 
and then extend the result to arbitrary strategies.

\paragraph{Invariance under relabeling.}
A strategy~$\psi$ is said to be \textit{invariant under permutations of item labels} if, for any active seriation instance~$\nu$ 
and any permutation $\pi' \in \mathcal{S}_n$ acting on the item set~$[n]$,
\begin{equation}
\label{invariance-labeling}
\text{ the distribution of } N_{\{a,b\}} 
\text{ under } \nu^{\pi'} 
\text{ coincides with that of } N_{\{\pi'(a), \pi'(b)\}}
\text{ under } \nu,
\end{equation}
for any $\{a,b\}\subset[n]$.
This symmetry assumption simplifies the proof by allowing us to treat all item indices as exchangeable. 
Once the argument is established for invariant strategies,
we extend the proof to arbitrary strategies by averaging their performance
over all possible label permutations.

\paragraph{Choice of permutations.}
We now focus on a specific family of hypotheses indexed by transpositions of consecutive indices.  
Since the lower bound only requires exhibiting one instance for which the error probability is large,
we may, without loss of generality, fix the underlying permutation as the identity
$\pi = (1,\ldots,n)$. 
For each $k \in [n-1]$, let $\pi^k$ denote the permutation obtained by swapping positions~$k$ and~$k{+}1$, 
while leaving all others unchanged:
\begin{equation}
\label{pi^k-def}
\forall k \in [n-1], \quad \pi^k = (\pi^k_i )_{i\in[n]} \in \mathcal{S}_n , \qquad  \quad 
\pi^k_i =
\begin{cases}
k+1, & \text{if } i = k,\\
k,   & \text{if } i = k+1,\\
i,   & \text{otherwise.}
\end{cases}
\end{equation}
In other words, $\pi^k$ is the transposition $(k, k{+}1)$ acting on the identity permutation.

%%%%%%%%%%%%%%%%%%%%%%%%
%%%%%%%%%%%%%%%%%%%%%

\vspace{1em}
\noindent
\textbf{Proof for invariant-to-labeling strategies.} Recall that, for any permutation $\pi \in \mathcal{S}_n$, 
$\nu^{\pi}$ denotes the active seriation instance in which
each pair $\{a,b\}$ produces i.i.d.\ samples 
$\nu^{\pi}_{\{a,b\}} = \mathcal{N}(M_{a, b}, \sigma^2)$
with mean 
$M_{a, b} = R_{\pi_a, \pi_b} = \Delta \bigl(n - |\pi_a - \pi_b|\bigr)$ as defined earlier in~\eqref{robinson-mat-LB-parametric}.

Consider the \(n-1\) instances \(\nu^{\pi^1}, \ldots, \nu^{\pi^{n-1}}\) and define
\begin{align}
\label{x-simple-def}
x &:= \frac{1}{n-1} \sum_{k=1}^{n-1} \mathbb{P}_{\nu^{\pi^k}}\!\big( \hat{\pi} \in \{ \pi^k, (\pi^k)_{\mathrm{rev}} \} \big), \\
\label{y-simple-def}
y &:= \frac{1}{n-1} \sum_{k=1}^{n-1} \mathbb{P}_{\nu^{\pi}}\!\big( \hat{\pi} \in \{ \pi^k, (\pi^k)_{\mathrm{rev}} \} \big),
\end{align}
where \((\pi^k)_{\mathrm{rev}}\) is the reversal of \(\pi^k\).
Our goal is to upper and lower bound \(\kl(x,y)\).

\smallskip
\noindent
\textit{Upper bound.} By Lemma~\ref{lem:fonda} with \(N=n-1\) and \(\mathcal{E}_k=\{\hat{\pi}\in\{\pi^k,(\pi^k)_{\mathrm{rev}}\}\}\) for $k\in[n-1]$,
\begin{align}
\label{eq:kl}
\kl(x,y)
&\le \frac{1}{n-1} \sum_{k=1}^{n-1} \sum_{1\le a<b\le n} 
\E_{\nu^{\pi^k}}\!\big[N_{\{a,b\}}(T)\big] \, \mathrm{KL}\!\big(\nu^{\pi^k}_{\{a,b\}}, \nu^{\pi}_{\{a,b\}}\big).
\end{align}
Under our Gaussian parametric instance, $\mathrm{KL}(\nu^{\pi^k}_{\{a,b\}}, \nu^{\pi}_{\{a,b\}})=0$ if $\{a,b\}\cap\{k,k+1\}=\emptyset$, and is at most $\Delta^2/(2\sigma^2)$ otherwise (Lemma~\ref{lem:gausskl}).
Hence
\begin{equation}
\label{new-eq-lower-issue}
\kl(x,y) \le \frac{\Delta^2}{2\sigma^2}\cdot \frac{1}{n-1}
\sum_{k=1}^{n-1} \sum_{\substack{1\le a<b\le n\\ \{a,b\}\cap\{k,k+1\}\neq\emptyset}}
\E_{\nu^{\pi^k}}\!\big[N_{\{a,b\}}(T)\big].
\end{equation}
For fixed $k\in[n-1]$, every pair $\{a,b\}$ such that $\{a,b\}\cap\{k,k+1\}\neq\emptyset$ contains either $k$ or $k+1$; hence 
\begin{equation*}
\sum_{\substack{1\le a<b\le n\\ \{a,b\}\cap\{k,k+1\}\neq\emptyset}}
\E_{\nu^{\pi^k}}\!\big[N_{\{a,b\}}(T)\big]
\le
\sum_{a: a\neq k}\E_{\nu^{\pi^k}}\!\big[N_{\{a,k\}}(T)\big]
+
\sum_{a: a\neq k+1}\E_{\nu^{\pi^k}}\!\big[N_{\{a,k+1\}}(T)\big],
\end{equation*}
since $N_{\{a,b\}}(T)\ge 0$ for all $\{a,b\}$. 
Substituting this bound into~\eqref{new-eq-lower-issue} and using the convention $N_{\{i,i\}}(T)=0$, we obtain
\begin{equation}
\label{new-eq-lower-issue2}
\kl(x,y) \le \frac{\Delta^2}{2\sigma^2}\cdot \frac{1}{n-1}
\sum_{k=1}^{n-1} \sum_{a\in [n]}
\E_{\nu^{\pi^k}}\!\Big[ N_{\{a,k\}}(T)+N_{\{a,k+1\}}(T)\Big].
\end{equation}

Throughout the next calculations, we abbreviate $N_{\{a,b\}}:=N_{\{a,b\}}(T)$

By the invariance under relabeling assumption~\eqref{invariance-labeling}, 
for any permutation $\pi'$ and any pair $\{a,b\}$,
\[
\E_{\nu^{\pi'}}[N_{\{a,b\}}]
=
\E_{\nu}[N_{\{\pi'(a),\pi'(b)\}}].
\]
Applying this with $\pi'=\pi^k$ (the transposition exchanging $k$ and $k+1$), we obtain
\[
\E_{\nu^{\pi^k}}\!\big[N_{\{a,k\}} + N_{\{a,k+1\}}\big]
=
\E_{\nu}\!\big[N_{\{\pi^k(a),k\}} + N_{\{\pi^k(a),k+1\}}\big].
\]
Since $\pi^k$ merely swaps $k$ and $k+1$, the set
$\big{\{}\{\pi^k(a),k\},\{\pi^k(a),k+1\}\big{\}}$ coincides with
$\big{\{}\{a,k\},\{a,k+1\}\big{\}}$ for every $a\in[n]$ (recalling that $N_{\{i,i\}}=0$).
Therefore,
\[
\sum_{a\in[n]}
\E_{\nu^{\pi^k}}\!\big[N_{\{a,k\}} + N_{\{a,k+1\}}\big]
=
\sum_{a\in[n]}
\E_{\nu}\!\big[N_{\{a,k\}} + N_{\{a,k+1\}}\big].
\]
Plugging this identity into~\eqref{new-eq-lower-issue2} yields
\begin{equation}
\label{LB:eq-inv-label}
\kl(x,y)
\le
\frac{\Delta^2}{2\sigma^2}\cdot\frac{1}{n-1}
\sum_{k=1}^{n-1}
\sum_{a\in[n]}
\E_{\nu}\!\big[N_{\{a,k\}}(T) + N_{\{a,k+1\}}(T)\big].
\end{equation}
Since $\sum_{1\le a<b\le n} N_{\{a,b\}}(T)=T$, and $N_{\{a,a\}}(T) =0$ for any $a$, we have 
\[
\sum_{k=1}^{n-1}\sum_{a\in[n]}
\big(N_{\{a,k\}}(T)+N_{\{a,k+1\}}(T)\big)
\le
2\sum_{a,b\in[n]} N_{\{a,b\}}(T)
=
4T.
\]
Finally, we conclude that 
\begin{equation}
\label{ccl:upper_bound:proof:LB-inv}
\kl(x,y)
\le
\frac{2 T\,\Delta^2}{(n-1)\sigma^2}.
\end{equation}

\smallskip
\noindent
\textit{Lower bound.} Using the refined Pinsker inequality \cite[eq.~(13)]{gerchinovitz2020fano},
\[
\forall p\in[0,1],\ \forall q\in(0,1):\quad (p-q)^2 \max\!\Big(2,\ \ln \frac{1}{q}\Big) \le \kl(p,q),
\]
with \(p=x\), \(q=y\), and \eqref{ccl:upper_bound:proof:LB-inv}, we get \underline{for \(y\in(0,1)\)}:
\begin{equation}
\label{appli:refined-pinsker-1st}
x \ \le\  y + \sqrt{\frac{2 T \Delta^2}{(n-1)\, \sigma^2 \,\max\!\left(2, \ln \frac{1}{y}\right)}}.
\end{equation}
By definition~\eqref{y-simple-def} of $y$, and by disjointness of the events \(\{\hat{\pi}\in\{\pi^k,(\pi^k)_{\mathrm{rev}}\}\}_{k=1}^{n-1}\),
\begin{equation}
\label{y-bound-simple}
y \le \frac{1}{n-1}.
\end{equation}
By defintion~\eqref{x-simple-def} of $x$, and by averaging,
\begin{equation}
\label{x-bound-simple}
\min_{k\in[n-1]}\Big(1-\P_{\nu^{\pi^k}}(\hat{\pi}\notin\{\pi^k,(\pi^k)_{\mathrm{rev}}\})\Big) \le x.
\end{equation}

Combining \eqref{appli:refined-pinsker-1st}–\eqref{y-bound-simple}–\eqref{x-bound-simple}, for \(n\ge 1+e^2\),
\[
\min_{k\in[n-1]} \Big(1-\P_{\nu^{\pi^k}}(\hat{\pi}\notin\{\pi^k,(\pi^k)_{\mathrm{rev}}\})\Big)
\ \le\ \frac{1}{n-1} + \sqrt{\frac{2 T \Delta^2}{(n-1)\,\sigma^2 \,\ln(n-1)}}.
\]
In particular, if \(n \ge 5\) and \(\frac{2 T\Delta^2}{(n-1)\sigma^2} \le \frac{\ln(n-1)}{16}\), the right-hand side is at most \(1/2\). Hence there exists \(k \in [n-1]\) such that
\[
\mathbb{P}_{\nu^{\pi^k}}\!\big(\hat{\pi}\notin\{\pi^k,(\pi^k)_{\mathrm{rev}}\}\big)\ge \frac{1}{2}.
\]
(Notice that \(\frac{2T \Delta^2}{(n-1)\sigma^2}\le \frac{\ln(n-1)}{16}\)
is ensured by the slightly stronger condition 
\(\frac{T \Delta^2}{n \sigma^2}\le \frac{\ln n}{64}\) for all \(n\ge 3\).)

- \underline{If \(y=0\)} in~\eqref{y-simple-def}, then either \(x=0\), in which case \eqref{x-bound-simple} implies 
\(\max_{k}\mathbb{P}_{\nu^{\pi^k}}(\hat{\pi}\notin\{\pi^k,(\pi^k)_{\mathrm{rev}}\})=1\),
or \(\kl(x,0)=+\infty\), which would contradict \eqref{LB:eq-inv-label} for finite \(T\).

Combining both cases proves Proposition~\ref{thm:impossibility-regime}
for strategies invariant under relabeling. \hfill\(\square\)

%%%%%%%%%%

%%%%%%%%%%%%%%

\bigskip
\noindent
\textbf{Extension to arbitrary strategies.}
We now remove the assumption that the strategy is invariant under relabeling.
For any permutation $\pi \in \cS_n$ and any $k \in [n-1]$, define
\begin{equation}
\label{pi^k-def-general}
\pi^k := (k, k{+}1) \circ \pi,
\end{equation}
that is, $\pi^k$ is obtained by swapping the images $k$ and $k{+}1$ in~$\pi$.
Equivalently,
\[
\forall k \in [n-1], \quad \pi^k = (\pi^k_i )_{i\in[n]} \in \mathcal{S}_n , \qquad  \quad 
\pi^k_i =
\begin{cases}
k{+}1, & \text{if } \pi_i = k,\\
k, & \text{if } \pi_i = k{+}1,\\
\pi_i, & \text{otherwise.}
\end{cases}
\]
Hence, the two instances $\nu^{\pi}$ and $\nu^{\pi^k}$ differ only for comparisons involving items whose true labels belong to $\{k,k{+}1\}$. 
Formally,
\[
\nu^{\pi}_{\{a,b\}} = \nu^{\pi^k}_{\{a,b\}}
\quad\text{if } \{\pi_a,\pi_b\}\cap\{k,k{+}1\}=\emptyset,
\quad\text{and}\quad
\mathrm{KL}\!\big(\nu^{\pi}_{\{a,b\}}, \nu^{\pi^k}_{\{a,b\}}\big)
\le \frac{\Delta^2}{2\sigma^2}
\text{ otherwise.}
\]

Define the symmetrized quantities
\begin{align}
\label{x-def-general}
x' &:= \frac{1}{n!(n-1)} 
\sum_{\pi \in \cS_n} \sum_{k=1}^{n-1} 
\P_{\nu^{\pi^k}}\!\big(\hat{\pi} \in \{\pi^k, (\pi^k)_{\mathrm{rev}}\}\big),\\
\label{y-def-general}
y' &:= \frac{1}{n!(n-1)} 
\sum_{\pi \in \cS_n} \sum_{k=1}^{n-1} 
\P_{\nu^{\pi}}\!\big(\hat{\pi} \in \{\pi^k, (\pi^k)_{\mathrm{rev}}\}\big).
\end{align}

Applying Lemma~\ref{lem:fonda} with $N = n!(n-1)$ and 
$\mathcal{E}_{\pi,k} = \{\hat{\pi}\in\{\pi^k,(\pi^k)_{\mathrm{rev}}\}\}$ for $\pi\in\cS_n$ and $k\in[n-1]$ gives
\begin{equation}
\label{eq:kl-bis}
\kl(x',y')
\le \frac{1}{n!(n-1)} 
\sum_{\pi\in\cS_n} \sum_{k=1}^{n-1} 
\sum_{1\le a<b\le n} 
\E_{\nu^{\pi^k}}\!\big[N_{\{a,b\}}(T)\big] \,
\mathrm{KL}\!\big(\nu^{\pi^k}_{\{a,b\}}, \nu^{\pi}_{\{a,b\}}\big).
\end{equation}

Since the KL term is nonzero only when 
$\{\pi_a,\pi_b\}\cap\{k,k+1\}\neq\emptyset$
(equivalently, $\{\pi_a^k,\pi_b^k\}\cap\{k,k+1\}\neq\emptyset$),
in which case it equals $\Delta^2/(2\sigma^2)$, we get
\begin{equation}\label{eq:second-kl-eq}
 \kl(x',y') \le \frac{\Delta^2}{2\sigma^2}\frac{1}{n!(n-1)} 
\sum_{\pi\in\cS_n} \sum_{k=1}^{n-1} 
\sum_{\substack{1\le a<b\le n \\ \{\pi_a^k,\pi_b^k\}\cap\{k,k{+}1\}\neq\emptyset}} 
\E_{\nu^{\pi^k}}\!\big[N_{\{a,b\}}(T)\big] \,.  
\end{equation}

Since  for fixed $k$,  the mapping $f_k:\pi \mapsto \pi^k$ is a bijection of $\cS_n$, we have for any function $F:\cS_n\to\mathbb R$,
\[
\sum_{\pi\in\cS_n} F(\pi^k)=\sum_{\pi\in\cS_n} F(\pi) .
\]
Applying this with
\[
F(\pi):=\sum_{\substack{1\le a<b\le n\\ \{\pi_a,\pi_b\}\cap\{k,k+1\}\neq\emptyset}}
\E_{\nu^{\pi}}\!\big[N_{\{a,b\}}(T)\big]
\]
yields
\[
\sum_{\pi\in\cS_n}\sum_{\substack{1\le a<b\le n\\ \{\pi_a^k,\pi_b^k\}\cap\{k,k+1\}\neq\emptyset}}
\E_{\nu^{\pi^k}}\!\big[N_{\{a,b\}}(T)\big]
=
\sum_{\pi\in\cS_n}\sum_{\substack{1\le a<b\le n\\ \{\pi_a,\pi_b\}\cap\{k,k+1\}\neq\emptyset}}
\E_{\nu^{\pi}}\!\big[N_{\{a,b\}}(T)\big].
\]

Moreover, for fixed $\pi$ and $\{a,b\}$, the condition 
$\{\pi_a,\pi_b\}\cap\{k,k+1\}\neq\emptyset$ 
can hold for at most four values of $k$, namely 
$k\in\{\pi_a-1,\pi_a,\pi_b-1,\pi_b\}\cap[n-1]$. 
Therefore, using a similar counting argument as in~(\ref{new-eq-lower-issue}-\ref{LB:eq-inv-label}), we obtain
\[
\sum_{k=1}^{n-1}
\sum_{1\le a<b\le n}
\mathbf 1_{\{\pi_a,\pi_b\}\cap\{k,k+1\}\neq\emptyset}
\, N_{\{a,b\}}(T)
\ \le \
4\sum_{1\le a<b\le n} N_{\{a,b\}}(T)
\ = \
4T.
\]

Combining the above two equations with \eqref{eq:second-kl-eq}, then gives 
\begin{equation}
\label{LB:eq-inv-label-general-strat}
\kl(x',y') \le \frac{2T\,\Delta^2}{(n-1)\sigma^2}.
\end{equation}

\smallskip
\noindent
Finally, the analogues of~\eqref{x-bound-simple} and~\eqref{y-bound-simple} hold for $x'$ and $y'$:
\begin{equation}
\label{x-y-ineq}
\min_{\pi\in\cS_n,\,k\in[n-1]} 
\Big(1-\P_{\nu^{\pi^k}}(\hat{\pi}\notin\{\pi^k,(\pi^k)_{\mathrm{rev}}\})\Big) \le x',
\qquad
y' \le \frac{1}{n-1}.
\end{equation}
Combining~\eqref{LB:eq-inv-label-general-strat} and~\eqref{x-y-ineq}
with the refined Pinsker inequality used above yields the same lower bound for arbitrary strategies. 
This completes the proof of Proposition~\ref{thm:impossibility-regime}.
\hfill$\square$

%%%%%%%%%%%%%%%%%%%%%%%%%%%%
% 2nd Lower Bound
%%%%%%%%%%%%%%%%%%%%%%%%%%%%

\subsection{Proof of Proposition~\ref{thm:lower-bound:all-T}}
\label{proof-2nd-lower-bound}

As in the proof of Proposition~\ref{thm:impossibility-regime} (for any strategy), 
we consider the Kullback-Leibler divergence $\kl(x',y')$ for 
$x'$ and $y'$ defined in~\eqref{x-def-general}--\eqref{y-def-general}.
Our analysis splits into two cases depending on the value of~$x'$.

\paragraph{Case~1: $x' \le \tfrac{1}{2}$.}
From~\eqref{x-y-ineq}, we have
\[
\min_{\pi,k}\!\left(1-\mathbb{P}_{\nu^{\pi^k}}\big(\hat{\pi}\notin\{\pi^k,(\pi^k)_{\mathrm{rev}}\}\big)\right)
\le x' \le \frac{1}{2}.
\]
Hence there exists $(\pi,k)$ such that
\[
\mathbb{P}_{\nu^{\pi^k}}\big(\hat{\pi}\notin\{\pi^k,(\pi^k)_{\mathrm{rev}}\}\big)
\ge \frac{1}{2}.
\]
Since the mapping $f_k:\pi\mapsto\pi^k$ is a bijection on $\cS_n$, 
we can equivalently write
\[
\max_{\pi,k}\mathbb{P}_{\nu^{\pi^k}}\big(\hat{\pi}\notin\{\pi^k,(\pi^k)_{\mathrm{rev}}\}\big)
=
\max_{\pi\in\cS_n}\mathbb{P}_{\nu^{\pi}}\big(\hat{\pi}\notin\{\pi,\pi_{\mathrm{rev}}\}\big).
\]
Therefore,
\[
\max_{\pi\in\cS_n}\mathbb{P}_{\nu^{\pi}}\big(\hat{\pi}\notin\{\pi,\pi_{\mathrm{rev}}\}\big)
\ge \frac{1}{2}.
\]
Moreover, whenever $\frac{T\Delta^2}{n\sigma^2}\ge1$,
\[
\frac{1}{2}\ge e^{-1}\ge 
\exp\!\left(-\frac{T\Delta^2}{n\sigma^2}\right),
\]
so that
\[
\max_{\pi\in\cS_n}\mathbb{P}_{\nu^{\pi}}\big(\hat{\pi}\notin\{\pi,\pi_{\mathrm{rev}}\}\big)
\ge 
\exp\!\left(-\frac{T\Delta^2}{n\sigma^2}\right).
\]
This proves the lower bound of Proposition~\ref{thm:lower-bound:all-T} for the case $x'\le\tfrac{1}{2}$.

\medskip
\paragraph{Case~2: $x' \ge \tfrac{1}{2}$.}
For all $k\in[n{-}1]$ and $\pi\in\cS_n$, the sets 
$\{\pi,\pi_{\mathrm{rev}}\}$ and $\{\pi^k,(\pi^k)_{\mathrm{rev}}\}$ 
are disjoint (for $n\ge4$). Hence,
\[
\forall\,k\in[n{-}1],\ \forall\,\pi\in\cS_n:\quad
\mathbb{P}_{\nu^{\pi}}\big(\hat{\pi}\in\{\pi^k,(\pi^k)_{\mathrm{rev}}\}\big)
\le 
\mathbb{P}_{\nu^{\pi}}\big(\hat{\pi}\notin\{\pi,\pi_{\mathrm{rev}}\}\big).
\]
Consequently,
\begin{equation}
\label{y-upper-bound}
y'
\le 
\max_{\pi\in\cS_n,\,k\in[n{-}1]}
\mathbb{P}_{\nu^{\pi}}\big(\hat{\pi}\in\{\pi^k,(\pi^k)_{\mathrm{rev}}\}\big)
\le 
\max_{\pi\in\cS_n}
\mathbb{P}_{\nu^{\pi}}\big(\hat{\pi}\notin\{\pi,\pi_{\mathrm{rev}}\}\big).
\end{equation}

To lower bound $\kl(x',y')$, we apply the numerical inequality
\begin{equation}
\label{eq:fano-simple}
\forall\,p\in[0,1],\ \forall\,q\in(0,1):\quad
p\ln\!\Big(\frac{1}{q}\Big)-\ln(2)\le \kl(p,q),
\end{equation}
whose proof is given in Appendix~\ref{section:small-lem-proof}.

Setting $p=x'$, $q=y'$, and using $x'\ge\tfrac{1}{2}$ together with~\eqref{y-upper-bound}, we obtain
\[
\text{if } y'\in(0,1):\quad
\frac{1}{2}\ln\!\Big(
\frac{1}{\max_{\pi\in\cS_n}\mathbb{P}_{\nu^{\pi}}(\hat{\pi}\notin\{\pi,\pi_{\mathrm{rev}}\})}
\Big)
-\ln(2)
\le 
\kl(x',y').
\]
Combining this with~\eqref{LB:eq-inv-label-general-strat} yields
\[
\text{if } y'\in(0,1):\quad
\exp\!\Big(-\frac{4T\Delta^2}{(n-1)\sigma^2}-2\ln(2)\Big)
\le 
\max_{\pi\in\cS_n}
\mathbb{P}_{\nu^{\pi}}\big(\hat{\pi}\notin\{\pi,\pi_{\mathrm{rev}}\}\big).
\]

Moreover, since $\frac{T\Delta^2}{n\sigma^2}\ge 3$,
\[
\frac{4T\Delta^2}{(n-1)\sigma^2}+2\ln 2
\;\le\;
\frac{8T\Delta^2}{n\sigma^2}
\quad (\text{for all } n\ge 4),
\]
and therefore
\begin{equation}
\label{ccl-thm-LB2-bandit}
\text{if } y' \in (0,1): \quad 
\max_{\pi \in \cS_n} 
\mathbb{P}_{\nu^{\pi}}\big(\hat{\pi} \notin \{\pi, \pi_{\mathrm{rev}}\}\big)
\ge 
\exp\!\left(-\frac{8 T \Delta^2}{n \sigma^2}\right).
\end{equation}

If $y'=1$, then~\eqref{y-upper-bound} gives 
$\max_{\pi\in\cS_n}\mathbb{P}_{\nu^{\pi}}\big(\hat{\pi}\notin\{\pi,\pi_{\mathrm{rev}}\}\big)=1$, 
which also satisfies~\eqref{ccl-thm-LB2-bandit}. 
If $y'=0$, then since $x'\ge\tfrac{1}{2}$, one has $\kl(x',0)=+\infty$, and by~\eqref{LB:eq-inv-label-general-strat} this would imply $T=+\infty$, contradicting the finiteness of~$T$.

This completes the proof of Proposition~\ref{thm:lower-bound:all-T}.
\hfill$\square$

%%%%%%%%%%%%%%

%%%%%%%%%%%%%%%%%%%%%%%%%%%%
%2nfundamental lemma for KL divergences
%%%%%%%%%%%%%%%%%%%%%%%%%%

\section{Proofs of small lemmas and inequalities}
%\ya{I did not read this appendix}
\subsection{Proof of Lemma~\ref{lem:fonda}}\label{proof-fundamental-lemma}

We start by recalling a key inequality from \cite{garivier2019explore}[Lemma 1], which follows from the data processing inequality. Let \(Z\) be an \(\mathcal{F}_T\)-measurable random variable taking values in \([0,1]\), and let \(\mathbb{P}, \mathbb{Q}\) be two distributions on \(\mathcal{F}_T\) with expectations \(\mathbb{E}_P, \mathbb{E}_Q\). Then

\begin{equation}\label{eq:dataproc1}
\mathrm{KL}(\mathbb{P}, \mathbb{Q}) \geq \mathrm{kl}\bigl(\mathbb{E}_P[Z], \mathbb{E}_Q[Z]\bigr).
\end{equation}

Since for any event \(\mathcal{E} \in \mathcal{F}_T\) we have \(\mathbb{P}(\mathcal{E}) = \mathbb{E}_P[\mathbf{1}_{\mathcal{E}}]\), applying \eqref{eq:dataproc1} with \(Z = \mathbf{1}_{\mathcal{E}}\) immediately yields, for any event \(\mathcal{E}\),

\begin{equation}\label{eq:dataproc}
\mathrm{KL}(\mathbb{P}, \mathbb{Q}) \geq \mathrm{kl}\bigl(\mathbb{P}(\mathcal{E}), \mathbb{Q}(\mathcal{E})\bigr).
\end{equation}

Next, for any sequence of events \((\mathcal{E}^i)_{i \leq N}\) with \(\mathcal{E}^i \in \mathcal{F}_T\), and distributions \((\mathbb{P}_i)_{i \leq N}\), \((\mathbb{Q}_i)_{i \leq N}\) on \(\mathcal{F}_T\), the convexity of the \(\mathrm{kl}\)-divergence implies

\begin{equation}\label{eq:conv}
\frac{1}{N} \sum_{i=1}^N \mathrm{kl}\bigl(\mathbb{P}_i(\mathcal{E}^i), \mathbb{Q}_i(\mathcal{E}^i)\bigr) \geq \mathrm{kl}\left(\frac{1}{N} \sum_{i=1}^N \mathbb{P}_i(\mathcal{E}^i), \frac{1}{N} \sum_{i=1}^N \mathbb{Q}_i(\mathcal{E}^i)\right).
\end{equation}

For a detailed proof of this convexity property, see \cite{gerchinovitz2020fano}[Corollary 3].

Now, consider two sequences of seriatoin problems \((\nu^i)_{i \leq N}\), \((\tilde{\nu}^i)_{i \leq N}\) with corresponding distributions \(\mathbb{P}_{\nu^i}\), \(\mathbb{P}_{\tilde{\nu}^i}\). Combining \eqref{eq:dataproc} and \eqref{eq:conv}, we get

\begin{equation}\label{eq:comb}
\frac{1}{N} \sum_{i=1}^N \mathrm{KL}(\mathbb{P}_{\nu^i}, \mathbb{P}_{\tilde{\nu}^i}) \geq \mathrm{kl}\left(\frac{1}{N} \sum_{i=1}^N \mathbb{P}_{\nu^i}(\mathcal{E}^i), \frac{1}{N} \sum_{i=1}^N \mathbb{P}_{\tilde{\nu}^i}(\mathcal{E}^i)\right).
\end{equation}

Finally, we use the well-known chain rule for KL divergences in bandit problems (see \cite{lattimore2020bandit}[Lemma 15.1]),  for any seriation problems \(\nu, \tilde{\nu}\),

\begin{equation}\label{eq:tor}
\sum_{1\le a < b \le n} \mathbb{E}_\nu[N_{\{a,b\}}(T)] \, 
\mathrm{KL}(\nu_{\{a,b\}}, \tilde{\nu}_{\{a,b\}}) 
= \mathrm{KL}(\mathbb{P}_\nu, \mathbb{P}_{\tilde{\nu}}).
\end{equation}
Applying identity~\eqref{eq:tor} to each pair $(\nu^i, \tilde{\nu}^i)$ in~\eqref{eq:comb}
and averaging over $i$ yields the desired inequality.

%%%%%%%%%%%%%%%%%%

%%%%%%%%%%%%%%%

\subsection{Proof of \eqref{eq:klsup}}
\label{appendix:proof-kl-ln}

Define $\psi(\lambda) = \lambda q - \ln(1-p+pe^\lambda)$. 
We have $\psi'(\lambda) = q - \frac{pe^\lambda}{1-p+pe^\lambda}$. Setting $q - \frac{pe^\lambda}{1-p+pe^\lambda} = 0$, gives $\lambda = \ln\!\left(\frac{q(1-p)}{p(1-q)}\right) > 0$ since $q>p$. Now $\frac{d^2}{d\lambda^2} \psi(\lambda) = \frac{p(1-p)e^{\lambda}}{(1-p+pe^{\lambda})^2}(-1) < 0$ since $p<1$,
thus, the function $\psi(\lambda)$ achieves its maximum at 
$\lambda = \ln\!\left(\frac{q(1-p)}{p(1-q)}\right)$.

It remains to note that 
\begin{align*}
\psi\!\left(\ln\!\left(\frac{q(1-p)}{p(1-q)}\right)\right)
&= q\ln\!\left(\frac{q(1-p)}{p(1-q)}\right)
   - \ln\!\left(1-p+ p \frac{q(1-p)}{p(1-q)}\right) \\
&= q\ln\left(\frac{q}{p}\right) +q\ln\left(\frac{1-q}{1-p}\right) - \ln\left(\frac{(1-p)(1-q)}{1-q} + \frac{q(1-p)}{1-q}\right)\\
&= q\ln\left(\frac{q}{p}\right) +q\ln\left(\frac{1-q}{1-p}\right) - \ln\left(\frac{1-p}{1-q}\right) = \kl(q,p)    
\end{align*}

%%%%%%%%%%%%%%
\subsection{Proof of \eqref{eq:fano} and \eqref{eq:fano-simple}}
\label{section:small-lem-proof}

Let \( x \in [0,1] \) and \( y \in (0,1) \). We have
\begin{align*}
    \mathrm{kl}(x,y) 
    &= x \ln\left(\frac{x}{y}\right) + (1-x) \ln\left(\frac{1-x}{1-y}\right) \\
    &= x \ln\left(\frac{1}{y}\right) + (1-x) \ln\left(\frac{1}{1-y}\right) + x \ln(x) + (1-x) \ln(1-x) \\
    &\geq x \ln\left(\frac{1}{y}\right) + x \ln(x) + (1-x) \ln(1-x) \\
    &\geq x \ln\left(\frac{1}{y}\right) - \ln(2),
\end{align*}
where the last inequality follows from the fact that the entropy $H(x) := -x \ln(x) - (1-x) \ln(1-x)$
of a Bernoulli random variable \(X \sim \text{Bern}(x)\) is maximized at \(x = \frac{1}{2}\), with
\[
H(x) \leq H\left(\frac{1}{2}\right) = \ln(2).
\]

\subsection{Hoeffding inequality}
\begin{lem}\label{lem:chernoff}
Let $\sigma>0$, and $X_1,...,X_N$ be $N$ independent zero mean $\sigma$ sub-Gaussian random variables. For all $\epsilon >0$,  we have that
$$\P\left( \Big{|} \frac{1}{N}\sum_{i=1}^N X_i \Big{|} \geq \epsilon\right) \leq 2\exp\left(-\frac{N \epsilon^2}{2\sigma^2}\right)$$
\end{lem}
\begin{proof}
By applying the Markov inequality we have for all $\epsilon,\lambda>0$,
\begin{align*}
\P\left(\sum_{i=1}^N X_i \geq \epsilon\right) &= \P\left(\exp\left(\lambda\sum_{i=1}^N X_i\right) \geq \exp(\lambda \epsilon)\right)  \\
&\leq \frac{\E\left[\exp(\lambda\sum_{i=1}^N X_i)\right]}{\exp(\lambda \epsilon)}\\
&\leq \frac{\prod_{i=1}^N \E\left[\exp(\lambda X_i)\right]}{\exp(\lambda \epsilon)}\\
&\leq \exp\left(\frac{N\lambda^2\sigma^2}{2} - \lambda \epsilon\right)\,,
\end{align*}
where the second inequality follows from independence of the $X_i$'s and the final inequality follows from the definition of sub-Gaussian random variables. Setting $\lambda = \frac{\epsilon}{N\sigma^2}$ we have,
$$\P(\sum_{i=1}^N X_i \geq \epsilon) \leq \exp\left(-\frac{\epsilon^2}{2N\sigma^2}\right)\,.$$
This then implies,
$$\P\left(\frac{1}{N}\sum_{i=1}^N X_i \geq \epsilon\right) \leq \exp\left(-\frac{N\epsilon^2}{2\sigma^2}\right)\,.$$
By symmetry, the same bound holds for 
\[
\P\!\left(\frac{1}{N}\sum_{i=1}^N X_i \le -\epsilon\right) \le \exp\!\left(-\frac{N\epsilon^2}{2\sigma^2}\right).
\]
A union bound completes the proof.
\end{proof}

%%%%%%%%%%%%%%%%%%

%%%%%%%%%%%%%%%%%%%

%%%%%%%%%%%%%%%%%%%%%

%%%%%%%%%%%%%%%%%%%%

%%%%%%%%%%%%%
% appendix simu + real data
%%%%%%%%%%%%%%%%

\section{Numerical Simulations and Real-Data Application}
\label{sec:exp}

This appendix provides additional details on the numerical experiments presented in Section~\ref{sec:empirical} of the main text.
In Sub-appendix~\ref{app:num}, we describe the experimental setup used for the synthetic data experiments. 
Sub-appendix~\ref{app:real} details the real-data analysis based on RNA sequence data. 
Finally, Sub-appendix~\ref{appendix-pseudocode} provides the pseudocode of all competing algorithms used in our experiments.

%%%%%%

\subsection{Numerical simulations}
\label{app:num}

We first evaluate the empirical performance of the \asii procedure on synthetic data: 
%To the best of our knowledge, no seriation algorithms have been specifically designed for the active setting. 
we  compare \textsc{ASII} to established batch benchmarks: the \as algorithm of~\cite{cai2022matrix}, the classical \spec method~\cite{atkins1998spectral}, and a simple \textsc{Naive Insertion} baseline (a binary search without backtracking). 
Pseudocodes for all competitors are provided in Appendix~\ref{appendix-pseudocode}.

Batch algorithms typically operate from a single noisy observation $Y = M + E$, where entries of $E$ are independent centered $\sigma$ sub-Gaussian noise. 
To mirror this setting in our active framework, we distribute the sampling budget $T$ evenly among all pairs $\{i,j\}$, so that $Y_{ij} = \widehat M_{ij}$ represents the empirical mean of $O(T/n^2)$ noisy samples. 
This correspondence allows direct comparison between active and batch methods under a common sampling budget.

\begin{figure}
    \centering
    \includegraphics[width=0.70\linewidth]{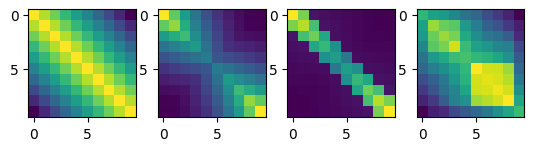}
    \caption{Representation of the Robinson matrices $R^{(s)}$, $s \in \{1, 2, 3, 4\}$, corresponding to the four scenarios. Scenarios (1)–(3) have a minimal gap $\Delta$, while in scenario (4) the minimal gap is random but lower bounded by $\Delta$. The matrix $R^{(1)}$ is Toeplitz, while $R^{(2)}$, $R^{(3)}$, and $R^{(4)}$ are not. Here, $\Delta = 0.2$.}
    \label{fig:scenario}
\end{figure}

\paragraph{Four scenarios.} We consider four scenarios illustrated in Figure~\ref{fig:scenario}. For the first three scenarios, the Robinson matrices are defined as follows for $i > j$: \
\[
R^{(1)}_{i,j} = \Delta n\left(1 - \frac{|i-j|}{n}\right) \quad \textrm{and} \quad R^{(2)}_{i,j} = \Delta(n - |i-j|)\max(j, n-i)^{1.5}
\]
% old scenario 3
% \[
% R^{(3)}_{i,j} = \Delta(n - |i-j|)\max(j, n-i) + 
% \begin{cases}
% 1000 \quad & \text{if } |i-j| \leq n/4, \\
% 0 \quad & \text{otherwise.}
% \end{cases}
% \]

% \[
% R^{(3)}_{i,j} = \Delta(n - |i-j|)+\max(j, n-i)^2 
% \]

\[
R^{(3)}_{i,j} =  + 
\begin{cases}
10\Delta(n - |i-j|)\max(j, n-i) \quad & \text{if } |i-j| \leq n/4, \\
\Delta(n - |i-j|)\max(j, n-i) \quad & \text{otherwise.}
\end{cases}
\]

In the fourth scenario, the matrix {{\small $R^{(4)}$}}  is generated randomly as follows. For $i \in [n]$, the diagonal element {{\small $R^{(4)}_{i,i}$}} is drawn from a Uniform$(1,10)$ distribution. The rest of the matrix is generated sequentially, where for all $i > j$,
\[
R^{(4)}_{i,j} = \min(R^{(4)}_{i-1,j}, R^{(4)}_{i,j+1}) - \mathrm{Uniform}(\Delta, 10\Delta).
\]
All matrices are symmetric, with {{\small $R^{(s)}_{j,i} = R^{(s)}_{i,j}$}} for $j > i$ and $s \in [4]$. The parameter $\Delta > 0$ is varied across scenarios.

We define the four similarity matrices $M = R^{(s)}_\pi$, $s \in [4]$, where $n = 10$ and $\pi$ is uniformly drawn from the set of permutations of $[n]$. We set $\sigma = 1$ and $T = 10000$.

% old exp
% \begin{figure}
% \begin{minipage}[b]{0.47\linewidth}
%     \centering
%     \includegraphics[width=0.75\linewidth]{scen1n10T10000.png}  
% \end{minipage}\hfill
% \begin{minipage}[b]{0.47\linewidth}
%     \centering
%     \includegraphics[width=0.75\linewidth]{scen5n10T1o000.png}
% \end{minipage}\hfill
% \begin{minipage}[b]{0.47\linewidth}
%     \centering
%     \includegraphics[width=0.75\linewidth]{scen55n10T10000.png}
% \end{minipage}\hfill
% \begin{minipage}[b]{0.47\linewidth}
%     \centering
%     \includegraphics[width=0.75\linewidth]{scen8n10T10000.png}
% \end{minipage}
% \caption{ 
% Empirical error probabilities  for \as, \asii, \spec, and \textsc{Naive Insertion}  as the  parameter $\Delta$ varies. Scenarios (1)–(4) are displayed from left to right and top to bottom. Each experiment involves $n=10$ items and $T=10{,}000$ observations. Each point is averaged over 50 Monte Carlo runs.}
% \label{fig:sim}
% \end{figure}    

\begin{figure}
\begin{minipage}[b]{0.47\linewidth}
    \centering
    \includegraphics[width=0.70\linewidth]{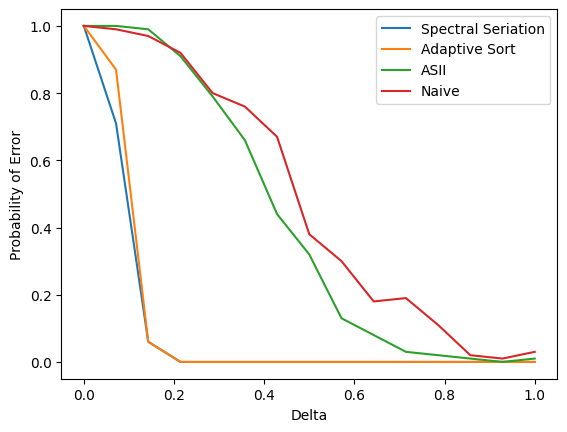}  
\end{minipage}\hfill
\begin{minipage}[b]{0.47\linewidth}
    \centering
    \includegraphics[width=0.70\linewidth]{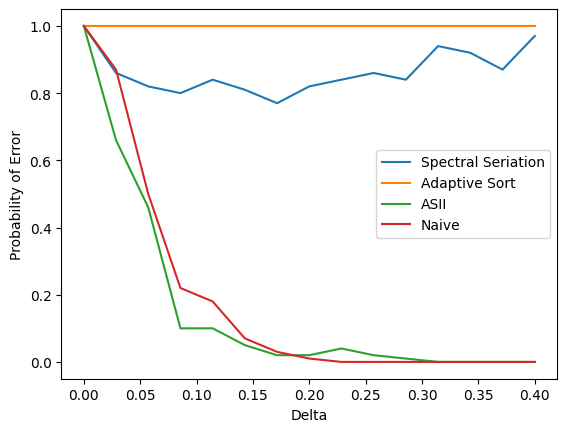}
\end{minipage}\hfill
\begin{minipage}[b]{0.47\linewidth}
    \centering
    \includegraphics[width=0.70\linewidth]{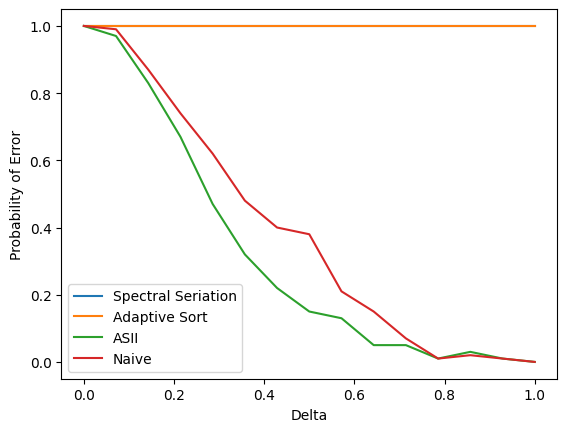}
\end{minipage}\hfill
\begin{minipage}[b]{0.47\linewidth}
    \centering
    \includegraphics[width=0.70\linewidth]{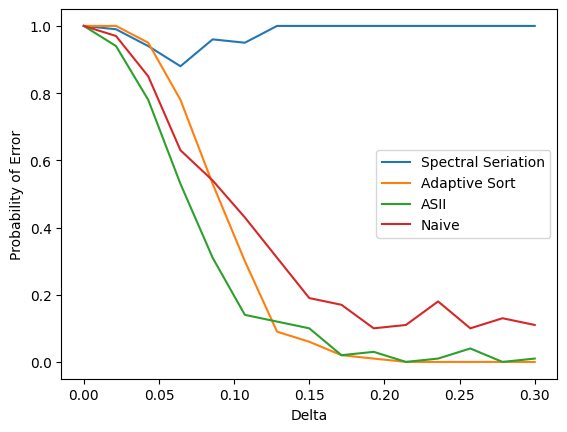}
\end{minipage}
\caption{ 
Empirical error probabilities  for \as, \asii, \spec, and \textsc{Naive Insertion}  as the  parameter $\Delta$ varies. Scenarios (1)–(4) are displayed from left to right and top to bottom. Each experiment involves $n=10$ items and $T=10{,}000$ observations. Each point is averaged over 100 Monte Carlo runs.}
\label{fig:sim}
\end{figure}    

\begin{figure}
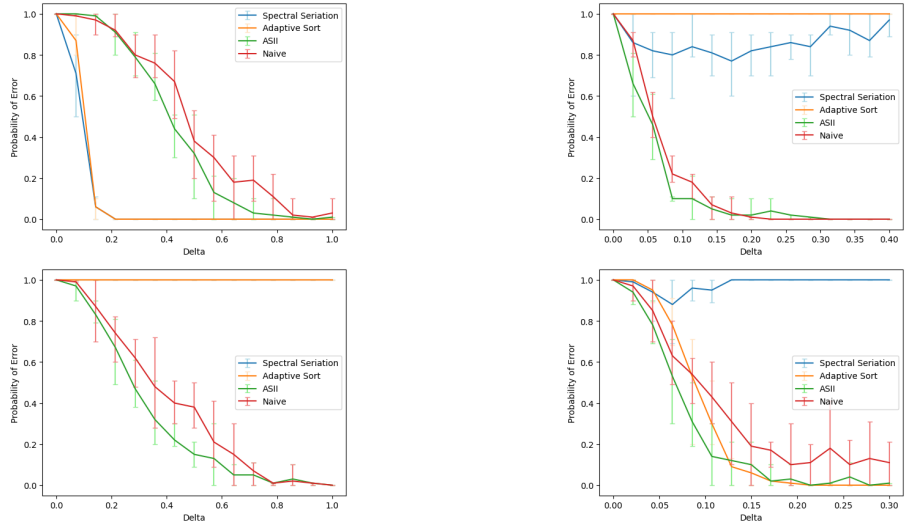

\begin{minipage}[b]{0.47\linewidth}
    \centering
    \includegraphics[width=0.70\linewidth]{newexp/2scen1n10t10000b.png}  
\end{minipage}\hfill
\begin{minipage}[b]{0.47\linewidth}
    \centering
    \includegraphics[width=0.70\linewidth]{newexp/2scen2n10t10000b.png}
\end{minipage}\hfill
\begin{minipage}[b]{0.47\linewidth}
    \centering
    \includegraphics[width=0.70\linewidth]{newexp/newscen3n10t10000b.png}
\end{minipage}\hfill
\begin{minipage}[b]{0.47\linewidth}
    \centering
    \includegraphics[width=0.70\linewidth]{newexp/2scen4n10t10000b.png}
\end{minipage}
\caption{
Same experiment as in Fig.~\ref{fig:sim}, with error bars added. 
Empirical error probabilities for \as, \asii, \spec, and \textsc{Naive Insertion} as the parameter~$\Delta$ varies. 
For each value of~$\Delta$, the 100 Monte Carlo runs are split into 10 equal groups; 
the error bars indicate the 0.1 and 0.9 quantiles of the empirical error across these groups.
}
\label{fig:simer}
\end{figure}    

Above we describe the experimental setup. 
The resulting empirical error curves are shown in Figures~\ref{fig:sim} and~\ref{fig:simer}; discussion of their performance is given in Section~\ref{sec:empirical} of the main paper.

\subsection{Application to real data}\label{app:real}

We now assess the performance of \textsc{ASII} on a real single-cell RNA sequencing dataset, following the biological setup previously studied in~\cite{cai2022matrix}. The goal is to infer the latent temporal ordering of cells during a differentiation process, based on pairwise similarities of gene-expression profiles.

\paragraph{Dataset and preprocessing.}
We use the dataset from~\cite{guo2015transcriptome}, which contains RNA sequencing data for $n=242$ human primordial germ cells (PGCs) collected at developmental ages of 4, 7, 10, 11, and 19~weeks. Each observation corresponds to a high-dimensional vector of gene-expression counts across more than $230{,}000$ genes.  
Following the preprocessing pipeline of~\cite{cai2022matrix}, we employ the \texttt{Seurat} package~\cite{seurat} in \texttt{R} to normalize and reduce the dimensionality of the data:
\begin{enumerate}
    \item normalization using \texttt{NormalizeData};
    \item identification of highly variable genes with \texttt{FindVariableFeatures};
    \item standardization via \texttt{ScaleData};
    \item principal component analysis with $d=10$ components.
\end{enumerate}
Let $X_1,\dots,X_n \in \mathbb{R}^{10}$ denote the PCA embeddings of the cells, and $D$ the resulting pairwise Euclidean distance matrix. The similarity matrix is then defined as
\[
M = c\,\mathbf{1}_n - D,
\qquad \text{where } c = \|D\|_\infty.
\]
This construction ensures that larger similarities correspond to smaller distances between cells.

\paragraph{Results.}
Figure~\ref{fig:rna} shows the similarity matrix $M$ under a random permutation of the cells (left) and after reordering by \textsc{ASII} (right).  
The recovered ordering reveals a clear block-diagonal structure consistent with developmental progression: dissimilar regions (blue) are pushed to the boundaries, while groups of highly similar cells (yellow and green) align along the diagonal.  
Although this dataset is far from satisfying the assumptions of our theoretical model, \textsc{ASII} still recovers biologically meaningful organization, demonstrating robustness to strong model misspecification.

\paragraph{Discussion.}
This experiment illustrates that \textsc{ASII} can yield interpretable orderings even on complex, high-dimensional biological data that deviate substantially from idealized Robinson structures.  
It therefore provides empirical evidence that active seriation remains effective beyond controlled synthetic settings, supporting its potential relevance for practical data-analysis tasks.  Together with the synthetic experiments presented above, these results confirm the robustness of active seriation and its potential relevance for real-world data-analysis tasks.

\begin{figure}[t]
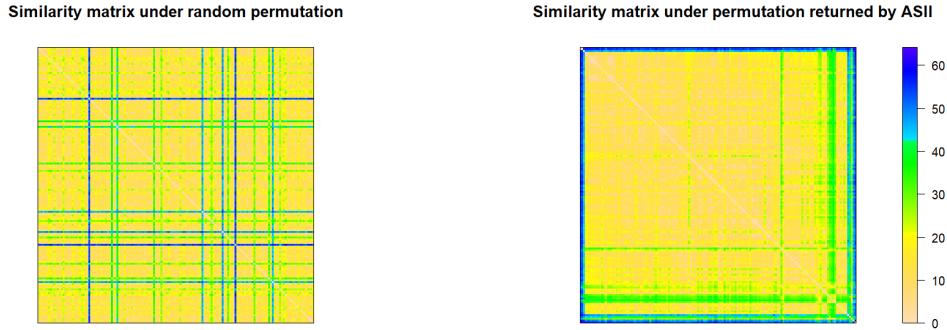

\begin{minipage}[b]{0.47\linewidth}
    \centering
    \includegraphics[width=0.95\linewidth]{simrnd.png}  
\end{minipage}\hfill
\begin{minipage}[b]{0.47\linewidth}
    \centering
    \includegraphics[width=0.95\linewidth]{assim.png}
\end{minipage}
\caption{
PGC similarity matrix before and after reordering by \textsc{ASII}.
Left: random permutation; right: ordering inferred by \textsc{ASII}.
The recovered structure highlights coherent developmental trajectories among cells.}
\label{fig:rna}
\end{figure}

%%%%%%%%%%%%
% Pseudo codes and ????
%%%%%%%%%%%%%%%%

\subsection{Pseudocodes of competitor algorithms}
\label{appendix-pseudocode}

In this appendix, we present the pseudocodes of the competitor algorithms used in the numerical simulations: the Naive Binary Search algorithm, the \as algorithm from \cite{cai2022matrix}, and the \spec algorithm introduced by \cite{atkins1998spectral}.

\paragraph{Naive Insertion Algorithm}  
The naive insertion procedure is identical to the \textsc{ASII} procedure described in Section~\ref{algo-section}, except that whenever \textsc{ASII} invokes the \bbs subroutine, the naive algorithm instead calls the \naive subroutine defined below. Recall that the \textsc{Test} subroutine was introduced at the start of Section~\ref{algo-section}.

\begin{algorithm}[H]
 \caption{\textsc{Naive Binary Search}}\label{algo:nBBS}
\begin{algorithmic}[1]
\REQUIRE $(\pi_1,\ldots,\pi_{k-1})$, a permutation of $[k-1]$
\ENSURE $\pi_k \in [k]$
\STATE Initialize $(l_0, r_0) \in [k-1]^2$ such that $(\pi_{l_0}, \pi_{r_0}) = (1, k-1)$
\WHILE{$\pi_{r_{t-1}} - \pi_{l_{t-1}} > 1$}
    \STATE Let $m_t \in [k-1]$ satisfy $\pi_{m_t} = \left\lfloor \frac{\pi_{l_{t-1}} + \pi_{r_{t-1}}}{2} \right\rfloor$
    \IF{\textsc{Test}$\left(k, l_{t-1}, m_t, \left\lfloor \frac{T}{n \log(k)} \right\rfloor \right) = 0$}
        \STATE $(l_t, r_t) \leftarrow (l_{t-1}, m_t)$
    \ELSE
        \STATE $(l_t, r_t) \leftarrow (m_t, r_{t-1})$
    \ENDIF
\ENDWHILE
\STATE Set $\pi_k \leftarrow \pi_{l_t} + 1$
\end{algorithmic}
\end{algorithm}

\paragraph{Adaptive Sorting Algorithm}  
The \as algorithm from \cite{cai2022matrix} operates in a batch setting, where the learner observes a single noisy matrix  
\begin{equation}
\label{batch-setting}
Y = M + Z,
\end{equation}
with $Z$ a noise matrix having independent zero-mean $\sigma$-sub-Gaussian entries, for $\sigma>0$. Algorithms designed for the batch setting are adapted to our active setting as follows, the algorithm first divides the budget $T$ uniformly across all coefficients of the matrix, i.e. each pair $(i,j)$ is sampled $\lfloor T/n^2 \rfloor$ times, so as to generate the noisy matrix $Y$. When run in the context of our active setting, algorithms that are designed exclusively for the batch setting naturally require $T\geq n^2$. 

For each $i \in [n]$, $Y_{i,-i}$ denotes the $i$th row of $Y$ excluding the $i$th entry, the \as procedure is then as follows:

\begin{algorithm}[H]
 \caption{\textsc{Adaptive Sorting (AS)}}\label{algo:AS}
\begin{algorithmic}[1]
\REQUIRE $Y \in \mathbb{R}^{n \times n}$
\ENSURE $\pi = (\pi_1, \ldots, \pi_n)$
\STATE Compute scores $S_i = \sum_{j\in [n]\setminus{\{i\}}} Y_{ij}$ for each $i \in [n]$
\STATE Set $\pi_1 = \arg\min_{i \in [n]} S_i$
\FOR{$i = 2, \ldots, n-1$}
    \STATE Select $\pi_i \in \arg\min_{j \in [n] \setminus \{\pi_1, \ldots, \pi_{i-1}\}} \|Y_{j,-j} - Y_{\pi_{i-1}, -\pi_{i-1}}\|_1$
\ENDFOR
\end{algorithmic}
\end{algorithm}

\paragraph{Spectral Seriation Algorithm}  
The widely used Spectral Seriation algorithm is also designed for the batch setting described in~\eqref{batch-setting}. It proceeds as follows:

\begin{algorithm}[H]
 \caption{\textsc{Spectral Seriation}}\label{algo:spec}
\begin{algorithmic}[1]
\REQUIRE $Y \in \mathbb{R}^{n \times n}$
\ENSURE $\pi = (\pi_1, \ldots, \pi_n)$
\STATE Compute the graph Laplacian $L = D - Y$, where $D = \mathrm{diag}(d_1, \ldots, d_n)$ with $d_i = \sum_{j=1}^n Y_{ij}$
\STATE Let $\hat{v}$ be the eigenvector associated with the second smallest eigenvalue of $L$
\STATE Define $\pi$ as the permutation that sorts the entries of $\hat{v}$ in ascending order
\end{algorithmic}
\end{algorithm}

%%%%%%%%%%%%%%%%%

\end{document}